%% file: main.tex
\documentclass{article}

\input{preamble.tex}

\title{SEEDS: Exponential SDE Solvers for Fast High-Quality Sampling from Diffusion Models}

\author{%
  Martin~Gonzalez\thanks{Corresponding author:  \href{mailto:martin.gonzalez@irt-systemx.fr}{\texttt{martin.gonzalez@irt-systemx.fr}}} \\
  IRT SystemX \\
   \And
   Nelson~Fernandez \\
   Air Liquide \\
   \And
   Thuy Tran \\
   IRT SystemX \\
   \And
   Elies Gherbi \\
   IRT SystemX \\
   \And
   Hatem Hajri \\
   IRT SystemX \\ \& Safran \\
   \And
   Nader Masmoudi \\
   New York University \\
}
\begin{document}
\input{figures.tex}
\maketitle

\begin{abstract}
A potent class of generative models known as Diffusion Probabilistic Models
(DPMs) has become prominent. A forward diffusion process adds gradually noise
to data, while a model learns to gradually denoise. Sampling from pre-trained
DPMs is obtained by solving differential equations (DE) defined by the learnt
model, a process which has shown to be prohibitively slow. Numerous efforts on
speeding-up this process have consisted on crafting powerful ODE solvers.
Despite being quick, such solvers do not usually reach the optimal quality
achieved by available slow SDE solvers. Our goal is to propose SDE solvers that
reach optimal quality without requiring several hundreds or thousands of NFEs
to achieve that goal. We propose Stochastic Explicit Exponential
Derivative-free Solvers (SEEDS), improving and generalizing Exponential
Integrator approaches to the stochastic case on several frameworks. 
After carefully analyzing the formulation of exact
solutions of diffusion SDEs, we craft SEEDS to analytically compute the linear
part of such solutions. Inspired by the Exponential Time-Differencing method,
SEEDS use a novel treatment of the stochastic components of solutions,
enabling the analytical computation of their variance, and contains high-order
terms allowing to reach optimal quality sampling $\sim3$-$5\times$ faster than previous
SDE methods. We validate our approach on several image generation benchmarks,
showing that SEEDS outperform or are competitive with previous SDE solvers.
Contrary to the latter, SEEDS are derivative and training free, and we fully
prove strong convergence guarantees for them. Our code is publicly available 
in \href{https://github.com/nfsrules/SEEDS}{this link}.
 \end{abstract}

\section{Introduction}\label{sec:intro}

Diffusion Probabilistic Models (DPMs) {\cite{sohl2015deep,ho2020denoising}}
have emerged as a powerful category of generative models and have proven to
quickly become SOTA for generative tasks such as image, video, audio
generation
{\cite{dhariwal2021diffusion,meng2021sdedit,ho2022video,chen2020wavegrad,chen2021wavegrad}},
and more {\cite{popov2021diffusion,vahdat2021score}}. These models employ a
forward diffusion process where noise is gradually added to the data, and the
model learns to remove the noise progressively. However, sampling from most
pre-trained DPMs is done by simulating the trajectories of associated
differential equations (DE) and has been found to be prohibitively slow
{\cite{song2020score}}. Previous attempts to accelerate this process have
mainly focused on developing efficient ODE solvers. On one hand,
training-based methods speed-up sampling by using auxiliary training such as
Progressive Distillation {\cite{salimans2022progressive}} and Fourier Neural
Operators {\cite{zheng2022fast}}, learning the noise schedule, scaling,
variance, or trajectories. On the other hand, training-free methods 
\cite{Karras2022edm,dpm-solver,jolicoeur2021gotta,zhang2022fast,lu2022dpm} 
are slower but are more versatile
for being employed on different models and achieve higher quality results than
current training-based methods. Although these solvers are fast, they often
fall short of achieving the optimal quality attained by slower SDE solvers
\cite{Karras2022edm}. The latter usually do not present theoretical
convergence guarantees and, while being training-free, they often still
require costly parameter optimization to achieve optimal results which might
be difficult to estimate for large datasets.

Our objective is to introduce SDE solvers that can achieve optimal quality
without requiring an excessively large number of function evaluations (NFEs).
To accomplish this, we propose Stochastic Explicit Exponential Derivative-free Solvers
(SEEDS). These are \tmtextit{off-the-shelf} SDE samplers: they
offer promising high-quality sampling without further training or parameter
optimization. SEEDS enhance and generalize existing Exponential Integrator
{\cite{dpm-solver,lu2022dpm, zhang2022gddim}} approaches to the stochastic 
case on several practical frameworks and is based on the following 4 building blocks:
(1) the exponential representation of
semi-linear SDE exact solutions which isolate linear terms to be computed analytically; (2) a
general change-of-variables recipe to simplify the integrals involved in
the solutions in order to better approximate the deterministic one; (3) a
method to analytically compute the variance of the stochastic one; (4) a method to 
decompose the obtained stochastic components in such a way that the resulting 
sequences of higher-stage numerical approximations yield Markov chains.

Overall, we make the following contributions: (a) our change-of-variables method allow us to re-
frame the gDDIM solver \cite{zhang2022gddim} as a special case of SEEDS and to craft bespoke 
solvers for the EDM-preconditioned DPMs in \cite{Karras2022edm} which attain equivalent 
sampling quality twice faster than the previous SOTA sampling method \cite{Karras2022edm}; (b) 
based on the Stochastic Exponential Time-Differencing method, we analytically compute of our 
solver's stochastic components in terms of the so-called $\varphi$-functions, allowing for 
efficient implementation; (c) our noise decomposition method (4), which is the key of success 
of SEEDS, is both theoretically grounded, experimentally shown to be optimal, and has no 
deterministic equivalent; (d) we provide full proofs of strong/weak convergence guarantees for 
our SDE solvers which, to our knowledge, has no precedent in the DPM literature. In particular, 
although the formula used for the truncated Itô-Taylor expansion might seem similar to that of 
\cite{dpm-solver}, our convergence theorems leverage the full Itô-Taylor expansion of 
solutions, making our proofs not incremental and different from \cite{dpm-solver}; (e) we conduct extensive
experiments demonstrating that
SEEDS establishes SOTA results among available solvers on several image generation
benchmarks, or is competitive with existing SDE solvers while being 2-5
times faster than the latter.

Although our solvers theoretically apply to certain non-isotropic DPMs such as
Critically-damped Langevin Dynamics (CLD) {\cite{dockhorn2022score}} (see
Prop. \ref{related} and Rem. \ref{cld-rem}), we will restrict our presentation
to the isotropic case for which many notations become simpler.

\section{Background on Diffusion Probabilistic Models}\label{sec:background}

\paragraph{General Isotropic DE Formulation.} The evolution of a data sample
$\mathbf{x}_0 \in \mathbb{R}^d$ taken from an unknown data distribution
$p_{\tmop{data}}$ into standard Gaussian noise can be defined as a forward
diffusion process $\{ \mathbf{x}_t \}_{t \in [0, T]}$, with $T > 0$, which is a
solution to a linear SDE:
\begin{equation}
  \label{diff1} \mathd \mathbf{x}_t = f (t) \mathbf{x}_t \mathd t + g (t)
  \mathd \tmmathbf{\omega}_t, \qquad f (t) \assign \frac{\mathd \log
  \alpha_t}{\mathd t}, \quad g (t) = \alpha_t  \sqrt{\frac{\mathd
  [\sigma^2_t]}{\mathd t}},
\end{equation}
where $f (t), g (t) \in \mathbb{R}^{d \times d}$ are called the drift and
diffusion coefficients respectively and $\tmmathbf{\omega}$ is a
$d$-dimensional standard Wiener process, and $\alpha_t, \sigma_t \in
\mathbb{R}^{> 0}$ are differentiable functions with bounded derivatives. 
In practice, when specifying the SDE \eqref{diff1}, $\sigma_t$ acts as a
schedule controlling the noise levels of an input at time $t$, and $\alpha_t$
as a time-dependent signal scaling controlling its dynamic range.

By
denoting $p_t (\mathbf{x}_t)$ the marginal distribution of $\mathbf{x}_t$ at
time $t$, functions $\alpha_t$ and $\sigma_t$ are designed so that the end-time distribution
of the process process is $p_T (\mathbf{x}_T) \approx \mathcal{N}
(\mathbf{x}_T |\tmmathbf{0}, \tilde{\sigma}^2 \mathbf{I}_d)$ for some
$\tilde{\sigma} > 0$. As \eqref{diff1} is linear, the transition probability
$p_{0 t} (\mathbf{x}_t | \mathbf{x}_0)$ from $\mathbf{x}_0$ to $\mathbf{x}_t$
is Gaussian whose mean and variance can be expressed in terms of $\alpha_t$
and $\sigma_t$. For simplicity, we will denote it
\[ p_{0 t} (\mathbf{x}_t | \mathbf{x}_0) =\mathcal{N} (\mathbf{x}_t ; \mu_t
   \mathbf{x}_0, \Sigma_t \textbf{}), \qquad \mu_t, \Sigma_t \in \mathbb{R}^{d
   \times d} . \]
The evolution of the reverse time process of $\{ \mathbf{x}_t \}_{t \in [0,
T]}$ (which we will still denote $\{ \mathbf{x}_t \}_{t \in [0, T]}$ for
simplicity) is then driven by a backward differential equation
\begin{equation}
  \label{eqq} \mathd \mathbf{x}_t = [f (t) \mathbf{x}_t - \frac{1 + \ell^2}{2}
  g^2 (t) \nabla_{\mathbf{x}_t} \log p_t (\mathbf{x}_t)] \mathd t + \ell g (t)
  \mathd \tmmathbf{\bar{\omega}}_t,
\end{equation}
where $\mathd t$ are negative infinitesimal time-steps and
$\tmmathbf{\bar{\omega}}_t$ is now a Wiener process with variance $- \mathd
t$. In this article, we will concentrate in the cases $\ell = 0, 1$, known in
the literature as the Probability Flow ODE (PFO) and diffusion reverse SDE
(RSDE), respectively.

\paragraph{Training.} Denoising score-matching is a technique to train a
time-dependent model $D_{\theta} (\mathbf{x}_t, t)$ to approach the score
function $\nabla_{\mathbf{x}_t} \log p_t (\mathbf{x}_t)$ at each time $t$.
Intuitively, as $D_{\theta}$ approaches the score, it produces a sample which
maximizes the log-likelihood. As such, this model is coined as a \tmtextit{data
prediction} model. However, in practice DPMs can be more efficiently trained
by reparameterizing $D_{\theta}$ into a different model $F_{\theta}
(\mathbf{x}_t, t)$ whose objective is to predict the noise to be removed
from a sample at time $t$. This \tmtextit{noise prediction} model is trained
by means of the loss
\[ \mathbb{E}_{t \sim \mathcal{U} [0, T], \mathbf{x}_0 \sim p_{\tmop{data}},
   \epsilon \sim \mathcal{N} (\tmmathbf{0}, \mathbf{I}_d)} [\| \epsilon -
   F_{\theta} (\mu_t \mathbf{x}_0 + \tmmathbf{K}_t \epsilon, t)
   \|^2_{\tmmathbf{K}_t^{- 1} \gamma_t  \tmmathbf{K}_t^{- \top}}], \]
where $\gamma_t$ is a time dependent weighting parameter and  
$\tmmathbf{K}_t \tmmathbf{K}^{\top}_t = \Sigma_t$.

\section{Accelerating Optimal Quality Solvers for Diffusion SDEs}
\label{sec:ingredients}

Once $F_{\theta}$ or $D_{\theta}$ have been trained, one can effectively solve
\eqref{eqq} after replacing the score function by its corresponding expression
involving either one of these models. For instance, taking the noise
prediction model and $\ell = 1$, sampling is conducted by simulating
trajectories of a SDE of the form
\begin{equation}
  \mathd \mathbf{x}_t = [A (t) \mathbf{x}_t + b (t) F_{\theta} (\mathbf{x}_t,
  t)] \mathd t + g (t) \mathd \tmmathbf{\omega}_t,\label{nsde}
\end{equation}
for some functions $A (t), b (t)$ which are usually not equal to $f (t), g^2
(t)$. In what follows, we consider a time discretization $\{t_i\}^M_{i=0}$ going
backwards in time starting from $t_0=T$ to $t_M=0$ and to ease the notation we will always denote $t < s$ for two consecutive time-steps $t_{i+1}< t_i$.

The usual representation of the analytic solution
$\mathbf{x}_t$ at time $t$ of \eqref{nsde} w.r.t. an initial
condition $\mathbf{x}_s$ is:
\begin{equation}
\label{ana-rep}
  \mathbf{x}_t =\mathbf{x}_s + \int_s^t [A (\tau)
   \tmcolor{red}{\mathbf{x}_{\tau}} + {\color[HTML]{008000}b (\tau)}
   \tmcolor{orange}{F_{\theta} (\mathbf{x}_{\tau}, \tau)}] \mathd \tau +
   \tmcolor{blue}{\int_s^t g (\tau) \mathd \tmmathbf{\omega}_{\tau}}.
\end{equation}
The numerical schemes we propose for approaching the trajectories of
\eqref{nsde} based on representation \eqref{ana-rep} are grounded in the 4 following
principles:
\begin{enumerate}
  \item The variation-of-parameters formula: representing analytic solutions
  with \tmcolor{red}{linear term }extracted from the integrand;
  
  \item Exponentially weighted integrals: extracting the
  {\color[HTML]{008000}time-varying linear coefficient} attached to the
  network from the integrand by means of a specific choice of change of
  variables which allows analytic computation of the leading coefficients in
  the truncated It{\^o}-Taylor expansion associated to
  $\tmcolor{orange}{F_{\theta} (\mathbf{x}_{\tau}, \tau)}$ up to any arbitrary
  order;
  
  \item Modified Gaussian increments: after replicating such change of
  variables onto the\tmcolor{blue}{ stochastic integral}, analytically
  computing its variance.

  \item Markov-preserving noise decomposition: \tmcolor{blue}{stochastic integrals} need 
  to be dependent on overlapping time intervals and independent on non-overlapping ones.
\end{enumerate}
\paragraph{Exponential representation of exact solutions of diffusion SDEs.} 
The first key insight of this work is that, using the
\tmtextit{variation-of-parameters} formula, we can represent the analytic
solution $\mathbf{x}_t$ at time $t$ of \eqref{nsde} with respect to an initial
condition $\mathbf{x}_s$ as follows:
\begin{equation}
  \label{eq:NSDE} \mathbf{x}_t = \Phi_A (t, s) \tmcolor{red}{\mathbf{x}_s} +
  \int^t_s {\color[HTML]{008000}\Phi_A (t, \tau)}  {\color[HTML]{008000}b
  (\tau)}  \tmcolor{orange}{F_{\theta} (\mathbf{x}_{\tau}, \tau)} \mathd \tau
  + \tmcolor{blue}{\int^t_s \Phi_A (t, \tau) g (\tau) \mathd
  \tmmathbf{\omega}_{\tau}},
\end{equation}
where $\Phi_A (t, s) = \exp \left( \int^t_s A (\tau) \mathd \tau \right)$ is
called the transition matrix associated with $A (t)$.
The separation of the linear and nonlinear components is achieved by this
formulation and also appears in \cite{dpm-solver,lu2022dpm, zhang2022gddim}. 
It differs from black-box SDE solvers as it enables the
exact calculation of the linear portion, thereby removing any approximation
errors associated with it. However, the integration of the nonlinear portion
remains complex due to the interaction of the new coefficient
${\color[HTML]{008000}\Phi_A (t, \tau) b (\tau)}$ and the intricate neural
network, making it challenging to approximate.

\paragraph{Exponentially weighted integrals.} Due to the regularity conditions 
usually imposed on the drift and diffusion coefficients of
\eqref{diff1}, one can make several choices of change-of-variables on the 
integral components in \eqref{eq:NSDE} in order to simplify it. Our second 
key insight is that there is a specific choice of change of variables allowing the
analytic computation of the It{\^o}-Taylor coefficients of
$\tmcolor{orange}{F_{\theta} (\mathbf{x}_{\tau}, \tau)}$ with respect to
$\tau$, and based at $s$ that will be used for crafting SEEDS. More
specifically, this expansion reads
\[ \tmcolor{orange}{F_{\theta} (\mathbf{x}_{\tau}, \tau)} = \sum_{k = 0}^n
   \frac{(\tau - s)^k}{k!} F^{(k)}_{\theta} (\mathbf{x}_s, s) +\mathcal{R}_n,
\]
where the residual $\mathcal{R}_n$ consists of deterministic iterated
integrals of length greater than $n + 1$ and all iterated integrals with at
least one stochastic component. As such, we obtain
\begin{equation}
  \int^t_s {\color[HTML]{008000}\Phi_A (t, \tau)}  {\color[HTML]{008000}b
  (\tau)}  \tmcolor{orange}{F_{\theta} (\mathbf{x}_{\tau}, \tau)} \mathd \tau
  = \sum_{k = 0}^n F^{(k)}_{\theta} (\mathbf{x}_s, s) \int^t_s
  {\color[HTML]{008000}\Phi_A (t, \tau)}  {\color[HTML]{008000}b (\tau)}
  \frac{(\tau - s)^k}{k!} \mathd \tau + \tilde{\mathcal{R}}_n, \label{taylor}
\end{equation}
where $\tilde{\mathcal{R}}_n$ is easily obtained from $\mathcal{R}_n$ and
$\int^t_s {\color[HTML]{008000}\Phi_A (t, \tau) b (\tau)} \mathd \tau$. The
third key contribution of our work is to rewrite, for any $k \geqslant 0$, the
integral $\int^t_s {\color[HTML]{008000}\Phi_A (t, \tau)} 
{\color[HTML]{008000}b (\tau)} \frac{(\tau - s)^k}{k!} \mathd \tau$ as an
integral of the form $\int^{\lambda_t}_{\lambda_s} e^{\lambda} \frac{(\lambda
- \lambda_s)^k}{k!} \mathd \lambda$ since the latter is recursively
analytically computed in terms of the $\varphi$-functions
\begin{align*}
\varphi_0 (t) : = e^t, \qquad \varphi_{k + 1} (t) \assign \int_0^1 e^{(1 -
   \tau) t}  \frac{\tau^k}{k!} \mathd \tau = \dfrac{\varphi_k (t) - \varphi_k
   (0)}{t}, \qquad k \geqslant 0.
\end{align*}

\paragraph{Modified Gaussian increments.} In order for making such change of 
variables to be consistent on the overall
system, one needs to replicate it accordingly in the stochastic integral
$\tmcolor{blue}{\int^t_s \Phi_A (t, \tau) g (\tau) \mathd
\tmmathbf{\bar{\omega}}_{\tau}}$. As such, our last key contribution is to
transform it into an exponentially weighted stochastic integral with
integration endpoints $\lambda_s, \lambda_t$ and apply the Stochastic
Exponential Time Differencing (SETD) method \cite{adamu} to compute its variance
analytically, as illustrated in \eqref{seeds-1} below.

Let us test our methodology in two key examples. As we explained in Section
\ref{sec:background}, sampling from pre-trained DPMs amounts on choosing a
schedule $\sigma_t$, a scaling $\alpha_t$, and a parameterized learnt
approximation of the score function $\nabla_{\mathbf{x}_t} \log p_t
(\mathbf{x}_t)$. In what follows, we denote by $t_{\lambda}$ the inverse of a
chosen change of variables $\lambda_t$ and we denote
$\widehat{\mathbf{x}}_{\lambda} \assign \mathbf{x} (t_{\lambda} (\lambda)),
\hat{F}_{\theta} (\widehat{\mathbf{x}}_{\lambda}, \lambda) \assign F_{\theta}
(\mathbf{x} (t_{\lambda} (\lambda)), t_{\lambda} (\lambda))$.

\paragraph{The VPSDE case.} Let $\tilde{\alpha}_t \assign \frac{1}{2}
\beta_d t^2 + \beta_m t$, where $\beta_d, \beta_m > 0$ and $t \in [0, 1]$.
Then, by denoting
\begin{equation}
  \sigma_t : = \sqrt{e^{\tilde{\alpha}_t} - 1}, \qquad \alpha_t : = e^{-
  \frac{1}{2}  \tilde{\alpha}_t}, \qquad \bar{\sigma}_t \assign \alpha_t
  \sigma_t, \qquad \nabla_{\mathbf{x}_t} \log p_t (\mathbf{x}_t) \simeq
  \bar{\sigma}^{- 1}_t F_{\theta}  (\mathbf{x}_t, t), \label{vp-coefs}
\end{equation}
we obtain the VP SDE framework from {\cite{song2020score}} and the following
result.
\begin{proposition}
  \label{Prop1}Let $t < s$. The analytic solution at time $t$ of the RSDE
  \eqref{eqq} with coefficients \eqref{vp-coefs} and initial value
  $\mathbf{x}_s$ is
  \begin{equation}
    \label{analytic-vp} \mathbf{x}_t = \frac{\alpha_t}{\alpha_s} \mathbf{x}_s
    - 2 \alpha_t  \int_{\lambda_s}^{\lambda_t} e^{- \lambda}  \hat{F}_{\theta}
    (\widehat{\mathbf{x}}_{\lambda}, \lambda) \mathd \lambda - \sqrt{2}
    \alpha_t  \int_{\lambda_s}^{\lambda_t} e^{- \lambda} \mathd
    \tmmathbf{\bar{\omega}}_{\lambda}, \qquad \lambda_t : = - \log (\sigma_t)
    .
  \end{equation}
\end{proposition}
The change of variables of \eqref{analytic-vp} is interesting as it allows to
compute analytically the It{\^o}-Taylor coefficients in \eqref{taylor} by
using, for $h = \lambda_t - \lambda_s$, the following key result which will 
be used in Prop. \ref{Prop2}:
\begin{equation}
  \int^{\lambda_t}_{\lambda_s} e^{- \lambda} \frac{(\lambda -
  \lambda_s)^k}{k!} \mathd \lambda = \sigma_t h^{k + 1} \varphi_{k + 1} (h) .
  \label{phi-dev}
\end{equation}
For instance, in the case when $k = 0$, it is easy to see that
$\int_{\lambda_s}^{\lambda_t} e^{- \lambda} \mathd \lambda = \sigma_t  (e^h -
1)$ and $\int_{\lambda_s}^{\lambda_t} e^{- \lambda} \mathd
\tmmathbf{\bar{\omega}}_{\lambda}$ obeys a normal distribution with zero mean,
and one can analytically compute its variance:
\begin{equation}
  \int_{\lambda_s}^{\lambda_t} e^{- 2 \lambda} \mathd \lambda =
  \frac{\sigma^2_t}{2}  (e^{2 h} - 1) . \label{s-analytic}
\end{equation}
\paragraph{The EDM case.} Denote $\sigma^2_d$ the variance of the considered
initial dataset and set
\begin{equation}
  \label{edm-coefs} \sigma_t \assign t, \alpha_t : = 1, \nabla_{\mathbf{x}_t} \log p_t
(\mathbf{x}_t) \simeq \frac{1}{t^2} \left[ \frac{\sigma^2_d \mathbf{x}_t}{t^2
+ \sigma^2_d} + \frac{t \sigma_d}{\sqrt{t^2 + \sigma^2_d}} F_{\theta}  \left(
\frac{\mathbf{x}_t}{\sqrt{t^2 + \sigma^2_d}}, \frac{\log (t)}{4}  \right)
\right].
\end{equation}
These parameters correspond to the preconditioned EDM framework introduced in
{\cite[Sec. 5, App. B.6]{Karras2022edm}}. The following result is the
basis for constructing our customized SEEDS in this case, and for which we
report experimental results in Table \ref{table-fid-all}. For simplicity, we
will write $F_{\theta}  (\mathbf{x}_t, t)$ for the preconditioned model in
\eqref{edm-coefs} and we refer to Appendix \ref{app:design} for details.
\begin{proposition}
  \label{Prop1-edm}Let $t < s$. The analytic solution at time $t$ of
  \eqref{eqq} with coefficients \eqref{edm-coefs} and initial value
  $\mathbf{x}_s$ is, for $\ell = 1$,
  \begin{equation}
    \label{analytic-edm-s} \mathbf{x}_t = \frac{t^2 + \sigma^2_d}{s^2 +
    \sigma^2_d} \mathbf{x}_s + 2(t^2 + \sigma^2_d) 
    \int_{\lambda_s}^{\lambda_t} e^{- \lambda}  \hat{F}_{\theta}
    (\widehat{\mathbf{x}}_{\lambda}, \lambda) \mathd \lambda - \sqrt{2} (t^2 +
    \sigma^2_d) \int_{\lambda_s}^{\lambda_t} e^{- \lambda } \mathd
    \overline{\tmmathbf{\omega} }_{\lambda},
  \end{equation}
  where $\lambda_t : = - \log \left[ \frac{t}{\sigma_d \sqrt{t^2 +
  \sigma^2_d}} \right]$. In the case when $\ell = 0$, it is given by
  \begin{equation}
    \label{analytic-edm-d} \mathbf{x}_t = \sqrt{\frac{t^2 + \sigma^2_d}{s^2 +
    \sigma^2_d}} \mathbf{x}_s + \sqrt{t^2 + \sigma^2_d}
    \int_{\lambda_s}^{\lambda_t} e^{- \lambda}  \hat{F}_{\theta}
    (\widehat{\mathbf{x}}_{\lambda}, \lambda) \mathd \lambda, \quad \lambda_t
    : = - \log \left[ \arctan \left[ \frac{t}{\sigma_d} \right] \right] .
  \end{equation}
\end{proposition}

\begin{remark}
  One can wonder about the generality of such change of variables. Our method
  is very general in that one can always make such change of variables with
  very mild regularity conditions: for $c : [0, T] \longrightarrow
  \mathbb{R}^{> 0}$ integrable, with primitive $C (t) > 0$, we have $c (t) =
  e^{\log (c (t))}$. This means we can write $c (t) = \dot{C} (t) =
  e^{\lambda_t} \dot{\lambda}_t$ with $\lambda_t = \log (C (t)) .$ In other
  words, for such $c$, we have
  \[ \int^t_s c (\tau) \mathd \tau = \int^t_s e^{\lambda_{\tau}}
     \dot{\lambda}_{\tau} \mathd \tau = \int^{\lambda_t}_{\lambda_s}
     e^{\lambda} \mathd \lambda . \]
\end{remark}

\section{Higher Stage SEEDS for DPMs}\label{sec:serks}

In this section we present our SEEDS algorithms by putting together all the
ingredients presented in the previous section. Let $t < s$. In all what
follows, we consider the analytic solution at time $t$ of the RSDE \eqref{eqq}
with coefficients \eqref{vp-coefs}, $h = \lambda_t - \lambda_s$ and initial
value $\mathbf{x}_s$. Plugging \eqref{phi-dev} with $k = 0$ and
$\eqref{s-analytic}$ into the exact solution \eqref{analytic-vp} allow us to
infer the first SEEDS scheme, given by iterations of the form
\begin{equation}
  \tilde{\mathbf{x}}_t = \cfrac{\alpha_t}{\alpha_s}
  \tilde{\mathbf{x}}_s - 2 \bar{\sigma}_t  (e^h - 1) \hat{F}_{\theta}
  (\widehat{\mathbf{x}}_{\lambda_s}, \lambda_s) - \bar{\sigma}_t  \sqrt{e^{2
  h} - 1} \epsilon, \qquad \epsilon \sim \mathcal{N} (\tmmathbf{0},
  \mathbf{I}_d) . \label{seeds-1}
\end{equation}
The following Theorem gives strong order convergence guarantees for this method, 
which we call SEEDS-1, under mild conditions which apply to all our experiments. 
We stress out that its proof (App. \ref{app:proofs}) is a non-trivial result, 
it is fundamentally different in nature from \cite{dpm-solver} and involves 
mathematical tools which have no deterministic counterparts.
\begin{theorem}
  \label{theorem1}Under Assumption \ref{assumption-1}, the numerical solution
  $\tilde{\mathbf{x}}_t$ produced by the SEEDS-1 method \eqref{seeds-1}
  converges to the exact solution $\mathbf{x}_t$ of
  \begin{equation}
    \mathd \mathbf{x}_t = [f (t) \mathbf{x}_t + g^2 (t) \bar{\sigma}^{- 1}_t
    F_{\theta} (\mathbf{x}_t, t)] \mathd t + g (t) \mathd \tmmathbf{\omega}_t
    \label{nrsde-vp}, \qquad (\bar{\sigma}^{- 1}_t : = 1 / \bar{\sigma}_t)
  \end{equation}
  with coefficients \eqref{vp-coefs} in Mean-Square sense with strong order 1.0:
  there is a constant $C > 0$ such that
  \[ \sqrt{\mathbb{E} \left[ \underset{0 \leqslant t \leqslant 1}{\sup} |
     \tilde{\mathbf{x}}_t -\mathbf{x}_t |^2 \right]} \leqslant C h, \qquad
     \text{as } h \longrightarrow 0. \]
\end{theorem}
\paragraph{Higher stage SEEDS.} As announced, by fully exploiting
the analytic computations enabled by the expansion \eqref{phi-dev} we now turn
into crafting our multi-step SEEDS. Usually, SDE solvers are constructed by
using the full It{\^o}-Taylor expansion of the SDE solutions and usually need
a big number of evaluations of the network $\hat{F}_{\theta}$ to achieve
higher order of convergence. As our main concern is to present stochastic
solvers with a minimal amount of NFE, we choose to truncate such
It{\^o}-Taylor expansion so that the neural networks only appear in the
deterministic contributions.

\begin{proposition}
  \label{Prop2}Assume that $\hat{F}_{\theta}$ is a $\mathcal{C}^{2 n +
  1}$-function with respect to $\lambda$. Then the truncated It{\^o}-Taylor
  expansion of \eqref{analytic-vp} reads, for $\epsilon \sim \mathcal{N}
  (\tmmathbf{0}, \mathbf{I}_d)$,
  \begin{equation}
    \label{sol:ana-vp} \mathbf{x}_t = \cfrac{\alpha_t}{\alpha_s} \mathbf{x}_s
    - 2 \bar{\sigma}_t  \sum_{k = 0}^n h^{k + 1} \varphi_{k + 1} (h)
    \hat{F}_{\theta}^{(k)} (\widehat{\mathbf{x}}_{\lambda_s}, \lambda_s) -
    \bar{\sigma}_t  \sqrt{e^{2 h} - 1} \epsilon +\mathcal{R}_{n + 1},
  \end{equation}
  with $\hat{F}_{\theta}^{(k)} (\widehat{\mathbf{x}}_{\lambda}, \lambda) =
  L^k_{\lambda} \hat{F}_{\theta} (\widehat{\mathbf{x}}_{\lambda}, \lambda)$,
  with $L_{\lambda}$ is an infinitesimal operator defined in Appendix
  \ref{app:inf-op} and $\mathcal{R}_{n + 1}$ consists on the usual
  deterministic residue and all iterated integrals of length at greater or
  equal to 2 in which there is at least one stochastic component among them.
\end{proposition}

Our approach for constructing derivative-free 2-stage and 3-stage SEEDS
schemes consists on exploiting the analytic computation of the It{\^o}-Taylor
coefficients in Proposition \ref{Prop2} and replace the
$\hat{F}_{\theta}^{(k)} (\widehat{\mathbf{x}}_{\lambda}, \lambda)$ terms by
well-adapted correction terms which \tmtextit{do not need any derivative
evaluation} and dropping the $\mathcal{R}_{n + 1}$ contribution as in the Runge-Kutta approach. 

\paragraph{Markov-preserving noise decomposition.} We use collocation methods for 
constructing higher-stage derivative-free solvers. Although the chosen truncated 
Itô-Taylor expansion produces approximations for the deterministic integral 
similar to \cite{dpm-solver}, adding the corresponding noise contribution found 
by the SETD method at each step does not yield Markov chains in general. The 
reason is that stochastic integrals on overlapping time intervals need to be 
dependent, a phenomenon that has no deterministic counterpart. As such, our last 
and key element to construct SEEDS consists on a novel decomposition of 
stochastic integrals which enforces the Markov property for multi-stage SEEDS.

Algorithms \ref{alg:iter} to \ref{alg:SERK-solver-3} prescribe all SEEDS schemes 
obtained by this procedure in the VP case. We now show (see App. \ref{app:proofs} 
for the proofs) that all methods yield Markov chains and are weakly convergent.

\begin{figure*}[t]
    \begin{minipage}[t]{0.46\textwidth}
\begin{algorithm}[H]
   \caption{Iterative procedure}
   \label{alg:iter}
\small    
\begin{algorithmic}
   \State {\bfseries Input:} initial value $\x_T$, steps $\{t_i\}^M_{i=0}$, model $F_\theta$
   \State Initialize $\tilde\x_{t_0}\gets\x_T$
   \For{$i=1$ {\bfseries to} $M-1$}
   \State $(t,s) \gets (t_{i},t_{i-1}), \quad h \gets \lambda_{t} - \lambda_{s}$
   \State$\tilde\x_{t} = \text{SEEDS-k}(F_\theta,\tilde\x_{s}, s, t)$
   \EndFor
   \State Return $\tilde\x_{t_M}\gets \text{last-step}(\tilde\x_{t_{M-1}}, t_{M-1}, t_M)$
\end{algorithmic}
\end{algorithm}
\end{minipage}
\hfill
\begin{minipage}[t]{0.53\textwidth}
\begin{algorithm}[H]
   \caption{SEEDS-1$(F_\theta, \tilde\x_{s}, s, t)$}
   \label{alg:SERK-solver-1}
\small   
\begin{algorithmic}
   \State $z \gets \mathcal{N}(0,1)$
   \State $\tilde\x_{t} \gets \frac{\alpha_{t}}{\alpha_{s}} 
                     \tilde\x_{s} - 2\bar\sigma_{t}\left(e^{h} - 1\right)
                     F_\theta(\tilde \x_{s},s) 
                     - \bar\sigma_{t}\sqrt{e^{2 h} - 1}z $
\end{algorithmic}
\end{algorithm}
\end{minipage}
\end{figure*}
\begin{algorithm}[t]
   \caption{SEEDS-2$(F_\theta,\tilde\x_{s}, s, t)$}
   \label{alg:SERK-solver-2}
\small 
\begin{algorithmic}
   \State $s_{1} \gets t_{\lambda}(\lambda_{s} + \frac{h}{2}), \quad (z^1,z^2) \gets \mathcal{N}(0,\text{Id})\otimes \mathcal{N}(0,\text{Id})$
   \State $\mathbf{u} \gets \frac{\alpha_{s_1}}{\alpha_{s}} \tilde\x_{s} - 
                     2\bar\sigma_{s_1}\left(e^{\frac{h}{2}} - 1\right)F_\theta(\tilde \x_{s},s)
                      - \Av,\qquad \Av:=\bar\sigma_{s_1} \sqrt{e^{h}-1} z^1$
   \State $\tilde\x_{t} \gets \frac{\alpha_{t}}{\alpha_{s}} 
                     \tilde\x_{s} - 2\bar\sigma_{t}(e^{h} - 1)F_\theta(\mathbf{u},s_1)
                     - \Bv,\qquad \Bv:=\bar\sigma_{t}(\sqrt{e^{2h}-e^{h}} z^1+\sqrt{e^{h}-1} z^2)$                   
\end{algorithmic}
\end{algorithm}
\begin{algorithm}[t!]
   \caption{SEEDS-3$(F_\theta,\tilde\x_{s}, s, t)$ with $0 < r_1 < r_2 < 1 $}
   \label{alg:SERK-solver-3}
\small 
\begin{algorithmic}
   \State $s_{1} \gets t_\lambda\left(\lambda_{s} + r_1 h\right), 
             \quad s_{2} \gets t_\lambda\left(\lambda_{s} + r_2 h\right), \quad (z^1,z^2,z^3) \gets \mathcal{N}(0,\text{Id})^{\otimes3}$
   \State $\mathbf{u}_{1} \gets \frac{\alpha_{s_{1}}}{\alpha_{s}} \tilde\x_{s} 
             - 2\bar\sigma_{s_{1}}\left(e^{r_1 h} - 1\right)F_\theta(\tilde \x_{s},s)
             - \bar\sigma_{s_{1}} \sqrt{e^{2r_1 h}-1} z^1$  
   \State $\Av \gets \bar\sigma_{s_{2}}(\sqrt{e^{2r_2h}-e^{2r_1h}} z^1+\sqrt{e^{2r_1h}-1} z^2)$
   \State $\mathbf{u}_{2} \gets \frac{\alpha_{s_{2}}}{\alpha_{s}} \tilde\x_{s} 
         - 2\bar\sigma_{s_{2}}\left(e^{r_2 h} - 1\right)F_\theta(\tilde \x_{s},s)
         - 2\frac{\bar\sigma_{s_{2}}r_2}{r_1}\left( \frac{e^{r_2 h} - 1}{r_2 h} 
         - 1\right)(F_\theta(\mathbf{u}_{1},s_{1}) - F_\theta(\tilde \x_{s},s))) - \Av$
   \State $\Bv \gets \bar\sigma_{t}(\sqrt{e^{2h}-e^{2 r_2 h}} z^1
             + \sqrt{e^{2 r_2 h}-e^{2r_1h}} z^2 +\sqrt{e^{2r_1h}-1} z^3 )$ 
   \State $\tilde\x_{t} \gets \frac{\alpha_{t}}{\alpha_{s}} \tilde\x_{s}
             - 2\bar\sigma_{t}\left(e^{h} - 1\right)F_\theta(\tilde \x_{s},s)
             - 2\frac{\bar\sigma_{t}}{r_2}\left( \frac{e^{h} - 1}{h} - 1\right)(F_\theta(\mathbf{u}_{2},s_{2}) - F_\theta(\tilde \x_{s},s))) 
             - \Bv $
\end{algorithmic}
\end{algorithm}

\begin{proposition}\label{noise-dec}
  The sequences $\{ \tilde{\mathbf{x}}_t \}_t$ induced by the choice of
  stochastic noise contributions presented in Algorithms.\ref{alg:SERK-solver-2} and \ref{alg:SERK-solver-3} satisfy the Markov property.
\end{proposition}

\begin{corollary}
\label{cor:wocv}
  Under Assumption \ref{assumption-2}, the numerical solutions
  $\tilde{\mathbf{x}}_t$ produced by the SEEDS methods
  \eqref{alg:SERK-solver-2} and \eqref{alg:SERK-solver-3} converge to the
  exact solution $\mathbf{x}_t$ of \eqref{nrsde-vp} with coefficients
  \eqref{vp-coefs} in weak sense with global order 1 in both cases: there is a
  constant $C > 0$ such that, for any continuous bounded function $G$:
  \[ | \mathbb{E} [G (\tilde{\mathbf{x}}_{t_M})] -\mathbb{E} [G
     (\mathbf{x}_{t_M})] | \leqslant C h. \]
\end{corollary}

\paragraph{Comparison with existing sampling methods.} Let us now examine
the connection between SEEDS and existing sampling techniques used for DPMs,
emphasizing the contrasts between them.

The main distinctive feature of SEEDS is that they are \tmtextit{off-the-shelf} solvers. This means that, not only they are
\tmtextit{training-free}, contrary to {\cite{dockhorn2022genie}}, but they do
not require any kind of optimization procedure to achieve their optimal
results. This is in contrast to methods such as: gDDIM, which is training-free
but not off-the-shelf as one needs to make preliminary optimization procedures
such as simulating the transition matrix of their method in the CLD case; Heun-Like method
from EDM (for all baseline models and the EDM-optimized models for ImageNet)
since they need preliminary optimization procedures on 4 parameters which
actually break the convergence criteria. Moreover, neither gDDIM, EDM nor the SSCS method in \cite{dockhorn2022score}
present full proofs of convergence for their solvers. Also, both DEIS and
gDDIM identify their methods with stochastic DDIM theoretically, but the poor
results obtained by their stochastic solvers do not yield to further
experimentation in their works. In a way, SEEDS can be thought as improved and generalized DPM-Solver to
SDEs. Nevertheless, such generalization is not incremental as the tools for
proving convergence in our methods involve concepts which are exclusive to
SDEs. We now make rigorous statements of the above discussion.
\begin{proposition}
  \label{related}Consider the SEEDS approximation of \eqref{nrsde-vp} with
  coefficients \eqref{vp-coefs}. Then
  \begin{enumerate}
    \item If we set $g = 0$ in \eqref{nrsde-vp}, the resulting SEEDS
    do not yield DPM-Solver.
    
    \item If we parameterize \eqref{nrsde-vp} in terms of the data prediction
    model $D_{\theta}$, the resulting SEEDS are not equivalent to
    their noise prediction counterparts defined in Alg. \ref{alg:iter} to
    \ref{alg:SERK-solver-3}.
    
    \item The gDDIM solver \cite[Th. 1]{zhang2022gddim}
equals to SEEDS-1 in the data prediction mode, for $\ell = 1$.
  \end{enumerate}
\end{proposition}
The first point makes it explicit that SEEDS are not incremental
based on DPM-Solver. The second point in Prop. \ref{related} is analog to the
result in \href{https://openreview.net/pdf?id=4vGwQqviud5}{Appendix B} 
of {\cite{lu2022dpm}}, where the authors
compare DPM-Solver2 and DPM-Solver++(2S), that is the noise and data
prediction approaches, and find that they do not equate. The last point exhibits 
gDDIM as a special case of SEEDS-1 for isotropic DPMs.

\begin{remark}
  \label{cld-rem}
  Building solvers from the representation of the exact
  solution in \eqref{eq:NSDE} requires computing the transition
  matrix $\Phi_A (t, s)$, which cannot be analytically computed for
  non-isotropic DPMs such as CLD {\cite{dockhorn2022score}}. Nevertheless, the
  SEEDS approach can be applied in this scenario in at least two different
  ways.
  On the one hand, the SSCS method from {\cite{dockhorn2022score}} resides
  splitting $\Phi_A (t, s)$ into two separate terms. The first can be
  analytically computed. The second describes the evolution of a semi-linear
  differential equation {\cite[Eq. 92]{dockhorn2022score}}. While
  {\cite{dockhorn2022score}} approximates the latter by the Euler method,
  crafting exponential integrators for approximating such DE may yield an
  acceleration of the SSCS method.
  On the other hand, gDDIM {\cite{zhang2022gddim}} proposes an extension of
  DEIS sampling {\cite{zhang2022fast}} to CLD by setting a pre-sampling phase
  {\cite[App. C.4]{zhang2022gddim}} in which they compute an approximation of
  $\Phi_A (t, s)$ in order to apply their method, and the latter was shown in
  Prop. \ref{related}. to be a special case of our method. Unfortunately, the authors
  did not release pre-trained models in {\cite{zhang2022gddim}}, and the latter
  are not the same as those in {\cite{dockhorn2022score}}. Sampling in this 
  scenario may also benefit from our approach.
\end{remark}

\section{Experiments}
\label{sec:exps}

We compare SEEDS with several previous methods on discretely and continuously
pre-trained DPMs. We report results of many available sources, such as DDPM
{\cite{ho2020denoising}}, Analytic DDPM {\cite{bao2022analytic}}, PNDM
{\cite{liu2022pseudo}}, GGF {\cite{jolicoeur2021gotta}}, DDIM
{\cite{song2020denoising}}, gDDIM {\cite{zhang2022gddim}}, 
DEIS {\cite{zhang2022fast}} and DPM-Solver
{\cite{dpm-solver}}. Although we do not include training-based schemes here,
we still included GENIE {\cite{dockhorn2022genie}}, which trains a small
additional network but still solves the correct generative ODE at
higher-order. For each experiment, we compute the FID score for 50K sampled
images on multiple runs and report the minimum along different solvers. 
Details on model specifications and experiment illustrations are shown in Appendix
\ref{app:exp-details}.

\paragraph{Practical considerations.} For continuously trained models, SEEDS
use the EDM discretization {\cite[Eq. 5]{Karras2022edm}} with default
parameters and does \tmtextit{not} use the \tmtextit{last-step iteration
trick}, meaning that the last iteration of SEEDS is trivial. For discretely
trained models, SEEDS use the linear step schedule in the interval
$[\lambda_{t_0}, \lambda_{t_N}]$ interval following \ {\cite[Sec. 3.3,
3.4]{dpm-solver}}. All the reported SEEDS results were obtained using the
noise prediction mode. We conducted comparative experiments on SEEDS for both
the data and noise prediction modes and found better results with the latter
(see Tab. \ref{tab:grid-cifar-cont-dp-vs-np} for details). EDM solvers
{\cite[Alg. 2]{Karras2022edm}} depend on four parameters controlling the
amount of noise to be injected in a specific sub-interval of the iterative
procedure. We consider three scenarios: when stochasticity is injected along
all the iterative procedure, we denote it stochastic EDM, when no stochasticity
is injected we denote it EDM ($S_{\tmop{churn}} = 0$) and we denote EDM
(Optimized) the case where such parameters were subject to an optimization
procedure. To better evaluate sampling quality along the sample pre-trained
DPM, we recalculate DPM-Solver for sampling from the \tmtextit{non-deep} VP
continuous model on the CIFAR-10 dataset. All implementation details 
can be found in Appendix \ref{app:imple-details}.

\begin{table*}[t]
\caption{Sample quality measured by FID$\downarrow$ on pre-trained DPMs. We report 
the minimum FID obtained by each model and the NFE at which it was obtained. For 
CIFAR, CelebA and FFHQ, we use baseline pre-trained models 
\cite{song2020score,Karras2022edm}. For ImageNet, we use the optimized pre-trained 
model from \cite{Karras2022edm}.  *discrete-time model, $^{\star}$continuous-time 
model, $^\dagger$recomputed FID for the non-deep model.}
\label{table-fid-all}
\vskip 0.15in
\centering
\begin{small}
\begin{sc}
\begin{tabular}{cc}
{\begin{tabular}{lll}
    \toprule
  Sampling method & FID$\downarrow$ & NFE  \\
  \midrule
  CIFAR-10* vp-uncond. \\
  \hline
    DDIM {\cite{song2020denoising}} & 3.95 & 1000 \\
  Analytic-DDPM {\cite{bao2022analytic}} & 3.84 & 1000 \\
  GENIE {\cite{dockhorn2022genie}} & 3.64 & 25 \\
  Analytic-DDIM {\cite{bao2022analytic}} & 3.60 & 200 \\
  F-PNDM (linear) {\cite{liu2022pseudo}} & 3.60 & 250 \\
  DPM-Solver$^\dagger$ {\cite{dpm-solver}} & 3.48 & 44 \\
  F-PNDM (cosine) {\cite{liu2022pseudo}} & 3.26 & 1000 \\
  DDPM {\cite{ho2020denoising}} & 3.16 & 1000 \\
  SEEDS-3 (Ours) & \tmtextbf{3.08} & \tmtextbf{201} \\

  \hline
  CIFAR-10$^{\star}$ vp-cond. \\
  \hline
  DPM-solver$^\dagger$ {\cite{dpm-solver}} & 3.57 & 195 \\
  EDM ($S_{\textnormal{churn}}=0$) {\cite{Karras2022edm}} & 2.48 & 35 \\
  SEEDS-3 (Ours) & \tmtextbf{2.08} & \tmtextbf{129} \\
  
  \hline
  CIFAR-10$^{\star}$ vp-uncond. \\
  \hline
  DPM-Solver$^\dagger$ {\cite{dpm-solver}} & 2.59 & 51 \\
  GGF {\cite{jolicoeur2021gotta}} & 2.59 & 180 \\
  gDDIM {\cite{zhang2022gddim}} & 2.56 & 100 \\
  DEIS $\rho$3Kutta {\cite{zhang2022fast}} & 2.55 & 50 \\
  Euler-Maruyama {\cite{song2020score}} & 2.54 & 1024 \\
  Stochastic EDM {\cite{Karras2022edm}} & 2.54 &1534 \\
  SEEDS-3 (Ours) & 2.39 & 165 \\
  EDM (Optimized)  {\cite{Karras2022edm}} & \tmtextbf{2.27} & \tmtextbf{511} \\
  \bottomrule
\end{tabular}}
&
{\begin{tabular}{lll}
\toprule
  Sampling method & FID$\downarrow$ & NFE  \\
  \midrule
  CelebA-64* vp-uncond.\\
    \hline
    Analytic-DDPM {\cite{bao2022analytic}} & 5.21 & 1000 \\
   DDIM {\cite{song2020denoising}} & 4.78 & 200 \\
  DDPM {\cite{ho2020denoising}} & 3.50 & 1000 \\
  gDDIM {\cite{zhang2022gddim}} & 3.85 & 50 \\
    Analytic-DDIM {\cite{bao2022analytic}} & 3.13 & 1000 \\
3-DEIS {\cite{zhang2022fast}} & 2.95 & 50 \\
  F-PNDM (linear) {\cite{liu2022pseudo}} & 2.71 & 250 \\
  DPM-Solver {\cite{dpm-solver}} & 2.71 & 36 \\
  SEEDS-3 (Ours) & \tmtextbf{1.88} & \tmtextbf{90} \\
    \hline
  FFHQ-64$^{\star}$ vp-uncond.\\
    \hline
  DPM-Solver$^\dagger$  \cite{dpm-solver} & 3.52 & 90\\ 
SEEDS-3 (Ours) & 3.40 & 150\\
  EDM ($S_{\textnormal{churn}}=0$) \cite{Karras2022edm} & \textbf{3.39} & \textbf{79} \\

  \hline
  ImageNet-64 EDM-cond.\\
    \hline
    DPM-solver$^\dagger$ {\cite{dpm-solver}} & 3.01 & 270 \\
    EDM ($S_{\textnormal{churn}}=0$) {\cite{Karras2022edm}} & 2.22  & 511 \\
    SEEDS-3 (Ours) & 1.38 & 270 \\
    EDM (Optimized) {\cite{Karras2022edm}} & \textbf{1.36} & \textbf{511} \\
\bottomrule
\end{tabular}}
\end{tabular}
\end{sc}
\end{small}

\end{table*} 

 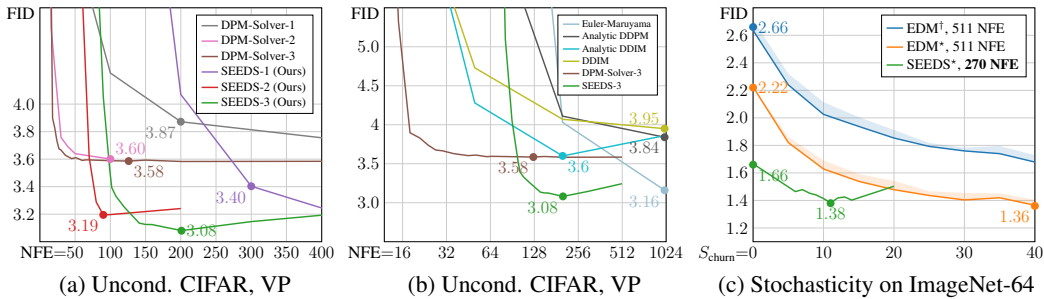
\begin{figure*}[t]
 \noindent\resizebox{\textwidth}{!}{
  \hspace*{-0.15cm}\input{tikz/cifar-uncond-disc.tex}
  \input{tikz/cifar-uncond-disc-1.tex}
 \hspace*{0.4cm}
  \input{tikz/ImageNet64.tex}
 }
 \\\hfill%
\makebox[0.33\linewidth]{\footnotesize (a) Uncond. CIFAR, VP  }\hfill%
\makebox[0.33\linewidth]{\hspace*{-.30em}\footnotesize (b) Uncond. CIFAR, VP}\hfill%
\makebox[0.33\linewidth]{\hspace*{-.15em}\footnotesize (c) Stochasticity on ImageNet-64}\hfill
\caption{\label{fig:fid-cont}\small{(a-b) Comparison of sample quality measured by 
FID $\downarrow$ of SEEDS, DPM-Solver and other methods for discretely trained DPMs on 
CIFAR-10 with varying number of function evaluations. (c) Effect of $S_{\text{churn}}$ on 
SEEDS-3 (at NFE = 270) and EDM method (at NFE = 511) on class-conditional ImageNet-64. 
$^\dagger$baseline ADM model. $^\star$EDM preconditioned model.} }
 \end{figure*}

\paragraph{Comparison with previous works.} In Table \ref{table-fid-all} we compare 
SEEDS with other sampling methods for pre-trained DPMs, and report the minimum FID 
obtained and their respective NFE. For each of the reported pre-trained models in 
CIFAR-10, CelebA-64 and ImageNet-64, SEEDS outperform all off-the-shelf methods in 
terms of quality with relatively low NFEs. For the discrete pre-trained DPM on 
CIFAR-10 (VP Uncond.) it is $\sim 5\times$ faster than the second best performant 
solver. Additionally, SEEDS remain competitive with the optimized EDM sampler. For 
ImageNet-64, it is nearly as good as the optimized EDM sampler while being almost 
twice faster than the latter.
Figure \ref{fig:fid-cont} (a) compares the FID score of SEEDS and DPM-Solver with 
varying NFEs. While DPM-Solver methods stabilize faster in a very low NFE regime, 
our methods eventually surpass them. Interestingly, after reaching their minimum, 
SEEDS methods tend to become worse at higher NFEs, a fact that is also visible in 
Figure \ref{fig:fid-cont} (b), where we notice that such phenomenon is also present 
on other SDE methods. We report in Appendix \ref{app:exp-details}, the results of 
our SEEDS methods in the low NFE regime
and connect their behavior with their proven convergence rate.

\paragraph{Combining SEEDS with other methods.} While being an off-the-shelf
method, SEEDS can be combined with the Churn-like method used in EDM incurring
into SDE solvers with an additional source of stochasticity. As done in
{\cite{Karras2022edm}}, we evaluate the effect of this second kind of
stochasticity, measured by a parameter denoted $S_{\tmop{churn}}$. Figure
\ref{fig:fid-cont} (c) shows that SEEDS and EDM show similar behavior, although
SEEDS-3 is twice faster, more sensitive to $S_{\tmop{churn}}$, and quickly achieves 
comparable performance to EDM. This indicates that SEEDS could possibly outperform 
EDM after a proper parameter optimization that will be left for future works. 
Nevertheless, we highlight the fact that obtaining such optimal parameters is 
costly and might scale poorly.

\input{plot_trajectory.tex}

\paragraph{Stiffness reduction with SEEDS.} 
In Fig. \ref{fig:stiffness}, we illustrate the impact of different choices of discretization
steps, noise schedule and dynamic scaling on SEEDS and stochastic EDM. We see
that choosing the EDM discretization over the linear one has the
effect of flattening the pixel trajectories at latest stages of the simulation
procedure. Also, choosing the parameters \eqref{edm-coefs} over those in
\eqref{vp-coefs} has the effect of greatly changing the distribution variances
as the trajectories evolve. Notice that all the SEEDS trajectories seem
perceptually more stable than those from EDM. It would be interesting
to relate this to the \tmtextit{stiffness} of the semi-linear DE describing these
trajectories, and to the magnitude of the parameters involved in the noise 
injection for EDM solver amplifying this phenomenon.

\paragraph{Ablation studies.} 
As said earlier, our principled use of the Chasles rule to enforce independence 
only on non-overlapping paths for SEEDS-2/3 ensures that the set of resulting 
iterations of our solvers satisfy the Markov property, is new and is the central 
key of success of SEEDS. To highlight this, we conduct an ablation study on how 4 
different combinations of the noise components $\Av$ and $\Bv$ in Algorithms 
\ref{alg:SERK-solver-2} and \ref{alg:SERK-solver-3} have an impact on the sampling 
quality of SEEDS.

For simplicity, we explain this for SEEDS-2 (Alg. \ref{alg:SERK-solver-2}). Set 
$(z^1,z^2,z^3)$ three independent standard Gaussian random variables. Denote 
$\Av=\bar\sigma_{s_1} \sqrt{e^h - 1} z^1$ for the noise contribution in the 
$\mathbf{u}$ term. We have the following choices for the noise contribution $\Bv$ 
in $\tilde{\mathbf{x}}_t$:
\begin{itemize}
    \item SEEDS-2-Correct: our noise combination $\Bv=\bar\sigma_t \left( \sqrt{e^{2 h} - e^h} z^1 + \sqrt{e^h - 1} z^2 \right)$
    \item SEEDS-2-Naive-1: one noise per stage $\Bv=\bar\sigma_t \left( \sqrt{e^{2 h} - e^h} + \sqrt{e^h - 1} \right) z^2$
    \item SEEDS-2-Naive-2: one noise per step $\Bv=\bar\sigma_t \left( \sqrt{e^{2 h} - e^h} + \sqrt{e^h - 1} \right) z^1$
    \item SEEDS-2-Naive-3: one noise per integral evaluation $\Bv=\bar\sigma_t \left( \sqrt{e^{2 h} - e^h} z^3 + \sqrt{e^h - 1} z^2 \right)$
    \item SEEDS-2-Naive-4: noises in inverse position $\Bv=\bar\sigma_t \left( \sqrt{e^{2 h} - e^h} z^2 + \sqrt{e^h - 1} z^1 \right)$.
\end{itemize}
Figure \ref{fig:fid-cifar10-cond-cont-noise-term-comparision} experimentally shows 
that all naive combinations of noises for SEEDS-2 and SEEDS-3 lead to FID/NFE 
curves with different behavior and a sharp drop of performance both in quality and 
speed in all of them, a phenomenon having no deterministic parallel.

\begin{figure*}[t]
 \noindent\resizebox{\textwidth}{!}{
  \hspace*{-0.15cm}
  \input{tikz/cifar-cond-cont-order-2-noise-term-comparision.tex}
 \hspace*{0.4cm}
  \input{tikz/cifar-cond-cont-order-3-noise-term-comparision.tex}
 }
 \\\hfill%
\makebox[0.49\linewidth]{\footnotesize (a) 2-stage SDE methods  }\hfill%
\makebox[0.49\linewidth]{\hspace*{-.30em}\footnotesize (b) 3-stage SDE methods}\hfill%
\caption{\label{fig:fid-cifar10-cond-cont-noise-term-comparision}Quality (measured 
by FID at increasing NFEs) comparison of SEEDS-2/3 with the enumerated ablation 
versions of it on CIFAR-10 in the (baseline) VP conditional framework. }
 \end{figure*}

\section{Conclusions}

Our focus is on addressing the challenge of training-free sampling from DPMs
without compromising sampling quality. To achieve this, we introduce SEEDS, an
off-the-shelf solution for solving diffusion SDEs. SEEDS capitalize on the
semi-linearity of diffusion SDEs by approximating a simplified formulation of
their exact solutions. Inspired by numerical methods for stochastic
exponential integrators, we propose three SEEDS schemes with proven convergence 
order. They transform the integrals involved in the exact solution into exponentially
weighted integrals, and estimate the deterministic one while analytically
computing the variance of the stochastic integral. We extend our approach to
handle other isotropic DPMs, and evaluate its performance on various benchmark
tests. Our experiments demonstrate that SEEDS can generate images of optimal
quality, outperforming existing SDE solvers while being $3 \sim 5 \times$
faster. \tmtextbf{Limitations and broader impact.} While SEEDS prioritize
optimal quality sampling, they may require substantial computational resources
and energy consumption, making them less suitable for scenarios where speed is
the primary concern. In such cases, alternative ODE methods may be more
appropriate. Additionally, as with other generative models, DPMs can be
employed to create misleading or harmful content, and our proposed solver
could inadvertently amplify the negative impact of generative AI for malicious
purposes.

\tmtextbf{Acknowledgments.} This work has been supported by the French government 
under the ”France 2030” program, as part of the SystemX Technological Research 
Institute. 

{
\small
\bibliographystyle{acm}
\bibliography{biblio}

}

\newpage

\tableofcontents
 
\appendix
\input{App-A}

\input{App-B}

\input{App-C}
\input{App-D}
\input{App-E}

\input{App-F}

\end{document}

%% file: preamble.tex
\usepackage[final,nonatbib]{neurips_2023}

\usepackage[utf8]{inputenc} 
\usepackage[T1]{fontenc}    
\usepackage{hyperref}       
\usepackage{url}            
\usepackage{booktabs}       
\usepackage{amsfonts}       
\usepackage{nicefrac}       
\usepackage{microtype}      
\usepackage{xcolor}         
\usepackage{wrapfig}
\usepackage{tikz}           
\usepackage{pgfplots,pgfplotstable} 

\usepackage{graphicx}
\usepackage{subfigure}
\usepackage{booktabs} 

\usepackage{rotating}

\usepackage{algpseudocode, algorithm}

\usepackage{amsmath}
\usepackage{amssymb}
\usepackage{mathtools}
\usepackage{amsthm}
\usepackage{bm}

\pgfplotsset{compat=1.18}	 %
\pgfplotsset{compat/show suggested version=false}
\definecolor{olive}{rgb}{0.5, 0.5, 0.0}
\definecolor{maroon}{rgb}{0.69, 0.19, 0.38}
\definecolor{celestialblue}{rgb}{0.29, 0.59, 0.82}
\definecolor{darkgreen}{rgb}{0.0, 0.6, 0.0}
\definecolor{grey}{rgb}{0.5,0.5,0.5}
\definecolor{darkblue}{rgb}{0.19, 0.19, 0.62}
\definecolor{silver}{rgb}{0.7,0.7,0.7}
\definecolor{darkcyan}{rgb}{0.0, 0.55, 0.55}

\usepackage[capitalize,noabbrev]{cleveref}

\newcommand{\vparagraph}[1]{\vspace*{-1mm}\paragraph{#1}}

\theoremstyle{plain}
\newtheorem{theorem}{Theorem}[section]
\newtheorem{proposition}[theorem]{Proposition}
\newtheorem{lemma}[theorem]{Lemma}
\newtheorem{corollary}[theorem]{Corollary}
\theoremstyle{definition}

\newtheorem{assumption}[theorem]{Assumption}
\theoremstyle{remark}
\newtheorem{remark}[theorem]{Remark}

\newcommand{\assign}{:=}

\newcommand{\mathd}{\mathrm{d}}
\newcommand{\tmcolor}[2]{{\color{#1}{#2}}}
\newcommand{\tmmathbf}[1]{\ensuremath{\boldsymbol{#1}}}
\newcommand{\tmop}[1]{\ensuremath{\operatorname{#1}}}
\newcommand{\tmtextbf}[1]{\text{{\bfseries{#1}}}}
\newcommand{\tmtextit}[1]{\text{{\itshape{#1}}}}

\newcommand{\vect}[1]{\bm{#1}}

\newcommand{\x}{\mathbf{x}}
\newcommand{\y}{\mathbf{y}}

\newcommand{\epsilonv}{\vect\epsilon}

\newcommand{\kv}{\vect k}

\newcommand{\uv}{\vect u}

\newcommand{\Av}{\vect A}
\newcommand{\Bv}{\vect B}

%% file: figures.tex
\input{data/TrainingPlotImgc}
\input{data/data_cifar10_cond_cont}

\input{data/data_cifar10_uncond_disc}
\input{data/data_cifar10_cond_cont_noise_term_comparison}
\input{data/data_cifar10_uncond_disc_1}

\usepgfplotslibrary{fillbetween}
\pgfplotsset{compat=1.15}
\pgfplotsset{xtick style={draw=none}}
\pgfplotsset{ytick style={draw=none}}

\pgfplotsset{major grid style={gray!40}}
\pgfplotsset{every axis plot/.style={line width=1.5pt, mark size=2pt}}
\pgfplotsset{legend image code/.code={\draw[mark repeat=2, mark phase=2] plot coordinates {(0cm, 0cm) (0.2cm, 0cm) (0.4cm, 0cm)};}} %
\pgfdeclarelayer{background}
\pgfdeclarelayer{foreground}
\pgfsetlayers{background,main,foreground}
\definecolor{C0}{rgb}{0.121569, 0.466667, 0.705882}
\definecolor{C1}{rgb}{1.000000, 0.498039, 0.054902}
\definecolor{C2}{rgb}{0.172549, 0.627451, 0.172549}
\definecolor{C3}{rgb}{0.839216, 0.152941, 0.156863}
\definecolor{C4}{rgb}{0.580392, 0.403922, 0.741176}
\definecolor{C5}{rgb}{0.549020, 0.337255, 0.294118}
\definecolor{C6}{rgb}{0.890196, 0.466667, 0.760784}
\definecolor{C7}{rgb}{0.498039, 0.498039, 0.498039}
\definecolor{C8}{rgb}{0.737255, 0.741176, 0.133333}
\definecolor{C9}{rgb}{0.090196, 0.745098, 0.811765}
\definecolor{C10}{rgb}{0.337255, 0.341176, 0.333333}
\definecolor{C11}{rgb}{0.590196, 0.745098, 0.811765}

\newcommand{\fillbetween}[3][]{\addplot+[name path=A, draw=none, mark=none, forget plot] #2; \addplot+[name path=B, draw=none, mark=none, forget plot] #3; \addplot[#1] fill between[of=A and B]}

\newcommand{\cs}{}
\newcommand{\hh}{0mm}
\newcommand{\hhh}{0mm}
\newcommand{\hhhh}{0mm}
\newcommand{\vvv}{0mm}
\newcommand{\vvvv}{0mm}

\newcommand{\vlabel}[3]{\makebox[0mm][l]{\rotatebox{90}{\makebox[#2][c]{#3}}}\hspace{#1}}
\newcommand{\hrlabel}[1]{\hfill\makebox[0mm]{#1}\hfill}
\newcommand{\vrlabel}[1]{\vfill\makebox[0mm]{#1}\vfill}

\newcommand{\Schurn}{S_\text{churn}}
\newcommand{\Stmin}{S_\text{tmin}}
\newcommand{\Stmax}{S_\text{tmax}}
\newcommand{\Snoise}{S_\text{noise}}
\newcommand{\StminStmax}{S_\text{tmin,tmax}}
\newcommand{\StminStmaxSnoise}{S_\text{tmin,tmax,noise}}

\newcommand{\tickYtop}[1]{\raisebox{0ex}[1ex][1ex]{#1}}
\newcommand{\tickYtopD}[1]{\raisebox{-1.5ex}[0ex][0ex]{#1}}
\newcommand{\tickFID}{\tickYtopD{FID}}
\newcommand{\tickLoss}{\tickYtopD{loss}}
\newcommand{\tickTau}{\tickYtopD{$\lVert\lte\rVert$}}
\newcommand{\tickNFE}[1]{$\mathllap{\smash{\text{NFE}{=}}}{#1}$}
\newcommand{\tickSchurn}[1]{$\mathllap{\smash{\Schurn{=}}}{#1}$}
\newcommand{\tickSTEP}[1]{$\mathllap{\smash{\text{STEP}{=}}}{#1}$}
\newcommand{\tickSchurnB}[1]{\smash{$\Schurn{=}$}{$#1$}\hspace*{2em}}
\newcommand{\tickRho}[1]{$\mathllap{\smash{\rho{=}}}{#1}$}
\newcommand{\tickSigma}[1]{$\mathllap{\smash{\sigma{=}}}{#1}$}

\newcommand{\atphantom}{\vphantom{${}^2$}}
\newcommand{\AProcedure}[2]{\Procedure{\smash{#1}}{\smash{#2}}}
\newcommand{\AComment}[1]{\Comment{\smash{#1}}}
\newcommand{\AState}[1]{\State{\smash{#1}}}
\newcommand{\AFor}[1]{\For{\smash{#1}}}
\newcommand{\AIf}[1]{\If{\smash{#1}}}
\newcommand{\DState}[1]{\State{#1 \vphantom{$\displaystyle\Bigg)$}}}
\newcommand{\MState}[1]{\State\raisebox{0mm}[3.2ex][1.8ex]{#1}}

%% file: data/TrainingPlotImgc.tex
\newcommand{\genTrainingPlotImgcOrig}{%
(0, 2.6363)
(5, 2.2398)
(10, 2.0252)
(15, 1.9378)
(20, 1.8534)
(25, 1.7925)
(30, 1.7597)
(35, 1.7399)
(40, 1.6788)
(45, 1.6419)
(50, 1.6459)
(55, 1.6285)
(60, 1.618)
(65, 1.6022)
(70, 1.6166)
(75, 1.5726)
(80, 1.5874)
(85, 1.6096)
(90, 1.5889)
(95, 1.6165)
(100, 1.5827)
}

\newcommand{\genTrainingPlotImgcOrigLo}{%
(0, 2.6363)
(5, 2.2398)
(10, 2.0252)
(15, 1.9378)
(20, 1.8534)
(25, 1.7925)
(30, 1.7597)
(35, 1.7399)
(40, 1.6788)
(45, 1.6419)
(50, 1.6459)
(55, 1.6285)
(60, 1.618)
(65, 1.6022)
(70, 1.6166)
(75, 1.5726)
(80, 1.5874)
(85, 1.6096)
(90, 1.5889)
(95, 1.6165)
(100, 1.5827)
}

\newcommand{\genTrainingPlotImgcOrigHi}{%
(0, 2.715)
(5, 2.3212)
(10, 2.1144)
(15, 2.0056)
(20, 1.9126)
(25, 1.8176)
(30, 1.7838)
(35, 1.7964)
(40, 1.7284)
(45, 1.7034)
(50, 1.6926)
(55, 1.6852)
(60, 1.6732)
(65, 1.6715)
(70, 1.6736)
(75, 1.6274)
(80, 1.6567)
(85, 1.6528)
(90, 1.6198)
(95, 1.6314)
(100, 1.6554)
}

\newcommand{\genTrainingPlotImgcEDM}{%
(0, 2.2246)
(5, 1.8202)
(10, 1.6285)
(15, 1.5386)
(20, 1.478)
(25, 1.436)
(30, 1.4037)
(35, 1.4192)
(40, 1.365)
(45, 1.3639)
(50, 1.3651)
(55, 1.3823)
(60, 1.3693)
(65, 1.3878)
(70, 1.3801)
(75, 1.3977)
(80, 1.4338)
(85, 1.449)
(90, 1.4653)
(95, 1.4788)
(100, 1.4859)
}

\newcommand{\genTrainingPlotImgcSEEDS}{%
(0, 1.66643)
(6, 1.46441)
(7, 1.47835)
(8, 1.45303)
(9, 1.4388)
(10, 1.41371)
(11, 1.38183)
(12, 1.41755)
(13, 1.42523)
(14, 1.4023)
(20, 1.50312)
}

\newcommand{\genTrainingPlotImgcEDMLo}{%
(0, 2.2246)
(5, 1.8202)
(10, 1.6285)
(15, 1.5386)
(20, 1.478)
(25, 1.436)
(30, 1.4037)
(35, 1.4192)
(40, 1.365)
(45, 1.3639)
(50, 1.3651)
(55, 1.3823)
(60, 1.3693)
(65, 1.3878)
(70, 1.3801)
(75, 1.3977)
(80, 1.4338)
(85, 1.449)
(90, 1.4653)
(95, 1.4788)
(100, 1.4859)
}

\newcommand{\genTrainingPlotImgcEDMHi}{%
(0, 2.2246)
(5, 1.8621)
(10, 1.6915)
(15, 1.5913)
(20, 1.5428)
(25, 1.4666)
(30, 1.4549)
(35, 1.4505)
(40, 1.4049)
(45, 1.4137)
(50, 1.4202)
(55, 1.446)
(60, 1.4265)
(65, 1.4225)
(70, 1.4526)
(75, 1.4308)
(80, 1.4598)
(85, 1.4863)
(90, 1.5)
(95, 1.5114)
(100, 1.514)
}

\newcommand{\genTrainingPlotImgcOrigMarks}{%
(75, 1.5726)
}

\newcommand{\genTrainingPlotImgcEDMMarks}{%
(40, 1.365)
}

%% file: data/data_cifar10_cond_cont.tex
\newcommand{\drawSeedsThree}{%
(36, 226.734)
(39, 189.543)
(42, 151.946)
(60, 25.0673)
(90, 3.19528)
(120, 2.17488)
(123, 2.15103)
(126, 2.12534)
(129, 2.082)
(150, 2.1534)
(180, 2.199)
}

\newcommand{\drawSeedsThreeMarks}{%
(129, 2.082)
}

\newcommand{\drawSeedsTwo}{%
(8, 438.937)
(18, 308.343)
(28, 210.349)
(48, 18.678)
(68, 3.303)
(88, 2.332)
(108, 2.349)
(148, 2.365)
}

\newcommand{\drawSeedsOne}{%
(5, 396.04)
(10, 281.30)
(15, 257.426)
(25, 117.109)
(35, 57.267)
(45, 32.584 )
(55, 21.143)
(100, 6.413)
(120, 4.912)
(150, 3.618)
(200, 2.8152)
(300, 2.349)
}

\newcommand{\drawDpmThree}{%
(9, 50.5133)
(18, 11.8015)
(30, 6.18883)
(48, 4.74407)
(66, 4.27767)
(87, 3.99994)
(105, 3.86144)
(147, 3.68234)
(195, 3.57076)
}

\newcommand{\drawDpmTwo}{%
(8, 16.768)
(18, 4.66595)
(28, 3.87879)
(48, 3.72524)
(68, 3.65129)
(88, 3.61933)
(108, 3.58095)
(148, 3.52136)
}

\newcommand{\drawDpmOne}{%
(9, 15.062)
(14, 9.71389)
(24, 6.22725)
(34, 4.95124)
(44, 4.29372)
(54, 3.92382)
(74, 3.5164)
(99, 3.29297)
}

%% file: data/data_cifar10_uncond_disc.tex
\newcommand{\drawCifarSeedsThree}{%
(9	,	483.044)
(12	,	482.192)
(15	,	479.639)
(21	,	462.612)
(30	,	280.484)
(51	,	62.620)
(54	,	43.352)
(60	,	21.592)
(66	,	12.448)
(81	,	5.093)
(90	,	4.051)
(99	,	3.531)
(102	,	3.393)
(108	,	3.33)
(120	,	3.25)
(141	,	3.135)
(150	,	3.1254)
 (159	,	3.1247)
(201	,	3.081)
(300	,	3.146)
(510, 3.244)
}

\newcommand{\drawCifarSeedsTwo}{%
(10	,	481.097)
(12	,	473.484)
(16	,	430.986)
(20		,305.885)
(30		,223.019)
(40		,51.435)
(50		,11.100)
(60		,4.823)
(70		,3.611)
(80		,3.288)
(90		,3.195)
(100		,3.197)
(200		,3.241)
}

\newcommand{\drawCifarSeedsOne}{%
(10	,	303.481)
(12	,	239.799)
(15	,	279.846)
(20	,	192.683)
(30	,	84.781)
(40	,	45.265)
(50	,	28.186)
(100	,	8.244)
(200	,	4.070)
(300	,	3.403)
(400	,	3.245)
(500	,	3.138)
(700	,	3.159)
}

\newcommand{\drawCifarDpmsolverThree}{%
(9	,	66.925)
(12	,	9.722)
(15	,	5.320)
(18	,	3.897)
(21	,	3.830)
(24	,	3.737)
(27	,	3.676)
(30	,	3.666)
(36	,	3.629)
(45	,	3.604)
(51	,	3.611)
(60	,	3.591)
(66	,	3.595)
(99	,	3.589)
(126, 3.586)
(150	,	3.593)
(201, 3.583)
(510, 3.585)
}

\newcommand{\drawCifarDpmsolverTwo}{%
(10	,	12.225)
(12	,	6.523)
(16	,	4.550)
(30	,	3.757)
(40	,	3.689)
(50	,	3.640)
(100,	3.601)
}

\newcommand{\drawCifarDpmsolverOne}{%
(10	,	22.907)
(12	,	17.735)
(15	,	13.367)
(20	,	9.782)
(30	,	6.879)
(40	,	5.775)
(50	,	5.173)
(100	,	4.226)
(200	,	3.872)
(500	,	3.697)
}

%% file: data/data_cifar10_cond_cont_noise_term_comparison.tex
\newcommand{\drawSeedsTwoCorrect}{%
(8,	 438)
(18, 308)
(28, 210)
(48, 18.678)
(68, 3.303)
(88, 2.332)
(108, 2.349)
(148, 2.362)
}

\newcommand{\drawSeedsTwoNaiveOne}{%
(8, 447.66)
(18 , 372.526)
(28 ,	334.121)
(48 ,	289.769)
(68 ,	275.999)
(88 ,	270.111)
(108 ,	266.524)
(148 ,	261.243)
}

\newcommand{\drawSeedsTwoNaiveTwo}{%
(8,  430.34)
(18, 285.465)
(28, 45.10)
(48, 124.512)
(68, 173.594)
(88, 193.452)
(108,203.791)
(148,213.896)
}
\newcommand{\drawSeedsTwoNaiveThree}{%
(8, 447.408)
(18, 349.188)
(28, 296.88)
(48, 268.774)
(68, 206.935)
(88, 152.959)
(108,113.053)
(148,65.0644)
(160 ,	106.735)
}
\newcommand{\drawSeedsTwoNaiveFour}{%
(8,	442.645)
(18,339.442)
(28,299.596)
(48,124.014)
(68,43.4791)
(88,23.8983)
(108,20.3142)
(148,18.5319)
}

\newcommand{\drawSeedsThreeCorrect}{%
(36, 226.734)
(39, 189.543)
(42, 151.946)
(60, 25.0673)
(90, 3.19528)
(129, 2.082)
(180, 2.199)
}

\newcommand{\drawSeedsThreeNaiveOne}{%
(36, 374.561)
(39, 370.756)
(42, 367.318)
(60, 380.241)
(90, 331.133)
(120, 319.557)
(123, 318.688)
(126, 317.934)
(129, 316.989)
(150, 313.476)
(180,307.622)
}
\newcommand{\drawSeedsThreeNaiveTwo}{%
(36, 	276.776)
(39, 	276.011)
(42, 	275.881)
(60, 	276.383)
(90, 	277.83)
(120,	279.347)
(123,	279.381)
(126,	279.34)
(150,	280.238)
(180,	280.879)
}
\newcommand{\drawSeedsThreeNaiveThree}{%
(36, 	373.759)
(42, 	290.604)
(60, 	283.44)
(90, 	225.167)
(120,	162.666)
(150,	149.55)
(180,	90.80)
}
\newcommand{\drawSeedsThreeNaiveFour}{%
(36, 	290.231)
(60,    101.304)
(90, 	20.3663)
(120,	6.6872)
(150,	3.08)
(180,	2.43)
}

%% file: data/data_cifar10_uncond_disc_1.tex
\newcommand{\drawCifarEulerMaruyama}{%
(10, 278.67)
 (12, 246.29)
 (15, 197.63)
 (20, 137.34)
 (50, 32.63)
 (200, 4.03)
 (1000, 3.16)
}

\newcommand{\drawCifarEulerMaruyamaMarks}{%
(999, 3.16)
}

\newcommand{\drawCifarAddpm}{%
(10, 35.03)
 (12, 27.69)
 (15, 20.82)
 (20, 15.35)
 (50, 7.34)
 (200, 4.11)
 (1000, 3.84)
}

\newcommand{\drawCifarAddpmMarks}{%
(999, 3.84)
}

\newcommand{\drawCifarAddim}{%
(10, 14.74)
 (12, 11.68)
 (15, 9.16)
  (20, 7.2)
 (50, 4.28)
 (200, 3.6)
 (1000, 3.86)
}

\newcommand{\drawCifarAddimMarks}{%
(200, 3.6)
}

\newcommand{\drawCifarDdim}{%
(10, 13.58)
 (12, 11.02)
 (15, 8.92)
  (20, 6.94)
 (50, 4.73)
 (200, 4.07)
 (1000, 3.95)
}

\newcommand{\drawCifarDdimMarks}{%
(999, 3.95)
}

\newcommand{\drawCifarSeeds}{%
(9	,	483.044)
(12	,	482.192)
(15	,	479.639)
(21	,	462.612)
(30	,	280.484)
(51	,	62.620)
(54	,	43.352)
(60	,	21.592)
(66	,	12.448)
(81	,	5.093)
(90	,	4.051)
(99	,	3.531)
(102	,	3.393)
(108	,	3.33)
(120	,	3.25)
(141	,	3.135)
(150	,	3.1254)
 (159	,	3.1247)
(201	,	3.081)
(300	,	3.146)
(510, 3.244)
}

\newcommand{\drawCifarSeedsMarks}{%
(149, 3.09)
}

\newcommand{\drawCifarDpmsolver}{%
(9	,	66.925)
(12	,	9.722)
(15	,	5.320)
(18	,	3.897)
(21	,	3.830)
(24	,	3.737)
(27	,	3.676)
(30	,	3.666)
(36	,	3.629)
(45	,	3.604)
(51	,	3.611)
(60	,	3.591)
(66	,	3.595)
(99	,	3.589)
(126, 3.586)
(150	,	3.593)
(201, 3.583)
(510, 3.585)
}

\newcommand{\drawCifarDpmsolverMarks}{%
(44, 3.48)
}

%% file: tikz/cifar-uncond-disc.tex
\begin{tikzpicture}
\begin{axis}[
  xmin={0}, xmax={400}, domain={0:400}, xmode={linear}, xtick={0, 50, 100, 150, 200, 250, 300, 350, 400}, xticklabels={, \tickNFE{50}, $100$, $150$, $200$, $250$, $300$, $350$, $400$},
  ymin={3}, ymax={4.7}, restrict y to domain*=3:16, ymode={linear}, ytick={3.2, 3.4, 3.6, 3.8, 4.0, 4.7}, yticklabels={$3.2$, $3.4$, $3.6$, $3.8$, $4.0$, \tickFID},
  grid={major}, legend pos={north east}, legend cell align={left}, legend style={nodes={scale=0.87, transform shape}, font=\normalsize, /tikz/every even column/.append style={column sep=1.6mm}},
  label style={font=\large},
        tick label style={font=\large} 
]
\addplot[C7, line width=1pt] coordinates {\drawCifarDpmsolverOne};
\addplot[C6, line width=1pt] coordinates {\drawCifarDpmsolverTwo};
\addplot[C5, line width=1pt] coordinates {\drawCifarDpmsolverThree};
\addplot[C4, line width=1pt] coordinates {\drawCifarSeedsOne};
\addplot[C3, line width=1pt] coordinates {\drawCifarSeedsTwo};
\addplot[C2, line width=1pt] coordinates {\drawCifarSeedsThree};

\addplot[C7, only marks, forget plot]  coordinates {(200,3.872)};\node at (axis cs:(200,3.872) [anchor={north east}] {\textcolor{C7}{\large $3.87$}};
\addplot[C6, only marks, forget plot]  coordinates {(100,	3.601)};\node at (axis cs:100,	3.601) [anchor={south west}] {\textcolor{C6}{\large $3.60$}};
\addplot[C5, only marks, forget plot]  coordinates {(126, 3.586)};\node at (axis cs:126, 3.586) [anchor={north west}] {\textcolor{C5}{\large $3.58$}};
\addplot[C4, only marks, forget plot]  coordinates {(300,3.403)};\node at (axis cs:300	,3.403) [anchor={north east}] {\textcolor{C4}{\large $3.40$}};
\addplot[C3, only marks, forget plot]  coordinates {(90,3.195)};\node at (axis cs:90,3.195) [anchor={north east}] {\textcolor{C3}{\large $3.19$}};
\addplot[C2, only marks, forget plot]  coordinates {(201	,	3.08)};\node at (axis cs:201,3.081) [anchor={west}] {\textcolor{C2}{\large $3.08$}};

\legend{
  {DPM-Solver-1},
  {DPM-Solver-2},
  {DPM-Solver-3},
  {SEEDS-1 (Ours)},
  {SEEDS-2 (Ours)},
  {SEEDS-3 (Ours)}
}
\end{axis}
\end{tikzpicture}

%% file: tikz/cifar-uncond-disc-1.tex
\begin{tikzpicture}
\begin{axis}[
  xmin={12}, xmax={1024}, domain={12:1024}, xmode={log}, xtick={12, 16, 32, 64, 128, 256, 512, 1024}, xticklabels={, \tickNFE{16}, $32$, $64$, $128$, $256$, $512$, $1024$},
  ymin={2.5}, ymax={5.5},restrict y to domain*=2.5:16,  ymode={linear}, ytick={3.0, 3.5, 4, 4.5, 5, 5.5}, yticklabels={$3.0$, $3.5$, $4$, $4.5$, $5.0$, \tickFID},
  grid={major}, legend pos={north east}, legend cell align={left}, legend style={nodes={scale=0.7, transform shape}, font=\normalsize, /tikz/every even column/.append style={column sep=1.6mm}},
  label style={font=\large},
        tick label style={font=\large} 
]
\addplot[C11, line width=1pt] coordinates {\drawCifarEulerMaruyama};
\addplot[C10, line width=1pt] coordinates {\drawCifarAddpm};
\addplot[C9, line width=1pt] coordinates {\drawCifarAddim};
\addplot[C8, line width=1pt] coordinates {\drawCifarDdim};
\addplot[C5, line width=1pt] coordinates {\drawCifarDpmsolver};
\addplot[C2, line width=1pt] coordinates {\drawCifarSeeds};

\addplot[C11, only marks, forget plot]  coordinates {(999, 3.16)};
\node at (axis cs:999,3.16) [anchor={north east}] {\textcolor{C11}{\large$3.16$}};
\addplot[C10, only marks, forget plot]  coordinates {(999, 3.84)};
\node at (axis cs:999, 3.84) [anchor={north east}] {\textcolor{C10}{\large$3.84$}};
\addplot[C9, only marks, forget plot]  coordinates {(200, 3.6)};
\node at (axis cs:200, 3.6) [anchor={north west}] {\textcolor{C9}{\large$3.6$}};
\addplot[C8, only marks, forget plot]  coordinates {(999, 3.95)};
\node at (axis cs:999, 3.95) [anchor={south east}] {\textcolor{C8}{\large$3.95$}};
\addplot[C5, only marks, forget plot]  coordinates {(126, 3.586)};
\node at (axis cs:126, 3.586) [anchor={north east}] {\textcolor{C5}{\large$3.58$}};
\addplot[C2, only marks, forget plot]  coordinates {(201, 3.081)};
\node at (axis cs:201, 3.081) [anchor={north east}] {\textcolor{C2}{\large$3.08$}};
\legend{
  {Euler-Maruyama},
  {Analytic DDPM},
  {Analytic DDIM},
  {DDIM},
  {DPM-Solver-3},
  {SEEDS-3}
}
\end{axis}

\end{tikzpicture}

%% file: tikz/ImageNet64.tex
\begin{tikzpicture}
\begin{axis}[
  xmin={0}, xmax={40}, xmode={linear}, xtick={0, 10, 20, 30, 40}, xticklabels={\tickSchurn{0}, $10$, $20$, $30$, $40$},
  ymin={1.1}, ymax={2.8}, ymode={linear}, ytick={1.2, 1.4, 1.6, 1.8, 2.0, 2.2, 2.4, 2.6, 2.8}, yticklabels={$1.2$, $1.4$, $1.6$, $1.8$, $2.0$, $2.2$, $2.4$, $2.6$, \tickFID},
  grid={major}, legend pos={north east}, legend cell align={left}, legend style={font=\normalsize, /tikz/every even column/.append style={column sep=1.6mm}},
  label style={font=\large},
        tick label style={font=\large} 
]
\fillbetween[C0, opacity=0.15, forget plot]{coordinates {\genTrainingPlotImgcOrigLo}}{coordinates {\genTrainingPlotImgcOrigHi}};
\fillbetween[C1, opacity=0.15, forget plot]{coordinates {\genTrainingPlotImgcEDMLo}}{coordinates {\genTrainingPlotImgcEDMHi}};

\addplot[C0, line width=1pt] coordinates {\genTrainingPlotImgcOrig};
\addplot[C1, line width=1pt] coordinates {\genTrainingPlotImgcEDM};
\addplot[C2, line width=1pt] coordinates {\genTrainingPlotImgcSEEDS};

\addplot[C1, only marks, forget plot] coordinates {(40, 1.36)}; 
\node at (axis cs:40,1.36) [anchor={north east}] {\textcolor{C1}{\large$1.36$}};
\addplot[C2, only marks, forget plot] coordinates {(11, 1.38)}; 
\node at (axis cs:11,1.38) [anchor={north}] {\textcolor{C2}{\large$1.38$}};
\addplot[C0, only marks, forget plot] coordinates {(0, 2.66)}; 
\node at (axis cs:0,2.66) [anchor={west}] {\textcolor{C0}{\large$2.66$}};
\addplot[C1, only marks, forget plot] coordinates {(0, 2.22)}; 
\node at (axis cs:0,2.22) [anchor={west}] {\textcolor{C1}{\large$2.22$}};
\addplot[C2, only marks, forget plot] coordinates {(0, 1.66)}; 
\node at (axis cs:0,1.66) [anchor={north west}] {\textcolor{C2}{\large$1.66$}};

\legend{
  {EDM$^\dagger$, 511 NFE},
  {EDM$^\star$, 511 NFE},
  {SEEDS$^\star$, \textbf{270 NFE}}
}
\end{axis}
\end{tikzpicture}

%% file: plot_trajectory.tex
\pgfplotsset{
compat=1.15,
width=\columnwidth,     
height=0.9\columnwidth,
every axis plot/.append style={line width=1.5pt},
}
\pgfplotsset{xtick style={draw=none}, xtickmin=0, xtickmax=50, xtick align=outside}
\begin{figure}[t]
\noindent\resizebox{\linewidth}{!}{
    \makebox[0.33\linewidth]{\hspace*{1.3cm}\scriptsize disc=vp, schedule=vp, scaling=vp  }\hfill%
\makebox[0.33\linewidth]{\hspace*{0.9cm}\scriptsize disc=edm, schedule=vp, scaling=vp}\hfill%
\makebox[0.33\linewidth]{\hspace*{0.3cm}\scriptsize disc=edm, schedule=linear, scaling=none}\hfill
}
    \centering
    \noindent\resizebox{\linewidth}{!}{
    \begin{tikzpicture}
    \input{trajectories/plot-trajectories1.tex}
    \end{tikzpicture}
    \begin{tikzpicture}
    \input{trajectories/plot-trajectories2.tex}
    \end{tikzpicture}
    \begin{tikzpicture}
    \input{trajectories/plot-trajectories3-zoom.tex}
    \end{tikzpicture}
    }

    \noindent\resizebox{\linewidth}{!}{
    \begin{tikzpicture}
    \input{trajectories/plot-trajectories4.tex}
    \end{tikzpicture}
    \begin{tikzpicture}
    \input{trajectories/plot-trajectories5.tex}
    \end{tikzpicture}
    \begin{tikzpicture}
    \input{trajectories/plot-trajectories6-zoom.tex}
    \end{tikzpicture}
    }
    \caption{Trajectories of 10 pixels (R channel) sampled from SEEDS (1st line) and Stochastic EDM (2nd line) on the optimized pre-trained model \cite{Karras2022edm} on ImageNet64. Schedule=scaling=vp corresponds to the VP coefficients in \eqref{vp-coefs} and schedule=linear, scaling=none to the EDM coefficients \eqref{edm-coefs}. We use the time discretizations disc=vp (linear) and disc=edm given in \cite[Tab.1]{Karras2022edm}.}
    \label{fig:stiffness}
\end{figure}

%% file: trajectories/plot-trajectories1.tex
\input{data/drawPixel1.tex}
\node[label={[label distance=0.5cm,text depth=-1ex,rotate=90]right:{\Huge SEEDS-3}}] at (-2,3) {};
\begin{axis}[
  ymin={-2}, ymax={3}, ymode={linear}, ytick={1.2, 1.4, 1.6, 1.8, 2.0, 2.2, 2.4, 2.6, 2.8}, yticklabels={},
  grid={minor}, legend pos={north east}, legend cell align={left}, legend style={font=\normalsize, /tikz/every even column/.append style={column sep=1.6mm}},
  xmin=0,
        xmax=55,
		axis y line*=left,
		axis x line*=bottom,
		xtick={0, 10, 20, 30, 40, 50, 55}, xticklabels={\tickSTEP{0}, $10$, $20$, $30$, $40$, $50$},
        yticklabels={},
        ylabel={Pixel state},
        label style={font=\huge},
        tick label style={font=\huge}
]

\addplot[C0] coordinates {\drawPixelOne}; 
\addplot[C1] coordinates {\drawPixelTwo};
\addplot[C2] coordinates {\drawPixelThree};
\addplot[C3] coordinates {\drawPixelFour};
\addplot[C4] coordinates {\drawPixelFive};
\addplot[C5] coordinates {\drawPixelSix};
\addplot[C6] coordinates {\drawPixelSeven};
\addplot[C7] coordinates {\drawPixelEight};
\addplot[C8] coordinates {\drawPixelNine};
\addplot[C9] coordinates {\drawPixelTen};

\end{axis}

%% file: trajectories/plot-trajectories2.tex
\input{data/drawPixel2.tex}
\begin{axis}[
  xmin={0}, xmax={50}, xmode={linear}, xtick={0, 10, 20, 30, 40, 50}, xticklabels={$0$, $10$, $20$, $30$, $40$, $50$},
  ymin={-2}, ymax={3}, ymode={linear}, ytick={1.2, 1.4, 1.6, 1.8, 2.0, 2.2, 2.4, 2.6, 2.8}, yticklabels={},
  grid={minor}, legend pos={north east}, legend cell align={left}, legend style={font=\normalsize, /tikz/every even column/.append style={column sep=1.6mm}},
  xmin=0,
        xmax=55,
		axis y line*=left,
		axis x line*=bottom,
		xtick={0, 10, 20, 30, 40, 50, 55}, xticklabels={$0$, $10$, $20$, $30$, $40$, $50$},
        yticklabels={},
        label style={font=\huge},
        tick label style={font=\huge}
]

\addplot[C0] coordinates {\drawPixelOne}; 
\addplot[C1] coordinates {\drawPixelTwo};
\addplot[C2] coordinates {\drawPixelThree};
\addplot[C3] coordinates {\drawPixelFour};
\addplot[C4] coordinates {\drawPixelFive};
\addplot[C5] coordinates {\drawPixelSix};
\addplot[C6] coordinates {\drawPixelSeven};
\addplot[C7] coordinates {\drawPixelEight};
\addplot[C8] coordinates {\drawPixelNine};
\addplot[C9] coordinates {\drawPixelTen};

\end{axis}

%% file: trajectories/plot-trajectories3-zoom.tex
\input{data/drawPixel3.tex}
    \begin{pgfonlayer}{background}
    \begin{axis}[
        ymin=-120,
        ymax=152,
		xmin=0,
        xmax=55,
		axis y line*=left,
		axis x line*=bottom,
		xtick={0, 10, 20, 30, 40, 50, 55}, xticklabels={$0$, $10$, $20$, $30$, $40$, $50$},
        yticklabels={},
        label style={font=\huge},
        tick label style={font=\huge} 
    ]
			\addplot[C0] coordinates {\drawPixelOne}; 
			\addplot[C1] coordinates {\drawPixelTwo};
			\addplot[C2] coordinates {\drawPixelThree};
			\addplot[C3] coordinates {\drawPixelFour};
			\addplot[C4] coordinates {\drawPixelFive};
			\addplot[C5] coordinates {\drawPixelSix};
			\addplot[C6] coordinates {\drawPixelSeven};
			\addplot[C7] coordinates {\drawPixelEight};
			\addplot[C8] coordinates {\drawPixelNine};
			\addplot[C9] coordinates {\drawPixelTen};
			
			\coordinate (inset) at (axis description cs:0.95,0.95);
    \end{axis}
    \end{pgfonlayer}

    \begin{pgfonlayer}{foreground}
    \begin{axis}[
        at={(inset)},
        anchor=north east,
        small,
        %
        width=7cm,
        height=4cm,
        xmin=25,
        xmax=35,
        ymin=-5,
        ymax=5,
        yticklabels={},
        xtick={25, 35},
        xticklabels={$25$, $35$},
        axis background/.style={
            fill=white,
        },
        name=insetAxis,
        label style={font=\huge},
        tick label style={font=\huge} 
    ]
        \addplot[C0] coordinates {\drawPixelOne}; 
		\addplot[C1] coordinates {\drawPixelTwo};
		\addplot[C2] coordinates {\drawPixelThree};
		\addplot[C3] coordinates {\drawPixelFour};
		\addplot[C4] coordinates {\drawPixelFive};
		\addplot[C5] coordinates {\drawPixelSix};
		\addplot[C6] coordinates {\drawPixelSeven};
		\addplot[C7] coordinates {\drawPixelEight};
		\addplot[C8] coordinates {\drawPixelNine};
		\addplot[C9] coordinates {\drawPixelTen};
    	\end{axis}
    \end{pgfonlayer}

%% file: trajectories/plot-trajectories4.tex
\input{data/drawPixel4.tex}
\node[label={[label distance=-0.5cm,text depth=-1ex,rotate=90]right:{\Huge Stochastic EDM}}] at (-2,3) {};
\begin{axis}[
  ymin={-3}, ymax={3}, 
  grid={minor}, legend pos={north east}, legend cell align={left}, legend style={font=\normalsize, /tikz/every even column/.append style={column sep=1.6mm}},
  xmin=0,
        xmax=55,
		axis y line*=left,
		axis x line*=bottom,
		xtick={0, 10, 20, 30, 40, 50, 55}, xticklabels={\tickSTEP{0}, $10$, $20$, $30$, $40$, $50$},
        yticklabels={},
        ylabel={Pixel state},
        label style={font=\huge},
        tick label style={font=\huge} 
]

\addplot[C0] coordinates {\drawPixelOne}; 
\addplot[C1] coordinates {\drawPixelTwo};
\addplot[C2] coordinates {\drawPixelThree};
\addplot[C3] coordinates {\drawPixelFour};
\addplot[C4] coordinates {\drawPixelFive};
\addplot[C5] coordinates {\drawPixelSix};
\addplot[C6] coordinates {\drawPixelSeven};
\addplot[C7] coordinates {\drawPixelEight};
\addplot[C8] coordinates {\drawPixelNine};
\addplot[C9] coordinates {\drawPixelTen};

\end{axis}

%% file: trajectories/plot-trajectories5.tex
\input{data/drawPixel5.tex}

\begin{axis}[
  xmin={0}, xmax={50}, xmode={linear}, xtick={0, 10, 20, 30, 40, 50}, xticklabels={$0$, $10$, $20$, $30$, $40$, $50$},
  ymin={-3}, ymax={3}, ymode={linear}, ytick={1.2, 1.4, 1.6, 1.8, 2.0, 2.2, 2.4, 2.6, 2.8}, yticklabels={},
  grid={minor}, legend pos={north east}, legend cell align={left}, legend style={font=\normalsize, /tikz/every even column/.append style={column sep=1.6mm}},
  xmin=0,
        xmax=55,
		axis y line*=left,
		axis x line*=bottom,
		xtick={0, 10, 20, 30, 40, 50, 55}, xticklabels={$0$, $10$, $20$, $30$, $40$, $50$},
        yticklabels={},
        label style={font=\huge},
        tick label style={font=\huge} 
]

\addplot[C0] coordinates {\drawPixelOne}; 
\addplot[C1] coordinates {\drawPixelTwo};
\addplot[C2] coordinates {\drawPixelThree};
\addplot[C3] coordinates {\drawPixelFour};
\addplot[C4] coordinates {\drawPixelFive};
\addplot[C5] coordinates {\drawPixelSix};
\addplot[C6] coordinates {\drawPixelSeven};
\addplot[C7] coordinates {\drawPixelEight};
\addplot[C8] coordinates {\drawPixelNine};
\addplot[C9] coordinates {\drawPixelTen};

\end{axis}

%% file: trajectories/plot-trajectories6-zoom.tex
\input{data/drawPixel6.tex}

    \begin{pgfonlayer}{background}
    \begin{axis}[
        ymin=-145,
        ymax=135,
        xmin=0,
        xmax=55,
		axis y line*=left,
		axis x line*=bottom,
		xtick={0, 10, 20, 30, 40, 50, 55}, xticklabels={$0$, $10$, $20$, $30$, $40$, $50$},
        yticklabels={},
        label style={font=\huge},
        tick label style={font=\huge} 
    ]
			\addplot[C0] coordinates {\drawPixelOne}; 
			\addplot[C1] coordinates {\drawPixelTwo};
			\addplot[C2] coordinates {\drawPixelThree};
			\addplot[C3] coordinates {\drawPixelFour};
			\addplot[C4] coordinates {\drawPixelFive};
			\addplot[C5] coordinates {\drawPixelSix};
			\addplot[C6] coordinates {\drawPixelSeven};
			\addplot[C7] coordinates {\drawPixelEight};
			\addplot[C8] coordinates {\drawPixelNine};
			\addplot[C9] coordinates {\drawPixelTen};
			
			\coordinate (inset) at (axis description cs:0.95,0.95);
    \end{axis}
    \end{pgfonlayer}

    \begin{pgfonlayer}{foreground}
    \begin{axis}[
        at={(inset)},
        anchor=north east,
        small,
        %
        width=7cm,
        height=4cm,
        ymin=-5,
        ymax=5,
        xmin=25,
        xmax=35,
        yticklabels={},
        xtick={25, 35},
        xticklabels={$25$, $35$},
        label style={font=\huge},
        tick label style={font=\huge},
        axis background/.style={
            fill=white,
        },
        name=insetAxis,
    ]
        \addplot[C0] coordinates {\drawPixelOne}; 
		\addplot[C1] coordinates {\drawPixelTwo};
		\addplot[C2] coordinates {\drawPixelThree};
		\addplot[C3] coordinates {\drawPixelFour};
		\addplot[C4] coordinates {\drawPixelFive};
		\addplot[C5] coordinates {\drawPixelSix};
		\addplot[C6] coordinates {\drawPixelSeven};
		\addplot[C7] coordinates {\drawPixelEight};
		\addplot[C8] coordinates {\drawPixelNine};
		\addplot[C9] coordinates {\drawPixelTen};
    	\end{axis}
    \end{pgfonlayer}

%% file: tikz/cifar-cond-cont-order-2-noise-term-comparision.tex
\pgfplotsset{xtick style={draw=none}, xtickmin=0, xtickmax=150, xtick align=outside}

\pgfplotsset{every axis plot/.style={line width=2pt, mark size=3.5pt, mark=x}}


\begin{tikzpicture}
\begin{axis}[
  xmin={0}, xmax={170}, domain={0:170}, xmode={linear}, xtick={0, 8, 30, 60, 90, 120, 150}, xticklabels={, \tickNFE{8}, $30$, $60$, $90$, $120$, $150$},
  ymin={-30}, ymax={450}, restrict y to domain*=-30:450, ymode={linear}, ytick={2, 90, 180, 270, 360, 450}, yticklabels={$2$, $90$, $180$, $270$, $360$, \tickFID},
  grid={major}, legend pos={north east}, legend cell align={left}, legend style={nodes={scale=0.9, transform shape}, font=\normalsize, /tikz/every even column/.append style={column sep=1.6mm}},
  label style={font=\large},
        tick label style={font=\large} 
]
\addplot[C0, line width=1pt] coordinates {\drawSeedsTwoCorrect};
\addplot[C1, line width=1pt] coordinates {\drawSeedsTwoNaiveOne};
\addplot[C2, line width=1pt] coordinates {\drawSeedsTwoNaiveTwo};
\addplot[C3, line width=1pt] coordinates {\drawSeedsTwoNaiveThree};
\addplot[C4, line width=1pt] coordinates {\drawSeedsTwoNaiveFour};

\addplot[C0, only marks, forget plot]  coordinates {(88,	2.332)};\node at (axis cs:(88,	2.332) [anchor={north east}] {\textcolor{C0}{\large $2.33$}};
\addplot[C0, only marks, forget plot]  coordinates {(148, 2.362)};\node at (axis cs:(148, 2.362) [anchor={north east}] {\textcolor{C0}{\large $2.36$}};
\addplot[C1, only marks, forget plot]  coordinates {(148 ,	261.243)};\node at (axis cs:(148 ,	261.243) [anchor={north east}] {\textcolor{C1}{\large $261.24$}};
\addplot[C2, only marks, forget plot]  coordinates {(28, 45.10)};\node at (axis cs:(28, 45.10) [anchor={north east}] {\textcolor{C2}{\large $45.1$}};
\addplot[C3, only marks, forget plot]  coordinates {(148,65.0644)};\node at (axis cs:(148,65.0644) [anchor={north east}] {\textcolor{C3}{\large $65.06$}};
\addplot[C3, only marks, forget plot]  coordinates {(160 ,	106.735)};\node at (axis cs:(160 ,	106.735) [anchor={south east}] {\textcolor{C3}{\large $106.73$}};
\addplot[C4, only marks, forget plot]  coordinates {(148, 18.5319)};\node at (axis cs:(148, 18.5319) [anchor={south east}] {\textcolor{C4}{\large $18.53$}};
\addplot[C4, only marks, forget plot]  coordinates {(88,23.8983)};\node at (axis cs:(88,23.8983) [anchor={south}] {\textcolor{C4}{\large $23.89$}};
\legend{
  {SEEDS-2-Correct},
  {SEEDS-2-Naive-1},
  {SEEDS-2-Naive-2},
  {SEEDS-2-Naive-3},
  {SEEDS-2-Naive-4}
}
\end{axis}
\end{tikzpicture}

%% file: tikz/cifar-cond-cont-order-3-noise-term-comparision.tex
\pgfplotsset{xtick style={draw=none}, xtickmin=0, xtickmax=150, xtick align=outside}

\pgfplotsset{every axis plot/.style={line width=2pt, mark size=3.5pt, mark=x}}


\begin{tikzpicture}
\begin{axis}[
  xmin={25}, xmax={190}, domain={25:190}, xmode={linear}, xtick={25, 36, 60, 90, 120, 150, 180}, xticklabels={, \tickNFE{36}, $60$, $90$, $120$, $150$, $180$},
  ymin={-30}, ymax={420}, restrict y to domain*=-30:420, ymode={linear}, ytick={2, 80, 160, 240, 320, 400, 420}, yticklabels={$2$, $80$, $160$, $240$, $320$, $400$, \tickFID},
  grid={major}, legend pos={north east}, legend cell align={left}, legend style={nodes={scale=0.9, transform shape}, font=\normalsize, /tikz/every even column/.append style={column sep=1.6mm}},
  label style={font=\large},
        tick label style={font=\large} 
]

\addplot[C0, line width=1pt] coordinates {\drawSeedsThreeCorrect};
\addplot[C1, line width=1pt] coordinates {\drawSeedsThreeNaiveOne};
\addplot[C2, line width=1pt] coordinates {\drawSeedsThreeNaiveTwo};
\addplot[C3, line width=1pt] coordinates {\drawSeedsThreeNaiveThree};
\addplot[C4, line width=1pt] coordinates {\drawSeedsThreeNaiveFour};

\addplot[C0, only marks, forget plot]  coordinates {(129, 2.082)};\node at (axis cs:(129, 2.082) [anchor={north east}] {\textcolor{C0}{\large $2.08$}};
\addplot[C0, only marks, forget plot]  coordinates {(180, 2.199)};\node at (axis cs:(180, 2.199) [anchor={north east}] {\textcolor{C0}{\large $2.19$}};
\addplot[C1, only marks, forget plot]  coordinates {(180, 307.622)};\node at (axis cs:(180, 307.622) [anchor={north east}] {\textcolor{C1}{\large $307.62$}};
\addplot[C2, only marks, forget plot]  coordinates {(42, 275.881)};\node at (axis cs:(42, 275.881) [anchor={north west}] {\textcolor{C2}{\large $275.88$}};
\addplot[C3, only marks, forget plot]  coordinates {(180,	90.80)};\node at (axis cs:(180,	90.80)[anchor={north east}] {\textcolor{C3}{\large $90.80$}};
\addplot[C4, only marks, forget plot]  coordinates {(120,	6.6872)};\node at (axis cs:(120, 6.6872) [anchor={south}] {\textcolor{C4}{\large $6.68$}};
\addplot[C4, only marks, forget plot]  coordinates {(180, 2.43)};\node at (axis cs:(180,	2.43) [anchor={south east}] {\textcolor{C4}{\large $2.43$}};
\legend{
  {SEEDS-3-Correct},
  {SEEDS-3-Naive-1},
  {SEEDS-3-Naive-2},
  {SEEDS-3-Naive-3},
  {SEEDS-3-Naive-4}
}
\end{axis}
\end{tikzpicture}

%% file: App-A.tex
\section{Discussion}
\label{app:discussion}

\paragraph{Why do SEEDS exhibit high FID scores in the low NFE regime?}

In Appendix \ref{app:tables}, we provide all tables with FID values at 
increasing NFEs corresponding to Fig. \ref{fig:fid-cont}. Notice that SEEDS 
exhibit high FID scores in the low NFE regime. Let us discuss some 
theoretical facts that might explain this phenomenon. In \cite[Theorem 1.]{xu2023restart}, 
there is a formal explanation on why ODE samplers outperform SDE samplers in 
the small NFE regime and fall short in the large NFE regime. In particular, 
it is theoretically shown that, at large step sizes, it is the 
discretization error that dominates sampling errors while, at small step 
sizes, it is the approximation error that dominates it. One can infer that 
in the large NFE regime, SDE methods (with proven convergence orders) will 
outperform ODE methods in terms of sampling quality. It would be interesting 
to see if combining SEEDS with the new ideas in \cite{xu2023restart} might 
improve SEEDS in the low NFE regime. 

Additionally, SDE solvers and ODE solvers may have different sample quality, 
even if the score models are optimal. Indeed, as proven in \cite[Appendix B.]{lu2022maximum}, 
even if the score model has been trained to the optimal score function, the 
distributions between SDEs and ODEs are still different. This is because the 
distribution at time $T$ in the forward process is always not exactly a 
standard normal distribution as in the reverse SDE/ODE process, and thus the 
distributions at time 0 are different. 

Finally, \cite{xu2023restart, chen2023sampling} show that the overall 
generalization error divides into discretisation and approximation errors, 
indicating the theoretical role of the sampling error as a contribution to a 
DPM's performance. This is also consistent with the suggestion in 
\cite[Section 5]{Karras2022edm}, stating that more diverse datasets continue 
to benefit from stochastic sampling rather than deterministic sampling.

\paragraph{Do SEEDS maintain good performances for higher resolution image generation?}

As training-free and optimization-free SDE solvers, SEEDS naturally maintain 
good performances in higher resolution images in the realm of unguided image 
generation (unconditional and conditional). For unconditional generation 
using Latent Diffusion Model, SEEDS are able to generate good quality images 
already at 100 NFEs. Now, for guided image generation, higher-stage SEEDS 
will, as expected, see their performance sharply drop as the guidance scale 
grows for the same reasons DPM-solver do (see \cite{lu2022dpm} for details). 
Yet, SEEDS-1 still maintains high quality sampling in this scenario.
In Appendix \ref{app:figs}, Fig. \ref{fig:SD-512x512} exhibits a $512^2$ 
image generated with SEEDS-1 at 90 NFEs on Stable Diffusion with default 
guidance scale.

\paragraph{Do SEEDS exhibit other distinguishing features beyond generation optimal quality?}

As SDE solvers, SEEDS have a distinctive (although indirect) capability in 
the realm of adversarial robustness compared to DPM-Solver:
\begin{enumerate}
    \item For more than 2 years, the leader-board of 
    \href{https://robustbench.github.io}{RobustBench} has been dominated by 
    Diffusion-based data augmentation techniques on top of Adversarial Training. 
    The current SOTA \cite{peng2023robust,wang2023better} uses EDM-preconditioned 
    DPMs and need to generate as many as 50M images to achieve SOTA robustness 
    results. For ImageNet-64, \cite{wang2023better} use the EDM pretrained model 
    and optimization-based stochastic EDM sampler for data augmentation, leading 
    to a 5\% robust accuracy improvement compared to doing so for the baseline ADM 
    pretrained model. Since SEEDS reach same FID quality as EDM, but twice faster, 
    we believe it will have a positive impact in this domain, making diffusion-
    based data augmentation schemes more affordable with limited computational 
    capacity. 
    \item The work \cite{nie2022DiffPure} uses DPM-based adversarial purification 
    as a test-time adversarial defense. The idea is to use off-the-shelf DPMs to 
    annihilate the adversarial content in inputs in test-time before feeding it to 
    a pretrained classifier. In \cite[Table 6]{nie2022DiffPure}, one can see that 
    SDE solvers show substantial robustness capabilities compared to ODE solvers.  
\end{enumerate}
We'd like to stress out that, although indirect, such capabilities are on 
the side of the sampling methods rather than on the side of intrinsic 
robustness properties of score-based learning method. More broadly, 
robustness properties of ML models determined as neural differential 
equations (DPMs being in this scope) has been studied in 
\cite{yan2022robustness, gonzalez}.

\paragraph{How do SEEDS compare to Stochastic Runge-Kutta methods?}

Contrary to the ODE case, there are many stochastic Runge-Kutta approaches, 
usually tailored for SDEs of a specific form. Nonetheless, a common way of 
distinguishing solvers with same strong order is to assign them couples 
$(p_d,p_s)$, where $p_s$ is the (stochastic) strong convergence order and 
$p_d$ is the resulting (deterministic) order determined if setting $g=0$ in 
the considered SDE i.e. when they are deterministic. For instance, 
\cite[Tab. 6.2]{Rossler} determines solvers with orders (1,1.0) and (2,1.0) 
respectively and \cite[Tab. 6.3]{Rossler} determines solvers SRA1 and SRA3 
with orders (2,1.5) and (3,1.5) respectively. Many of these solvers' speed 
was already tested in \cite[Table 3]{jolicoeur2021gotta} on CIFAR-10 (VP). 
Experimentally, SEEDS-1 shows to be 6.83x faster than the baseline Euler-Maruyama scheme.

To our knowledge, the only available strong order SERKs method for SDEs with 
in-homogeneous diffusion coefficients are the exponential Euler-Maruyama 
(EEM) method \cite{komori2016} and the stochastic RK Lawson (SRKL) schemes 
\cite{debrabant2021runge}. In short, the SRKL schemes only compute 
analytically the linear coefficient and use the Integrating Factor (IF) 
method to approximate the integrals in the representation of the exact 
solution given by the variation-of-parameters formula. This way, by a 
special change of variables (see \cite[Alg. 1]{debrabant2021runge}), one can 
create exponential integrator versions of many SDE methods. We implemented 
our own version of the SRKL schemes, which we denote SRKL-1/2/3, to take 
into account that $\sigma,\alpha$ are not constant and used the Integrating 
Factor method to approximate the integrals in the representation of the 
exact solution given by the variation-of-parameters formula and under the VP 
$\lambda$-change-of-variables. Interestingly, the SRKL schemes stabilize at 
increasing NFEs but at much higher FID values than their SETD counterparts.

In Table \ref{table:SEEDS-vs-SRK} below, we draw a comparison of SEEDS with 
current Stochastic RKL methods (on the $\lambda$ temporal parameter space) 
on CIFAR-10 for the discretely trained DPM in the VP unconditional regime.

\begin{table}[ht]
  \caption{\label{table:SEEDS-vs-SRK}Comparison of SEEDS with adapted Stochastic RKL methods on
  CIFAR-10 in the VP unconditional discrete framework.}
  \begin{center}
    \begin{center}
      {  \textsc{ \begin{tabular}{lrrrrr}
        \toprule
        Method $\backslash$ NFE & 10 & 20 & 50 & 100 & Best known\\
        \midrule
        SRKL-1($\lambda$) & 332.52 & 282.96 & 33.42 & 8.62 & -\\
        SEEDS-1 & 303.48 & 153.21 & 22.70 & 7.97 & (500 NFE) 3.13\\
        SRKL-2($\lambda$) & 475.20 & 469.64 & 134.82 & 7.74 & -\\
        SEEDS-2 & 476.90 & 226.70 & 7.17 & 3.23 & (90 NFE) 3.21\\
        SRKL-3($\lambda$) & 462.24 & 376.15 & 8.36 & 7.46 & -\\
        SEEDS-3 & 483.00 & 428.60 & 43.30 & 3.41 & (201 NFE) 3.08 \\
        \bottomrule
      \end{tabular}}}
    \end{center}
  \end{center}

\end{table}

\paragraph{Is the proven convergence order of each of the SEEDS methods optimal?}

The proposed convergence orders for each SEEDS-1/2/3 is optimal: this is a 
consequence of the general result from \cite{CC} about maximum convergence 
rates for SDE schemes with uncorrelated Gaussian increments. The underlying 
idea, fully detailed in Appendix \ref{app:solvers}, is that any solver with 
strong order $\geq 1.5$ has to account for double stochastic integrals in 
the non-truncated Itô-Taylor expansion, ultimately forcing any SRK-like 
solver to use correlated random variables (see \cite[Tab. 6.3]{Rossler} and 
\cite{KP1} more generally). SEEDS avoid this additional complexity but an 
interesting future avenue would be to extend SEEDS to the higher strong 
order case (and not in the IF approach but the SETD approach) as long as it 
doesn't incur into an explosion of needed NFEs per step to craft such 
solvers. Another interesting path would be to craft weak second order SERK 
methods for DPMs (the work \cite{komori} addressed this only for homogeneous 
semi-linear SDEs with constant linear coefficients).

%% file: App-B.tex
\section{Detailed Derivation of the SEEDS Design Space}
\label{app:design}

In Sections \ref{sec:background} and \ref{sec:ingredients} we proposed a
simplified presentation of the design space of diffusion models and of the
ingredients that constitute our proposed SEEDS methodology. In this section,
we further develop our presentation in a technical manner, making explicit the
formalization of our design choices.

\subsection{The Isotropic General SDE framework}

In Section \ref{sec:background}, we presented a parametric family of
differential equations \eqref{eqq} driving the generative process for DPMs,
based on time-reversing the forward noising diffusion process \eqref{diff1}.
While doing so, we presented two parameters - the noise schedule $\sigma_t$
and the scaling $\alpha_t$ - for which the effects on DPMs have been widely
studied in {\cite{Karras2022edm}}.

{\vparagraph{Expressing the forward and reverse SDEs in terms of $\alpha_t$
and $\sigma_t$.}}As the shape of the trajectories of \eqref{diff1} and
\eqref{eqq} (for $\ell = 0, 1$) are defined by $\alpha_t$ and $\sigma_t$, we
start by writing down, for the scaling $\x_t = \alpha_t
\widehat{\mathbf{x}}_t$, the scaled generalization of the proposed SDEs in
{\cite[Eq. 103]{Karras2022edm}} which unifies in a single framework the
forward and reverse trajectories:
\begin{equation}
  \mathd \mathbf{x}^{\pm}_t = \left[ \frac{\dot{\alpha}_t}{\alpha_t}
  \mathbf{x}^{\pm}_t - \alpha^2_t  \dot{\sigma}_t \sigma_t
  \nabla_{\mathbf{x}^{\pm}_t} \log p (\widehat{\mathbf{x}}^{\pm}_t ; \sigma_t)
  \pm \alpha^2_t  \dot{\sigma}_t \sigma_t \nabla_{\mathbf{x}^{\pm}_t} \log p
  (\widehat{\mathbf{x}}^{\pm}_t ; \sigma_t) \right] \mathd t + \alpha_t
  \sigma_t  \sqrt{2 \frac{\dot{\sigma}_t}{\sigma_t}} \mathd
  \tmmathbf{\omega}^{\pm}_t . \label{gen-sde}
\end{equation}
Following {\cite{Karras2022edm}}, the previous VP, VE, iDDPM, DDIM and EDM
frameworks all are unified as different choices of $\alpha_t$, $\sigma_t$,
among other choices presented in {\cite[Tab. 1]{Karras2022edm}} and we 
will use this as a basis for all the proofs contained in this Appendix. In
particular, forward time means taking $\mathbf{x}^+_t$ for which the score
vanishes in this context. Now set $\mathbf{x}_t = \mathbf{x}^-_t$.

The formulation in \eqref{eqq} involves a family of backwards differential
equations controlled by a parameter $\ell \in [0, 1]$ which all yield
reverse-time processes for \eqref{diff1}, a fact that can be obtained by
studying the Fokker-Planck equation for marginals $p (\widehat{\mathbf{x}}^+_t
; \sigma_t)$ of \eqref{gen-sde}.

When $\ell = 1$, \eqref{eqq} the obtained SDE is known as the reverse SDE
(RSDE) and, when $\ell = 0$, we obtain an ODE that is known as the Probability
Flow ODE (PFO):
\begin{eqnarray}
  \mathd \mathbf{x}_t & = & \left[ \frac{\dot{\alpha}_t}{\alpha_t}
  \mathbf{x}_t - \alpha^2_t  \dot{\sigma}_t \sigma_t \nabla_{\mathbf{x}_t}
  \log p (\widehat{\mathbf{x}}_t ; \sigma_t) \right] \mathd t. 
  \label{pfo-gen}
\end{eqnarray}
Now, finding a minimum for the loss function in {\cite[Eq. 51]{Karras2022edm}}
is formulated as a convex optimization problem. As such, for the ideal model
$D (\mathbf{x}_t ; \sigma_t) = \arg \min_D \mathcal{L} (D ; \mathbf{x}_t,
\sigma_t)$, the score function with scaled input is expressed as
\begin{eqnarray*}
  \nabla_{\mathbf{x}_t} \log p (\widehat{\mathbf{x}}_t ; \sigma_t) & = &
  \frac{D (\widehat{\mathbf{x}}_t ; \sigma_t) -
  \widehat{\mathbf{x}}_t}{\alpha_t \sigma^2_t} .
\end{eqnarray*}
This ideal model is usually subtracted by a \tmtextit{raw} network $F$ in the
form of a time-dependent preconditioning:
\[ D (\mathbf{x}_t ; \sigma_t) \assign c_1 (t) \mathbf{x}_t + c_2 (t) F (c_3
   (t) \mathbf{x}_t ; c_4 (t)), \qquad c_i (t) \in \mathbb{R}^d, \quad i = 1,
   \ldots, 4. \]
As such, we can express the score function as two parameterizations involving
$D$ or $F$ as follows:
\begin{eqnarray}
  \nabla_{\mathbf{x}_t} \log p (\widehat{\mathbf{x}}_t ; \sigma_t) & = &
  \frac{D (\widehat{\mathbf{x}}_t ; \sigma_t) -
  \widehat{\mathbf{x}}_t}{\alpha_t \sigma^2_t} = \frac{(c_1 (t) - 1)
  \widehat{\mathbf{x}}_t + c_2 (t) F (c_3 (t) \widehat{\mathbf{x}}_t ; c_4
  (t))}{\alpha_t \sigma^2_t} .  \label{precond}
\end{eqnarray}
Let us now denote $D^1_{\theta, t} \assign D_{\theta} (\widehat{\mathbf{x}}_t
; \sigma_t)$ for a pre-trained network approximating the ideal denoiser and
let $D^2_{\theta, t} \assign F_{\theta} (c_3 (t) \widehat{\mathbf{x}}_t ; c_4
(t))$ be the corresponding raw pre-trained network. Substituting the score
function in the RSDE and the PFO with each of these models yields four
\tmtextit{different} differential equations with a neural network as one of
their components. These are given, for $i = 1, 2$, by
\begin{eqnarray}
  \mathd \mathbf{x}_t & = & [A^i (t) \mathbf{x}_t + B^i (t) D^i_{\theta, t}]
  \mathd t + g (t) \mathd \tmmathbf{\bar{\omega}}_t,  \label{NRSDE-1}\\
  \mathd \mathbf{x}_t & = & [A^{i + 2} (t) \mathbf{x}_t + B^{i + 2} (t)
  D^i_{\theta, t}] \mathd t.  \label{NPFO-1}
\end{eqnarray}
When $D^1_{\theta, t}$ (resp. $D^2_{\theta, t}$) \ is employed to replace the
score function in \eqref{gen-sde} using \eqref{precond}, the resulting SDE
\eqref{NRSDE-1} will be called \tmtextit{data (resp. noise) prediction neural
SDE}. Proceeding analogously for the PFO \eqref{pfo-gen} yield two ODEs
\eqref{NPFO-1} which will be called \tmtextit{data (resp. noise) prediction
neural PFO}. The general form of the $A^i$ and $B^i$ coefficients determining
each of these DEs is as follows:
\begin{align*}
       A^1 (t) &= \frac{\dot{\alpha}_t}{\alpha_t} + 2
  \frac{\dot{\sigma}_t}{\sigma_t} & B^1 (t) &= - 2 \alpha_t
  \frac{\dot{\sigma}_t}{\sigma_t} \tag{DP NRSDE} \\
    A^2 (t) &= \frac{\dot{\alpha}_t}{\alpha_t} + 2
  \frac{\dot{\sigma}_t}{\sigma_t} (1 - c_1 (t)) & B^2 (t) &= - 2 \alpha_t
  \frac{\dot{\sigma}_t}{\sigma_t} c_2 (t) \tag{NP NRSDE}\\
   A^3 (t) &= \frac{\dot{\alpha}_t}{\alpha_t} +
  \frac{\dot{\sigma}_t}{\sigma_t} & B^3 (t) &= - \alpha_t
  \frac{\dot{\sigma}_t}{\sigma_t} \tag{DP NPFO}\\
    A^4 (t) &= \frac{\dot{\alpha}_t}{\alpha_t} +
  \frac{\dot{\sigma}_t}{\sigma_t} (1 - c_1 (t)) & B^4 (t) &= - \alpha_t
  \frac{\dot{\sigma}_t}{\sigma_t} c_2 (t) \tag{NP NPFO}
\end{align*}
\begin{remark}
  \label{four-modes}At first glance, it would seem misleading to differentiate
  four DEs as these essentially correspond to different choices of $\alpha_t$,
  $\sigma_t$, $c_1 (t), \ldots, c_4 (t)$. But the reason why we do so is that,
  after applying the \tmtextit{variation of constants} formula, each of
  these DEs will yield a different representation of their exact solutions
  (see \eqref{exp-rep-sde} and \eqref{exp-rep-ode} below). As we will see
  below, constructing exponential integrators heavily depends on such
  representation and will show to lead to four different modes of SEEDS
  solvers, each one showing different behavior and performance for sampling
  from pre-trained DPMs. As such, we will articulate this difference already
  at the DE formulation.
\end{remark}

For $t < s$, the variation of constants formulae for NSDEs allows to represent
the exact solutions of \eqref{NRSDE-1} as
\begin{equation}
  \mathbf{x}_t = \Phi^i (t, s) \mathbf{x}_s + \int^t_s \Phi^i (t, \tau) B^i
  (\tau) D^i_{\theta, \tau} \mathd \tau + \int^t_s \Phi^i (t, \tau) g (\tau)
  \mathd \tmmathbf{\bar{\omega}}_t, \qquad i = 1, 2 \label{exp-rep-sde}
\end{equation}
and those for NPFOs \eqref{NPFO-1} as
\begin{equation}
  \mathbf{x}_t = \Phi^i (t, s) \mathbf{x}_s + \int^t_s \Phi^i (t, \tau) B^i
  (\tau) D^{i - 2}_{\theta, \tau} \mathd \tau, \qquad i = 3, 4
  \label{exp-rep-ode}
\end{equation}
where
\begin{equation}
  \Phi_{A^i} (t, s) = \exp \left( \int^t_s A^i (\tau) \mathd \tau \right)
  \label{Phi}
\end{equation}
is called the transition matrix associated with $A^i (t)$ and is defined as
the solution to
\[ \frac{\partial}{\partial t} \Phi_{A^i} (t, s) = A^i (t) \Phi_{A^i} (t, s),
   \qquad \Phi_{A^i} (s, s) =\tmmathbf{I}_d . \]
When $A^i (t)$ is constant and $B^i (t) = 1$, there is a well-established
literature on exponential ODE and SDE solvers with explicit \tmtextit{stiff
order conditions} and prescribed by different forms of Butcher tableaux. When
$A^i (t)$ is not constant, for the expression in \eqref{Phi} to make sense in
the usual sense (in terms of exponential series expansion) instead of having
to make use of time-ordered exponentials/ Magnus expansions, the $f (t)
\assign A^i (t)$ coefficients must satisfy $[f^{(k)} (t), f^{(l)} (s)] = 0$.
This condition is trivially satisfied here as the $A^i (t)$ considered here
are $d$-dimensional diagonal matrices.

Notice that, if $A^i \neq A^j$ for some $i \neq j$, their associated
transition matrices will not be equal. In particular, if $A^1 \neq A^2$, then
the variances of the stochastic integrals in \eqref{exp-rep-sde} are different
for $i = 1$ and $i = 2$. This is the first step in explaining the statement in
Rem. \ref{four-modes}, and we refer the reader to the proof of Proposition
\ref{related} where we put into evidence its validity.
\subsection{Re-framing and Generalizing Previous Exponential Solvers}
\label{app:refr}

\subsubsection{The VP case}\label{case}

Let $\tilde{\alpha}_t \assign \int_0^t (\beta_d \tau - \beta_m) \mathd \tau =
\frac{1}{2} \beta_d t^2 + \beta_m t$, where $\beta_d > \beta_m > 0$. Set
\[ f (t) \assign \frac{\mathd \log \alpha_t}{\mathd t}, \quad g (t) = \alpha_t
   \sqrt{\frac{\mathd [\sigma^2_t]}{\mathd t}}, \qquad \sigma_t =
   \sqrt{e^{\tilde{\alpha}_t} - 1}, \qquad \alpha_t = e^{- \frac{1}{2} 
   \tilde{\alpha}_t} = \frac{1}{\sqrt{\sigma^2_t + 1}} . \]
Recall that in the VP case, and the \tmtextit{noise prediction mode},
{\cite{dpm-solver}} construct exponential solvers on the base of the following
ODE
\begin{eqnarray}
  \mathd \mathbf{x}_t & = & \left[ f (t) \mathbf{x}_t + \frac{g^2 (t)}{2
  \bar{\sigma}_t} \epsilon_{\theta} (\mathbf{x}_t ; t) \right] \mathd t, \quad
  t \in [T, 0],  \label{dpm-ode}
\end{eqnarray}
where $\bar{\sigma}_t \assign \alpha_t \sigma_t$. The ODE \eqref{dpm-ode}
identifies with that in {\cite{Karras2022edm}} for the VP case for which the
authors identify the preconditioning
\[ c_1 (t) = 1, \quad c_2 (t) = - \sigma_t, \quad c_3 (t) =
   \dfrac{1}{\sqrt{\sigma_t^2 + 1}}, \quad c_4 (t) = (M - 1) \sigma^{- 1}
   (\sigma_t) = (M - 1) t. \]
As such, we obtain the following coefficients for the NP NPFO:
\[ A^4 (t) = \frac{\dot{\alpha}_t}{\alpha_t}, \quad B^4 (t) = \alpha_t 
   \dot{\sigma}_t, \quad \Phi^4 (t, s) = \frac{\alpha_t}{\alpha_s} \]
and
\begin{eqnarray*}
  \nabla_{\mathbf{x}_t} \log p (\widehat{\mathbf{x}}_t ; \sigma_t) =
  \frac{D_{\theta} (\widehat{\mathbf{x}}_t ; \sigma_t) -
  \widehat{\mathbf{x}}_t}{\alpha_t \sigma^2_t} & = & \frac{(c_1 (t) - 1)
  \widehat{\mathbf{x}}_t + c_2 (t) F (c_3 (t) \widehat{\mathbf{x}}_t ; c_4
  (t))}{\alpha_t \sigma^2_t}  \\
  & = &  \frac{- \sigma_t F_{\theta} (\mathbf{x}_t, (M -
  1) t)}{\alpha_t \sigma^2_t} .
\end{eqnarray*}
\subsubsection{Proof of Proposition \ref{Prop1}}
First of all, denote $F_{\theta} (\mathbf{x}_t, (M - 1) t)
=\tmmathbf{\epsilon}_{\theta} (\mathbf{x}_t, t) .$ We have
\[ f (t) = \frac{\mathd \log \alpha_t}{\mathd t}, \quad g^2 (t) = 2
   \bar{\sigma}_t^2  \left( \frac{\mathd \log \bar{\sigma}_t}{\mathd t} -
   \frac{\mathd \log \alpha_t}{\mathd t} \right) = - 2 \bar{\sigma}_t^2 
   \frac{\mathd \lambda_t}{\mathd t}, \qquad \dfrac{\bar{\sigma}_t}{\alpha_t}
   = e^{- \lambda_t} . \]
This way, one can directly relate $\lambda_t$ with the
\tmtextit{signal-to-noise ratio} $\tmop{SNR} (t) = \alpha_t^2 /
\bar{\sigma}^2_t$, also being used in {\cite{dpm-solver}}. As such,
$\tmop{SNR} (t)$ is strictly monotonically decreasing in time. Thus, the
analytic solution to \eqref{eqq} yields
\begin{eqnarray*}
  \mathbf{x}_t & = & e^{\int_s^t f (\tau) \mathd \tau} \mathbf{x}_s + \int_s^t
  \left( e^{\int_{\tau}^t f (r) \mathd r}  \frac{g^2
  (\tau)}{\bar{\sigma}_{\tau}} \tmmathbf{\epsilon}_{\theta}
  (\mathbf{x}_{\tau}, \tau) \right) \mathd \tau + \int_s^t \left(
  e^{\int_{\tau}^t f (r) \mathd r} g (\tau) \right) \mathd
  \tmmathbf{\bar{\omega}} (\tau)\\
  & = & \frac{\alpha_t}{\alpha_s} \mathbf{x}_s + \alpha_t  \int_s^t \frac{g^2
  (\tau)}{\alpha_{\tau}  \bar{\sigma}_{\tau}} \tmmathbf{\epsilon}_{\theta}
  (\mathbf{x}_{\tau}, \tau) \mathd \tau + \alpha_t  \int_s^t \frac{g
  (\tau)}{\alpha_{\tau}} \mathd \tmmathbf{\bar{\omega}} (\tau)\\
  & = & \frac{\alpha_t}{\alpha_s} \mathbf{x}_s - \alpha_t  \int_s^t \frac{2
  \sigma_{\tau}^2}{\alpha_{\tau}  \bar{\sigma}_{\tau}}  \frac{\mathd
  \lambda_{\tau}}{\mathd \tau} \tmmathbf{\epsilon}_{\theta}
  (\mathbf{x}_{\tau}, \tau) \mathd \tau + \alpha_t  \int_s^t \frac{g
  (\tau)}{\alpha_{\tau}} \mathd \tmmathbf{\bar{\omega}} (\tau)\\
  & = & \frac{\alpha_t}{\alpha_s} \mathbf{x}_s - 2 \alpha_t  \int_s^t
  \frac{\bar{\sigma}_{\tau}}{\alpha_{\tau}}  \frac{\mathd
  \lambda_{\tau}}{\mathd \tau} \tmmathbf{\epsilon}_{\theta}
  (\mathbf{x}_{\tau}, \tau) \mathd \tau + \alpha_t  \int_s^t \frac{g
  (\tau)}{\alpha_{\tau}} \mathd \tmmathbf{\bar{\omega}} (\tau)\\
  & = & \frac{\alpha_t}{\alpha_s} \mathbf{x}_s - 2 \alpha_t  \int_s^t e^{-
  \lambda_{\tau}}  \frac{\mathd \lambda_{\tau}}{\mathd \tau}
  \tmmathbf{\epsilon}_{\theta} (\mathbf{x}_{\tau}, \tau) \mathd \tau -
  \sqrt{2} \alpha_t  \int_s^t e^{- \lambda_{\tau}}  \sqrt{\frac{\mathd
  \lambda_{\tau}}{\mathd \tau}} \mathd \tmmathbf{\bar{\omega}} (\tau) .
\end{eqnarray*}
By using the change of variables to $\lambda (t)$, our equation now reads
\begin{eqnarray}
  \mathbf{x}_t & = & \frac{\alpha_t}{\alpha_s} \mathbf{x}_s - 2 \alpha_t 
  \int_{\lambda_s}^{\lambda_t} e^{- \lambda}
  \hat{\tmmathbf{\epsilon}}_{\theta} (\widehat{\mathbf{x}}_{\lambda}, \lambda)
  \mathd \lambda - \sqrt{2} \alpha_t  \int_{\lambda_s}^{\lambda_t} e^{-
  \lambda} \mathd \tmmathbf{\bar{\omega}} (\lambda).  \label{ante}
\end{eqnarray}
Finally, notice that $\alpha_t = \sqrt{\frac{1}{1 + e^{- 2 \lambda_t}}}$ and
$\bar{\sigma}_t = \sqrt{\frac{1}{1 + e^{2 \lambda_t}}}$ so that \eqref{ante}
is equivalent to
\begin{eqnarray*}
  \label{sde} \widehat{\mathbf{x}}_{\lambda_t} & = &
  \frac{\hat{\alpha}_{\lambda_t}}{\hat{\alpha}_{\lambda_s}}
  \widehat{\mathbf{x}}_{\lambda_s} - 2 \hat{\alpha}_{\lambda_t} 
  \int_{\lambda_s}^{\lambda_t} e^{- \lambda}
  \hat{\tmmathbf{\epsilon}}_{\theta} (\widehat{\mathbf{x}}_{\lambda}, \lambda)
  \mathd \lambda - \sqrt{2}  \hat{\alpha}_{\lambda_t} 
  \int_{\lambda_s}^{\lambda_t} e^{- \lambda} \mathd \tmmathbf{\bar{\omega}}
  (\lambda) .
\end{eqnarray*}
This finishes the proof.
\subsubsection{Proof of Proposition \ref{Prop2}}

Recall that the functions $\varphi_k$ are the integrals
\[ \varphi_{k + 1} (t) = \int_0^1 e^{(1 - \delta) t}  \frac{\delta^k}{k!}
   \mathd \delta, \]
which satisfy $\varphi_k (0) = \frac{1}{k!}$. The truncated It{\^o}-Taylor
expansion of $\widehat{\epsilonv}_{\theta}$ with respect to $\lambda$ reads
\[ \widehat{\epsilonv}_{\theta} (\widehat{\mathbf{x}}_{\lambda}, \lambda) =
   \sum_{k = 0}^n \frac{(\lambda - \lambda_s)^k}{k!}
   \widehat{\epsilonv}_{\theta}^{(k)} (\widehat{\mathbf{x}}_{\lambda_s},
   \lambda_s) +\mathcal{R}_{n + 1}, \]
where here $\widehat{\epsilonv}_{\theta}^{(k)}$ denotes the $L_t^k$ operators
defined in \eqref{diff-op1} applied to $\widehat{\epsilonv}_{\theta}$. On the
one hand, since $\int_{\lambda_s}^{\lambda_t} e^{- \lambda} d \lambda =
\dfrac{\bar{\sigma}_t}{\alpha_t} (e^h - 1)$, we obtain by iteratively
integrating by parts
\begin{eqnarray*}
  \int_{\lambda_s}^{\lambda_t} e^{- \lambda} \widehat{\epsilonv}_{\theta}
  (\widehat{\mathbf{x}}_{\lambda}, \lambda) \mathd \lambda & = & \sum_{k =
  0}^n \widehat{\epsilonv}_{\theta}^{(k)} (\widehat{\mathbf{x}}_{\lambda_s},
  \lambda_s)  \int_{\lambda_s}^{\lambda_t} e^{- \lambda}  \frac{(\lambda -
  \lambda_s)^k}{k!} d \lambda +\mathcal{R}_{n + 2}\\
  & = & \frac{\bar{\sigma}_t}{\alpha_t}  \sum_{k = 0}^n
  \widehat{\epsilonv}_{\theta}^{(k)} (\widehat{\mathbf{x}}_{\lambda_s},
  \lambda_s) h^{k + 1} \varphi_{k + 1} (h) +\mathcal{R}_{n + 2} .
\end{eqnarray*}
On the other hand, we have $s > t, h = \lambda_t - \lambda_s > 0$. Note that
since the stochastic integrals $\int_{\lambda_s}^{\lambda_t} e^{- \lambda}
\mathd \tmmathbf{\bar{\omega}} (\lambda)$ are measurable with respect to
$(\tmmathbf{\bar{\omega}} (\lambda) - \tmmathbf{\bar{\omega}} (\lambda_s), 0
\le \lambda \leq \lambda_t - \lambda_s)$, they are independent on disjoint
time intervals by the independence of increments property of Brownian motion.
Thus the random variable $\epsilon \assign \epsilon_{s, t}$ in our algorithms
are independent on disjoint time intervals. We then write
\begin{eqnarray*}
  \int_{\lambda_s}^{\lambda_t} e^{- \lambda} \mathd \tmmathbf{\bar{\omega}}
  (\lambda) & = & \mathcal{N} \left( 0, \int_{\lambda_s}^{\lambda_t} e^{- 2
  \lambda} \mathd \lambda \right)\\
  & = & \frac{1}{\sqrt{2}}  \sqrt{e^{- 2 \lambda_s} - e^{- 2 \lambda_t}}
  \epsilon, \quad \epsilon \sim \mathcal{N} (\tmmathbf{0}, \mathbf{I}_d)\\
  & = & \frac{1}{\sqrt{2}}  \sqrt{\left( \frac{\bar{\sigma}_s}{\alpha_s}
  \right. - \left. \frac{\bar{\sigma}_t}{\alpha_t} \right) \left(
  \frac{\bar{\sigma}_s}{\alpha_s} \right. + \left.
  \frac{\bar{\sigma}_t}{\alpha_t} \right)} \epsilon, \quad \epsilon \sim
  \mathcal{N} (\tmmathbf{0}, \mathbf{I}_d)\\
  & = & \frac{1}{\sqrt{2}}  \frac{\bar{\sigma}_t}{\alpha_t}  \sqrt{e^{2 h} -
  1} \epsilon, \quad \epsilon \sim \mathcal{N} (\tmmathbf{0}, \mathbf{I}_d) .
\end{eqnarray*}
In conclusion, the truncated It{\^o}-Taylor expansion of the analytic
expression
\begin{eqnarray}
  \mathbf{x}_t & = & \frac{\alpha_t}{\alpha_s} \mathbf{x}_s - 2 \alpha_t 
  \int_{\lambda_s}^{\lambda_t} e^{- \lambda} \widehat{\epsilonv}_{\theta}
  (\widehat{\mathbf{x}}_{\lambda}, \lambda) \mathd \lambda - \sqrt{2} \alpha_t
  \int_{\lambda_s}^{\lambda_t} e^{- \lambda} \mathd \tmmathbf{\bar{\omega}}
  (\lambda) 
\end{eqnarray}
simplifies to
\begin{eqnarray*}
  \mathbf{x}_t & = & \cfrac{\alpha_t}{\alpha_s} \mathbf{x}_s - 2
  \bar{\sigma}_t  \sum_{k = 0}^n h^{k + 1} \varphi_{k + 1} (h)
  \widehat{\epsilonv}_{\theta}^{(k)} (\widehat{\mathbf{x}}_{\lambda_s},
  \lambda_s) - \bar{\sigma}_t  \sqrt{e^{2 h} - 1} \epsilon +\mathcal{R}_{n +
  2}, \quad \epsilon \sim \mathcal{N} (\tmmathbf{0}, \mathbf{I}_d) .
\end{eqnarray*}
This finishes the proof.
\subsubsection{Generalization to the remaining data prediction and
deterministic modes}

Propositions \ref{Prop1} and \ref{Prop2} consist on the first steps for
crafting SEEDS solvers in the VP case associated to the noise prediction
neural RSDE (NP NRSDE). Generalizing the above procedure for crafting SEEDS
for the 4 modes associated to \eqref{exp-rep-sde} and \eqref{exp-rep-ode}
yields the following sets of coefficients
\begin{align*}
      A^1 (t) &= \frac{\dot{\alpha}_t}{\alpha_t} + 2
  \frac{\dot{\sigma}_t}{\sigma_t} & B^1 (t) &= - 2 \alpha_t
  \frac{\dot{\sigma}_t}{\sigma_t} \tag{DP NRSDE}\\
    A^2 (t) &= \frac{\dot{\alpha}_t}{\alpha_t} & B^2 (t) &= 2
  \alpha_t \dot{\sigma}_t \tag{NP NRSDE}\\
    A^3 (t) &= \frac{\dot{\alpha}_t}{\alpha_t} +
  \frac{\dot{\sigma}_t}{\sigma_t} & B^3 (t) &= - \alpha_t
  \frac{\dot{\sigma}_t}{\sigma_t} \tag{DP NPFO}\\
    A^4 (t) &= \frac{\dot{\alpha}_t}{\alpha_t} & B^4 (t) &=
  \alpha_t \dot{\sigma}_t\tag{NP NPFO}
\end{align*}
We readily obtain
\[ \Phi^2 (t, s) = \Phi^4 (t, s) = \frac{\alpha_t}{\alpha_s}, \qquad \Phi^3
   (t, s) = \frac{\bar{\sigma}_t}{\bar{\sigma}_s}, \qquad \Phi^1 (t, s) =
   \frac{\sigma^2_t \alpha_t}{\sigma^2_s \alpha_s} . \]
Then, by setting the simpler change of variables $\lambda_t : = - \log
(\sigma_t)$, we obtain
\[ \int^t_s \Phi^4 (t, \tau) B^4 (\tau) \mathd \tau = \alpha_t  \int^t_s
   \frac{1}{\alpha_{\tau}} \alpha_{\tau}  \dot{\sigma}_{\tau} \mathd \tau =
   \alpha_t  \int^t_s \dot{\sigma}_{\tau} \mathd \tau = - \alpha_t 
   \int^{\lambda_t}_{\lambda_s} e^{- \lambda} \mathd \lambda = -
   \bar{\sigma}_t  (e^h - 1) . \]
By recursion we obtain
\[ \int^t_s \Phi^4 (t, \tau) B^4 (\tau) F_{\theta, \tau} \mathd \tau = -
   \bar{\sigma}_t  \sum^{n - 1}_{k = 0} h^{k + 1} \varphi_{k + 1} (h)
   F^{(k)}_{\theta, s} +\mathcal{O} (h^{n + 1}) . \]
In the same way, we obtain
\begin{eqnarray*}
  \int^t_s \Phi^3 (t, \tau) B^3 (\tau) \mathd \tau = \sigma_t \alpha_t 
  \int^t_s \frac{- 1}{\alpha_{\tau} \sigma_{\tau}} \alpha_{\tau} 
  \frac{\dot{\sigma}_{\tau}}{\sigma_{\tau}} \mathd \tau & = & \sigma_t
  \alpha_t  \int^t_s \frac{- \dot{\sigma}_{\tau}}{\sigma^2_{\tau}} \mathd
  \tau\\
  & = & \sigma_t \alpha_t  \int^{\lambda_t}_{\lambda_s} e^{\lambda} \mathd
  \lambda = - \alpha_t  (e^{- h} - 1) .
\end{eqnarray*}
Next,
\[ \int^t_s \Phi^2 (t, \tau) B^2 (\tau) \mathd \tau = \alpha_t  \int^t_s
   \frac{2}{\alpha_{\tau}} \alpha_{\tau}  \dot{\sigma}_{\tau} \mathd \tau = -
   2 \bar{\sigma}_t  (e^h - 1), \]
and finally, as already shown in Propositions \ref{Prop1} and \ref{Prop2}:
\begin{eqnarray*}
  \int^t_s \Phi^1 (t, \tau) B^1 (\tau) \mathd \tau = \sigma^2_t \alpha_t 
  \int^t_s \frac{- 2 \dot{\sigma}_{\tau}}{\sigma^3_{\tau}} \mathd \tau & = &
  \sigma^2_t \alpha_t  \int^t_s - 2 \frac{\mathd \lambda_{\tau}}{\mathd \tau}
  e^{2 \lambda_{\tau}} \mathd \tau\\
  & = & \sigma^2_t \alpha_t  \int^{\lambda_t}_{\lambda_s} e^{2 \lambda}
  \mathd \lambda = - \alpha_t  (e^{- 2 h} - 1) .
\end{eqnarray*}
Now, since $g^2 (t) = 2 \alpha^2_t  \dot{\sigma}_t \sigma_t$, the stochastic
integrals $\int^t_s \Phi^i (t, \tau) g (\tau) \mathd
\tmmathbf{\bar{\omega}}_{\tau}$, for $i = 1, 2$, have zero mean and variances:
\begin{eqnarray*}
  \int^t_s (\Phi^1 (t, \tau))^2 g^2 (\tau) \mathd \tau & = & \sigma^4_t
  \alpha^2_t  \int^t_s \frac{1}{\sigma^4_{\tau} \alpha^2_{\tau}} g^2 (\tau)
  \mathd \tau = \bar{\sigma}_t^2  (1 - e^{- 2 h})\\
  \int^t_s (\Phi^2 (t, \tau))^2 g^2 (\tau) \mathd \tau & = & \alpha^2_t 
  \int^t_s 2 \dot{\sigma}_{\tau} \sigma_{\tau} \mathd \tau = - \bar{\sigma}_t^2  (e^{2 h} - 1) .
\end{eqnarray*}
We deduce from this the SEEDS-1 schemes in all four modes as given by
iterates:
\begin{eqnarray}
  \widetilde{\mathbf{x}}_t & = & \frac{\sigma^2_t \alpha_t}{\sigma^2_s
  \alpha_s} \widetilde{\mathbf{x}}_s - \alpha_t  (e^{- 2 h} - 1) D_{\theta}
  (\widetilde{\mathbf{x}}_s, s) + \bar{\sigma}_t  \sqrt{1 - e^{- 2 h}}
  \epsilon \quad \epsilon \sim \mathcal{N} (\tmmathbf{0}, \mathbf{I}_d) 
  \label{dp-seeds1}\\
  \widetilde{\mathbf{x}}_t & = & \cfrac{\alpha_t}{\alpha_s}
  \widetilde{\mathbf{x}}_s - 2 \bar{\sigma}_t  (e^h - 1)
  \tmmathbf{\epsilon}_{\theta} (\widetilde{\mathbf{x}}_s, s) - \bar{\sigma}_t 
  \sqrt{e^{2 h} - 1} \epsilon \quad \epsilon \sim \mathcal{N} (\tmmathbf{0},
  \mathbf{I}_d)  \label{np-seeds1}\\
  \widetilde{\mathbf{x}}_t & = & \frac{\bar{\sigma}_t}{\bar{\sigma}_s}
  \widetilde{\mathbf{x}}_s - \alpha_t  (e^{- h} - 1) D_{\theta}
  (\widetilde{\mathbf{x}}_s, s)  \label{dpmpp1}\\
  \widetilde{\mathbf{x}}_t & = & \cfrac{\alpha_t}{\alpha_s}
  \widetilde{\mathbf{x}}_s - \bar{\sigma}_t  (e^h - 1)
  \tmmathbf{\epsilon}_{\theta} (\widetilde{\mathbf{x}}_s, s) .  \label{dpm1}
\end{eqnarray}
Notice that the iterates \eqref{dpmpp1} and \eqref{dpm1} are exactly the
iterates of the first stage solvers in {\cite{dpm-solver}} and
{\cite{lu2022dpm}} with $F_{\theta} (\mathbf{x}_t ; (M - 1) t)
=\tmmathbf{\epsilon}_{\theta}  (\mathbf{x}_t, t)$. The iterates
\eqref{np-seeds1} coincide with the SEEDS-1 method presented in
\eqref{seeds-1} and \eqref{dp-seeds1} consist on our SEEDS-1 method in the
\tmtextit{data prediction mode}, which we will use in the following section.
\subsubsection{Proof of Proposition \ref{related}}

We will write in \tmtextbf{bold} the statement to be proven.

 \tmtextbf{If we set $g = 0$ in \eqref{nrsde-vp}, the resulting SEEDS
  solvers do not yield DPM-Solver.} Indeed, if we set $g = 0$ in
  \eqref{nrsde-vp}, then the method \eqref{np-seeds1} does not contain a noise
  contribution and we readily see that it cannot be equal to \eqref{dpm1}. As
  the latter has been shown to be DPM-Solver-1, the conclusion follows.
  
{\bf{If we parameterize \eqref{nrsde-vp} in terms of the data
  prediction model $D_{\theta}$, the resulting SEEDS solvers are not
  equivalent to their noise prediction counterparts defined in Alg.
  \ref{alg:iter} to \ref{alg:SERK-solver-3}.}}
  
  One can check that the SEEDS solver are not the same between the noise and
  data prediction modes by simply noticing that the noise contributions in
  \eqref{dp-seeds1} and \eqref{np-seeds1} do not equate.
  
{\bf{The gDDIM solver {\cite{zhang2022gddim}} Theorem 1, for
  $\ell = 1$, is equal to SEEDS-1 in the data prediction mode.}}

As shown in \eqref{np-seeds1}, our proposed method SEEDS-1 in the data
prediction mode for the VP case has iterates of the form
\begin{eqnarray}
  \widetilde{\mathbf{x}}_t & = & \frac{\sigma_t^2 \alpha_t}{\sigma_s^2
  \alpha_s} \widetilde{\mathbf{x}}_s - \alpha_t  (e^{- 2 h} - 1) D_{\theta}
  (\widetilde{\mathbf{x}}_s, s) + \bar{\sigma}_t  \sqrt{1 - e^{- 2 h}}
  \epsilon,  \label{eq:seeds1-dp}
\end{eqnarray}
where $\epsilon \sim \mathcal{N} (\tmmathbf{0}, \mathbf{I}_d)$,
$\bar{\sigma}_t = \alpha_t \sigma_t$ and $h = \log
\dfrac{\sigma_s}{\sigma_t}$. As our notation and that of {\cite[Theorem
1]{zhang2022gddim}} overlap, we will use \tmcolor{blue}{blue color} when
referring to their notation.

On the one hand, gDDIM constructs iterates over a representation of the exact
solution of the following family of neural differential equations:
\begin{eqnarray}
  \tmcolor{blue}{\mathd \mathbf{u}_t} & = & \tmcolor{blue}{\left[ f (t)
  \textbf{$\mathbf{u}$}_t + \frac{1 + \lambda^2}{2} \frac{g^2 (t)}{\sqrt{1 -
  \alpha_t}} \epsilon_{\theta} (\textbf{$\mathbf{u}$}_t ; t) \right] \mathd t
  + \lambda g (t) \mathd \tmmathbf{\bar{\omega}}_t},  \label{gDDIM-ode}
\end{eqnarray}
where $\tmcolor{blue}{\alpha_t}$ decreases from $\tmcolor{blue}{\alpha_0} = 1$
to $\tmcolor{blue}{\alpha_T} = 0$, and with coefficients
\[ \tmcolor{blue}{f (t) \assign \frac{1}{2} \frac{\mathd \log
   \alpha_t}{\mathd t}, \quad g (t) = \sqrt{- \frac{\mathd \log
   \alpha_t}{\mathd t}}}. \]
In particular, they choose an approximation, for $\tmcolor{blue}{\tau \in [t -
\Delta t, t]}$, given by
\[ \tmcolor{blue}{\tmmathbf{s}_{\theta} (\mathbf{u}, \tau) =
   \frac{\epsilon_{\theta}  (\textbf{$\mathbf{u}$}_{\tau}, \tau)}{\sqrt{1 -
   \alpha_t}} \approx \frac{1 - \alpha_t}{1 - \alpha_{\tau}} 
   \sqrt{\frac{\alpha_{\tau}}{\alpha_t}} \tmmathbf{} \tmmathbf{s}_{\theta}
   (\mathbf{u}(t), t) - \frac{1}{1 - \alpha_{\tau}}  (\mathbf{u}-
   \sqrt{\frac{\alpha_{\tau}}{\alpha_t}} \mathbf{u}(t))} . \]
The gDDIM iterates, for $\tmcolor{blue}{\lambda} = 1 = \ell$, are then written
as follows:
\begin{equation}
  \tmcolor{blue}{\mathbf{u} (t - \Delta t) = \sqrt{\frac{\alpha_{t - \Delta
  t}}{\alpha_t}} \mathbf{u} (t) + \left[ - \sqrt{\frac{\alpha_{t - \Delta
  t}}{\alpha_t}}  \sqrt{1 - \alpha_t} + \sqrt{1 - \alpha_{t - \Delta t} -
  \sigma_t^2} \right] \epsilon_{\theta} (\mathbf{u}(t), t) + \sigma_t
  \epsilon}, \label{gddim}
\end{equation}
where $\epsilon \sim \mathcal{N} (\tmmathbf{0}, \mathbf{I}_d)$ and
\begin{equation}
  \tmcolor{blue}{\sigma^2_t = (1 - \alpha_{t - \Delta t})  \left[ 1 - \left(
  \frac{1 - \alpha_{t - \Delta t}}{1 - \alpha_t} \right) \left(
  \frac{\alpha_t}{\alpha_{t - \Delta t}} \right) \right]} . \label{gddim-var}
\end{equation}
Now set $\tmcolor{blue}{(s, t) \leftarrow (t, t - \Delta t)}$. Then
\begin{eqnarray*}
  \tmcolor{blue}{\sigma^2_s} & = & \tmcolor{blue}{(1 - \alpha_t) \left[ 1 - \left( \frac{1 - \alpha_t}{1 -
  \alpha_s} \right) \left( \frac{\alpha_s}{\alpha_t} \right) \right]},\\
  \tmcolor{blue}{\mathbf{u} (t)} & = &
  \tmcolor{blue}{\sqrt{\frac{\alpha_t}{\alpha_s}} \mathbf{u} (s) + \left[
  \sqrt{1 - \alpha_t - \sigma_s^2} - \sqrt{\frac{\alpha_t}{\alpha_s}}  \sqrt{1
  - \alpha_s} \right] \epsilon_{\theta} (\mathbf{u}(s), s) + \sigma_s
  \epsilon} .
\end{eqnarray*}
Next, we identify $\alpha_t = \tmcolor{blue}{\sqrt{\alpha_t}}$, and $\bar{\sigma}_t =
\tmcolor{blue}{\sqrt{1 - \alpha_t}}$. Then the variance of the noise in
\eqref{gddim-var} is
\begin{eqnarray*}
  {\color[HTML]{000000}\tmcolor{blue}{\sigma_s^2}} & = & \bar{\sigma}_t^2 
  \left[ 1 - \left( \frac{\bar{\sigma}_t}{\bar{\sigma}_s} \right)^2 \left(
  \cfrac{\alpha_s}{\alpha_t} \right)^2 \right] = \bar{\sigma}_t^2  \left[ 1 -
  \left( \frac{\alpha_t \sigma_t}{\alpha_s \sigma_s} \right)^2 \left(
  \cfrac{\alpha_s}{\alpha_t} \right)^2 \right]\\
  & = & \bar{\sigma}_t^2  \left[ 1 - \left( \frac{\sigma_t}{\sigma_s}
  \right)^2 \right]\\
  & = & \bar{\sigma}_t^2  (1 - e^{- 2 h}) .
\end{eqnarray*}
Hence, by denoting $\widetilde{\mathbf{x}}_t = \tmcolor{blue}{\mathbf{u}
(t)}$, the gDDIM iterate \eqref{gddim} reads, for $\epsilon \sim \mathcal{N}
(\tmmathbf{0}, \mathbf{I}_d)$:
\begin{eqnarray}
  \widetilde{\mathbf{x}}_t & = & \cfrac{\alpha_t}{\alpha_s}
  \widetilde{\mathbf{x}}_s + \left[ \sqrt{\bar{\sigma}_t^2 - \bar{\sigma}_t^2 
  (1 - e^{- 2 h})} - \cfrac{\alpha_t}{\alpha_s} \bar{\sigma}_s \right]
  \epsilon_{\theta}  (\widetilde{\mathbf{x}}_s, s) + \sqrt{\bar{\sigma}_t^2 
  (1 - e^{- 2 h})} \epsilon \nonumber\\
  & = & \cfrac{\alpha_t}{\alpha_s} \widetilde{\mathbf{x}}_s + \left[
  \bar{\sigma}_t  \sqrt{e^{- 2 h}} - \alpha_t \sigma_s \right]
  \epsilon_{\theta}  (\widetilde{\mathbf{x}}_s, s) + \bar{\sigma}_t  \sqrt{1 -
  e^{- 2 h}} \epsilon \nonumber\\
  & = & \cfrac{\alpha_t}{\alpha_s} \widetilde{\mathbf{x}}_s + \bar{\sigma}_t 
  \left[ \frac{\sigma_t}{\sigma_s} - \frac{\sigma_s}{\sigma_t} \right]
  \epsilon_{\theta}  (\widetilde{\mathbf{x}}_s, s) + \bar{\sigma}_t  \sqrt{1 -
  e^{- 2 h}} \epsilon .  \label{eq:gddim-re}
\end{eqnarray}
On the other hand, the data and noise prediction models in our case are
related by the following equation:
\[ D_{\theta} (\widetilde{\mathbf{x}}_s, s) = c_1 (s)
   \dfrac{\widetilde{\mathbf{x}}_s}{\alpha_s} + c_2 (s)
   \tmmathbf{\epsilon}_{\theta}  (\widetilde{\mathbf{x}}_s, s) =
   \dfrac{\widetilde{\mathbf{x}}_s}{\alpha_s} - \sigma_s
   \tmmathbf{\epsilon}_{\theta}  (\widetilde{\mathbf{x}}_s, s) . \]
As such, and as $h = \lambda_t - \lambda_s = \log \dfrac{\sigma_s}{\sigma_t}$,
one can rewrite \eqref{eq:seeds1-dp} in terms of
$\tmmathbf{\epsilon}_{\theta}$ as follows:
\begin{eqnarray*}
  \widetilde{\mathbf{x}}_t & = & \frac{\sigma_t^2 \alpha_t}{\sigma_s^2
  \alpha_s} \widetilde{\mathbf{x}}_s - \alpha_t  (e^{- 2 h} - 1) D_{\theta}
  (\widetilde{\mathbf{x}}_s, s) + \bar{\sigma}_t  \sqrt{1 - e^{- 2 h}}
  \epsilon\\
  & = & \frac{\sigma_t^2 \alpha_t}{\sigma_s^2 \alpha_s}
  \widetilde{\mathbf{x}}_s - \alpha_t  \left( \frac{\sigma_t^2}{\sigma_s^2} -
  1 \right)  \left[ \dfrac{\widetilde{\mathbf{x}}_s}{\alpha_s} - \sigma_s
  \tmmathbf{\epsilon}_{\theta} (\widetilde{\mathbf{x}}_s, s) \right] +
  \bar{\sigma}_t  \sqrt{1 - e^{- 2 h}} \epsilon\\
  & = & \left[ \frac{\sigma_t^2 \alpha_t}{\sigma_s^2 \alpha_s} -
  \frac{\alpha_t}{\alpha_s}  \left( \frac{\sigma_t^2}{\sigma_s^2} - 1 \right)
  \right] \widetilde{\mathbf{x}}_s + \alpha_t  \left(
  \frac{\sigma_t^2}{\sigma_s^2} - 1 \right) \sigma_s
  \tmmathbf{\epsilon}_{\theta}  (\widetilde{\mathbf{x}}_s, s) + \bar{\sigma}_t
  \sqrt{1 - e^{- 2 h}} \epsilon\\
  & = & \cfrac{\alpha_t}{\alpha_s} \widetilde{\mathbf{x}}_s + \bar{\sigma}_t 
  \left( \frac{\sigma_t^2}{\sigma_s^2} - 1 \right)  \frac{\sigma_s}{\sigma_t}
  \tmmathbf{\epsilon}_{\theta}  (\widetilde{\mathbf{x}}_s, s) + \bar{\sigma}_t
  \sqrt{1 - e^{- 2 h}} \epsilon\\
  & = & \cfrac{\alpha_t}{\alpha_s} \widetilde{\mathbf{x}}_s + \bar{\sigma}_t 
  \left( \frac{\sigma_t}{\sigma_s} - \frac{\sigma_s}{\sigma_t} \right)
  \tmmathbf{\epsilon}_{\theta}  (\widetilde{\mathbf{x}}_s, s) + \bar{\sigma}_t
  \sqrt{1 - e^{- 2 h}} \epsilon,
\end{eqnarray*}
which coincides with the gDDIM iterate in Equation \eqref{eq:gddim-re}.
\subsubsection{The VE, DDIM and iDDPM cases}

Following {\cite[Eq. 217]{Karras2022edm}}, here $\alpha_t = 1$ and $c_1 (t) =
1$ and so the only possibilities incurring into semi-linear differential
equations are
\begin{align*}
      A^1 (t) &= 2 \frac{\dot{\sigma}_t}{\sigma_t} & B^1 (t) &= - 2
  \frac{\dot{\sigma}_t}{\sigma_t} \tag{DP SDE}\\
    A^3 (t) &= \frac{\dot{\sigma}_t}{\sigma_t} & B^3 (t) &= -
  \frac{\dot{\sigma}_t}{\sigma_t}. \tag{DP PFO}
\end{align*}
We obtain, again with the choice $\lambda_t = - \log (\sigma_t)$ and setting
$h = \lambda_t - \lambda_s$, the following:
\[ \Phi^1 (t, s) = \frac{\sigma^2_t}{\sigma^2_s}, \qquad \Phi^3 (t, s) =
   \frac{\sigma_t}{\sigma_s}, \]
\begin{eqnarray*}
  \int^t_s \Phi^1 (t, \tau) B^1 (\tau) \mathd \tau & = & - \sigma^2_t 
  \int^{\lambda_t}_{\lambda_s} e^{2 \lambda} \mathd \lambda = \frac{1}{2} 
  (e^{- 2 h} - 1)\\
  \int^t_s \Phi^3 (t, \tau) B^3 (\tau) \mathd \tau & = & - \sigma_t 
  \int^{\lambda_t}_{\lambda_s} e^{\lambda} \mathd \lambda = e^{- h} - 1.
\end{eqnarray*}
Next,
\begin{eqnarray*}
  \int^t_s (\Phi^1 (t, \tau))^2 g^2 (\tau) \mathd \tau & = & \sigma^4_t 
  \int^t_s \frac{1}{\sigma^4_{\tau}} \sigma^2_{\tau} 2
  \frac{\dot{\sigma}_{\tau}}{\sigma_{\tau}} \mathd \tau = \sigma^4_t  \int^t_s
  \frac{\dot{2 \sigma}_{\tau}}{\sigma^3_{\tau}} \mathd \tau = \sigma_t^2 
  (e^{- 2 h} - 1) .
\end{eqnarray*}
We readily see that these cases are identical to the VP case with $\alpha_t =
1$. In particular, the obtained SEEDS-1 iterates are
\begin{eqnarray*}
  \widetilde{\mathbf{x}}_t & = & \frac{\sigma^2_t}{\sigma^2_s}
  \widetilde{\mathbf{x}}_s - (e^{- 2 h} - 1) D_{\theta}
  (\widetilde{\mathbf{x}}_s ; \sigma_s) + \sigma_t  \sqrt{1 - e^{- 2 h}}
  \epsilon .\\
  \widetilde{\mathbf{x}}_t & = & \frac{\sigma_t}{\sigma_s}
  \widetilde{\mathbf{x}}_s - (e^{- h} - 1) D_{\theta}
  (\widetilde{\mathbf{x}}_s ; \sigma_s).
\end{eqnarray*}
Now denote $s_1 = t_{\lambda}  (\lambda_s + rh)$, for $0 < r \leqslant 1$,
where $t_{\lambda} = e^{- \lambda}$. There are two families of single-step
one-parameter two-stage exponential ODE schemes:
\begin{eqnarray*}
  \widetilde{\mathbf{x}}_t & = & \frac{\sigma_t}{\sigma_s}
  \widetilde{\mathbf{x}}_s - (e^{- h} - 1)  \left[ \left( 1 - \frac{1}{2 r}
  \right) D_{\theta} (\widetilde{\mathbf{x}}_s ; \sigma_s) + \frac{1}{2 r}
  D_{\theta} (\widetilde{\mathbf{x}}_1 ; \sigma_{s_1}) \right]\\
  \widetilde{\mathbf{x}}_t & = & \frac{\sigma_t}{\sigma_s}
  \widetilde{\mathbf{x}}_s - (e^{- h} - 1) D_{\theta}
  (\widetilde{\mathbf{x}}_s ; \sigma_s) + \frac{1}{r}  \left( \frac{e^{- h} -
  1}{h} + 1 \right)  [D_{\theta} (\widetilde{\mathbf{x}}_1 ; \sigma_{s_1}) -
  D_{\theta} (\widetilde{\mathbf{x}}_s ; \sigma_s)],
\end{eqnarray*}
with same supporting value
\begin{eqnarray*}
  \widetilde{\mathbf{x}}_1 & = & \frac{\sigma_{s_1}}{\sigma_s}
  \widetilde{\mathbf{x}}_s - (e^{- rh} - 1) D_{\theta}
  (\widetilde{\mathbf{x}}_s ; \sigma_s) .
\end{eqnarray*}
In the same vein we define a single-step two stage exponential SDE scheme:
\begin{eqnarray*}
  \widetilde{\mathbf{x}}_1 & = & \frac{\sigma^2_{s_1}}{\sigma^2_s}
  \widetilde{\mathbf{x}}_s - (e^{- 2 rh} - 1) D_{\theta}
  (\widetilde{\mathbf{x}}_s ; \sigma_s) + \sigma_{s_1}  \sqrt{e^{- 2 rh} - 1}
  \epsilon_1\\
  \widetilde{\mathbf{x}}_t & = & \frac{\sigma^2_t}{\sigma^2_s}
  \widetilde{\mathbf{x}}_s - (e^{- 2 h} - 1)  \left[ \left( 1 - \frac{1}{2 r}
  \right) D_{\theta} (\widetilde{\mathbf{x}}_s ; \sigma_s) + \frac{1}{2 r}
  D_{\theta} (\widetilde{\mathbf{x}}_1 ; \sigma_{s_1}) \right]\\
  &  & + \sigma_t  \left[ \sqrt{e^{- 2 h} - e^{- 2 rh}} \epsilon_1 +
  \sqrt{e^{- 2 rh} - 1} \epsilon_2 \right] .
\end{eqnarray*}
\subsubsection{The EDM case}

In the EDM-preconditioned case {\cite[App. B.6]{Karras2022edm}}, we set
$\sigma_t : = t$ and $\alpha_t : = 1$. We denote $\sigma_d : =
\sigma_{\tmop{data}}$ the variance of the considered initial dataset and we
set
\[ c_1 (t) = \frac{\sigma^2_d}{t^2 + \sigma^2_d}, \quad c_2 (t) = \frac{t
   \sigma_d}{\sqrt{t^2 + \sigma^2_d}}, \quad c_3 (t) = \frac{1}{\sqrt{t^2 +
   \sigma^2_d}}, \quad c_4 (t) = \frac{1}{4} \log (t), \]
so we obtain the following coefficients:
\begin{align*}
       A^1 (t) &= \frac{2}{t} & B^1 (t) &= - \frac{2}{t} \tag{DP NRSDE} \\
   A^2 (t) &= \frac{2}{t} \left( 1 - \frac{\sigma^2_d}{t^2 +
  \sigma^2_d} \right) & B^2 (t) &= \frac{- 2\sigma_d}{\sqrt{t^2 + \sigma^2_d}} \tag{NP NRSDE} \\
    A^3 (t) &= \frac{1}{t} & B^3 (t) &= - \frac{1}{t} \tag{DP NPFO}\\
   A^4 (t) &= \frac{1}{t} \left( 1 - \frac{\sigma^2_d}{t^2 +
  \sigma^2_d} \right) & B^4 (t) &= \frac{- \sigma_d}{\sqrt{t^2 + \sigma^2_d}} \tag{NP NPFO}
\end{align*}
In particular, the data prediction neural SDE/PFO are identical to those in the VE case with $\sigma_t = t$. So let us concentrate on the noise prediction regime, leading us to prove
Proposition \ref{Prop1-edm}.
\subsubsection{Proof of Proposition \ref{Prop1-edm}}
In the Noise Prediction case, we have
\[ \Phi^2 (t, s) = \frac{t^2 + \sigma_d^2}{s^2 + \sigma_d^2}, \quad \Phi^4 (t,
   s) = \sqrt{\frac{t^2 + \sigma_d^2}{s^2 + \sigma_d^2}}, \]
so we readily compute:
\begin{eqnarray*}
  \int^t_s \Phi^2 (t, \tau) B^2 (\tau) \mathd \tau & = & (t^2 + \sigma_d^2) 
  \int^t_s \frac{1}{\tau^2 + \sigma_d^2} \cdot \frac{- 2
  \sigma_d}{\sqrt{\tau^2 + \sigma^2_d}} \mathd \tau,\\
  \int^t_s \Phi^4 (t, \tau) B^4 (\tau) \mathd \tau & = & \sqrt{t^2 +
  \sigma_d^2}  \int^t_s \frac{\sigma_d}{\sqrt{\tau^2 + \sigma_d^2}} \cdot
  \frac{- 1}{\sqrt{\tau^2 + \sigma^2_d}} \mathd \tau .
\end{eqnarray*}
Let us consider two different changes of variables:
\begin{equation}
\label{edm-lambdas}
    \lambda_t \assign - \log \left( \arctan \left( \dfrac{t}{\sigma_d} \right)
   \right)  \qquad \text{and} \qquad \lambda_t \assign - \log \left(
   \dfrac{t}{\sigma_d  \sqrt{t^2 + \sigma_d^2}} \right),
\end{equation}
that will be used for the (NP NPFO) and (NP NRSDE), respectively. For the
former case, we have
\[ e^{- \lambda_t} \mathd \lambda_t = - \frac{\frac{1}{\sigma_d} \mathd t}{1 +
   \frac{t^2}{\sigma_d^2}} = - \frac{\sigma_d \mathd t}{\sigma_d^2 + t^2} . \]
Therefore, we can deduce that
\begin{eqnarray*}
  \int_s^t \Phi^4 (t ; \tau) B^4 (\tau) \mathd \tau & = & \int_s^t
  \sqrt{\frac{t^2 + \sigma_d^2}{\tau^2 + \sigma_d^2}} . \frac{-
  \sigma_d}{\sqrt{\tau^2 + \sigma_d^2}} d \tau\\
  & = & \sqrt{t^2 + \sigma_d^2}  \int_{\lambda_s}^{\lambda_t} \frac{-
  \sigma_d}{\tau^2 + \sigma_d^2} \mathd \tau\\
  & = & \sqrt{t^2 + \sigma_d^2}  \int_{\lambda_s}^{\lambda_t} e^{- \lambda}
  \mathd \lambda\\
  & = & \sqrt{t^2 + \sigma_d^2} \arctan \left( \frac{t}{\sigma_d} \right) 
  (e^h - 1) .
\end{eqnarray*}
For the latter case, we have
\begin{eqnarray*}
  e^{- \lambda_t} \mathd \lambda_t & = & - \frac{\sigma_d  \sqrt{t^2 +
  \sigma_d^2} - t \sigma_d \dfrac{t}{\sqrt{t^2 + \sigma_d^2}}}{\sigma_d^2 
  (t^2 + \sigma_d^2)}\\
  & = & - \frac{\sigma_d^2}{\sigma_d  (t^2 + \sigma_d^2)  \sqrt{t^2 +
  \sigma_d^2}} = - \frac{\sigma_d}{(t^2 + \sigma_d^2)  \sqrt{t^2 +
  \sigma_d^2}} .
\end{eqnarray*}
We then obtain
\begin{eqnarray*}
  \int_s^t \Phi^2 (t ; \tau) B^2 (\tau) \mathd \tau & = & \int_s^t \frac{t^2 +
  \sigma_d^2}{\tau^2 + \sigma_d^2} \cdot \frac{- 2 \sigma_d}{\sqrt{\tau^2 +
  \sigma_d^2}} \mathd \tau\\
  & = & 2 (t^2 + \sigma_d^2)  \int_s^t e^{- \lambda} \mathd \lambda\\
  & = & \frac{2 t \sqrt{t^2 + \sigma_d^2}}{\sigma_d}  (e^h - 1) .
\end{eqnarray*}
The stochastic integral $\int^t_s \Phi^2 (t, \tau) g (\tau) \mathd
\tmmathbf{\bar{\omega}}_{\tau}$ in noise prediction case is a Gaussian random
variable with zero mean and whose variance can be computed by It{\^o} isometry
as
\begin{eqnarray*}
  \int_t^s [\Phi^2 (t ; \tau) g^2 (\tau)]^2 \mathd \tau & = & (t^2 +
  \sigma_d^2)^2  \int_t^s \frac{1}{[\tau^2 + \sigma_d^2]^2} 2 \tau \mathd
  \tau\\
  & = & (t^2 + \sigma_d^2)^2  \int_t^s \frac{1}{[\tau^2 + \sigma_d^2]^2}
  \mathd \tau^2\\
  & = & (t^2 + \sigma_d^2)^2  \left( - \frac{1}{s^2 + \sigma_d^2} +
  \frac{1}{t^2 + \sigma_d^2} \right)\\
  & = & \frac{(t^2 + \sigma_d^2)  (s^2 - t^2)}{(s^2 + \sigma_d^2)} .
\end{eqnarray*}
Putting everything together, we obtain the analytic solution at time $t$ of
\eqref{eqq} with coefficients \eqref{edm-coefs} and initial value
$\mathbf{x}_s$ for the (NP NRSDE):
\begin{equation}
  \mathbf{x}_t = \frac{t^2 + \sigma^2_d}{s^2 + \sigma^2_d} \mathbf{x}_s + 2
  (t^2 + \sigma^2_d)  \int_{\lambda_s}^{\lambda_t} e^{- \lambda} 
  \hat{F}_{\theta} (\widehat{\mathbf{x}}_{\lambda}, \lambda) \mathd \lambda -
  \sqrt{2} (t^2 + \sigma^2_d) \int_{\lambda_s}^{\lambda_t} e^{- \lambda }
  \mathd \overline{\tmmathbf{\omega} }_{\lambda},
\end{equation}
where $\lambda_t : = - \log \left[ \frac{t}{\sigma_d \sqrt{t^2 + \sigma^2_d}}
\right]$. For the (NP NPFO), it is given by
\begin{equation}
  \mathbf{x}_t = \sqrt{\frac{t^2 + \sigma^2_d}{s^2 + \sigma^2_d}} \mathbf{x}_s
  + \sqrt{t^2 + \sigma^2_d} \int_{\lambda_s}^{\lambda_t} e^{- \lambda} 
  \hat{F}_{\theta} (\widehat{\mathbf{x}}_{\lambda}, \lambda) \mathd \lambda,
  \quad \lambda_t : = - \log \left[ \arctan \left[ \frac{t}{\sigma_d} \right]
  \right] .
\end{equation}
This finishes the proof.
\begin{remark}
  From the above proof, we immediately deduce the SEEDS-1 iterates in the
  EDM-preconditioned noise prediction case. These are given, for the (NP NPFO)
  and (NP NRSDE) respectively, by
  \begin{eqnarray}
    \widetilde{\mathbf{x}}_t & = & \sqrt{\frac{t^2 + \sigma_d^2}{s^2 +
    \sigma_d^2}} \widetilde{\mathbf{x}}_s + \sqrt{t^2 + \sigma_d^2} \arctan
    \left( \frac{t}{\sigma_d} \right)  (e^h - 1) \epsilon_{\theta} 
    (\widetilde{\mathbf{x}}_s, s), \\
    \widetilde{\mathbf{x}}_t & = & \frac{t^2 + \sigma_d^2}{s^2 + \sigma_d^2}
    \widetilde{\mathbf{x}}_s + \frac{2 t \sqrt{t^2 + \sigma_d^2}}{\sigma_d} 
    (e^h - 1) \epsilon_{\theta}  (\widetilde{\mathbf{x}}_s, s) +
    \sqrt{\frac{(t^2 + \sigma_d^2)  (s^2 - t^2)}{(s^2 + \sigma_d^2)}}
    \epsilon, 
  \end{eqnarray}
  where $\epsilon \sim \mathcal{N} (\tmmathbf{0}, \mathbf{I}_d)$.
\end{remark}

\subsubsection{The SEEDS algorithms in the EDM Noise Prediction case}

In the same way we presented the SEEDS algorithms \ref{alg:SERK-solver-2} and \ref{alg:SERK-solver-3} for the VP case, the change-of-variables optimized for the EDM framework in Proposition \ref{Prop1-edm} induces algorithms \ref{alg:SERK-solver-2-edm} for SEEDS-2 and \ref{alg:SERK-solver-3-edm} for SEEDS-3 respectively, under the Noise Prediction regime. This is the version of SEEDS that we use to achieve same performance as EDM solver but with twice less NFEs than the latter for ImageNet-64 (see Table \ref{table-fid-all} and Fig. \ref{fig:fid-cont} (c) in the main part of this article).

\begin{algorithm}[ht]
   \caption{SEEDS-2 (EDM NP)}
   \label{alg:SERK-solver-2-edm}
\small 
\begin{algorithmic}
\State {\bfseries Input:} initial value $\x_T$, steps $\{t_i\}^M_{i=0}$, data prediction model $D_\theta$, $r=1/2 $
   \State Initialize $\tilde\x_{t_0}\gets\x_T$
   \For{$i=1$ {\bfseries to} $M-1$}
   \State $(t,s) \gets (t_{i},t_{i-1}), \quad (z^1,z^2) \gets \mathcal{N}(0,\text{Id})^{\otimes2} $
   \State $\lambda (t) \gets - \log \left( \frac{t}{\sigma_d \sqrt{t^2 + \sigma^2_d}}
\right) \qquad t_{\lambda} (\lambda) \gets
\frac{\sigma_d}{\sqrt{\frac{1}{\sigma^2_d e^{- 2 \lambda}} - 1}}, \quad
\sigma_d \gets \sigma_{data}$
   \State $h \gets \lambda_t - \lambda_s, \quad \lambda_{s_1} \gets \lambda_s +
rh , \quad s_1 \gets t_{\lambda} (\lambda_{s_1})$
   \State $\epsilon_{\theta} \left( \widetilde{\mathbf{x}}_s, s \right) \gets \left[ D_{\theta}
\left( \frac{\widetilde{\mathbf{x}}_s}{\alpha_s} ; \sigma_s \right) -
\frac{\sigma_d^2}{s^2 + \sigma_d^2} \frac{\widetilde{\mathbf{x}}_s}{\alpha_s} \right]
\cdot \frac{\sqrt{s^2 + \sigma_d^2}}{s \sigma_d},$
    \State $\widetilde{\mathbf{x}}_{s_1} \gets \frac{s_1^2 + \sigma_d^2}{s^2 + \sigma_d^2}
\widetilde{\mathbf{x}}_s + 2 \frac{s_1  \sqrt{s_1^2 + \sigma_d^2}}{\sigma_d} (e^{r h} - 1) \epsilon_{\theta} \left( \widetilde{\mathbf{x}}_s, s \right) + \sqrt{\frac{(s_1^2 + \sigma_d^2) (s^2 - s_1^2)}{(s^2 + \sigma_d^2)}} z^1$
    \State $\widetilde{\mathbf{x}}_t \gets \frac{t^2 + \sigma_d^2}{s^2 + \sigma_d^2} \widetilde{\mathbf{x}}_s + 2 \frac{t  \sqrt{t^2 + \sigma_d^2}}{\sigma_d} (e^{ h} - 1) \epsilon_{\theta} \left({\widetilde{\mathbf{x}}_{s_1}}, s_1 \right) + \sqrt{\frac{(t^2 + \sigma_d^2) (s_1^2 - t^2)}{(s_1^2 + \sigma_d^2)}} (e^{r h} z^1 + z^2)$
    \EndFor
   \State Return $\tilde\x_{t_M}\gets \text{last-step}(\tilde\x_{t_{M-1}}, t_{M-1}, t_M)$
\end{algorithmic}
\end{algorithm}

\begin{algorithm}[ht]
   \caption{SEEDS-3 (EDM NP)}
   \label{alg:SERK-solver-3-edm}
\small 
\begin{algorithmic}
\State {\bfseries Input:} initial value $\x_T$, steps $\{t_i\}^M_{i=0}$, data prediction model $D_\theta$, $ r_1=1/3, r_2=2/3 $
   \State Initialize $\tilde\x_{t_0}\gets\x_T$
   \For{$i=1$ {\bfseries to} $M-1$}
   \State $(t,s) \gets (t_{i},t_{i-1}), \quad (z^1,z^2,z^3) \gets \mathcal{N}(0,\text{Id})^{\otimes3} $
   \State $\lambda (t) \gets - \log \left( \frac{t}{\sigma_d \sqrt{t^2 + \sigma^2_d}}
\right) \qquad t_{\lambda} (\lambda) \gets
\frac{\sigma_d}{\sqrt{\frac{1}{\sigma^2_d e^{- 2 \lambda}} - 1}}, \quad
\sigma_d \gets \sigma_{data}$
   \State $h \gets \lambda_t - \lambda_s, \quad \lambda_{s_1} \gets \lambda_s +
r_1h , \quad \lambda_{s_2} \gets \lambda_s +
r_2h , \quad s_1 \gets t_{\lambda} (\lambda_{s_1}), \quad s_2 \gets t_{\lambda} (\lambda_{s_2})$
   \State $\epsilon_{\theta} \left( \widetilde{\mathbf{x}}_s, s \right) \gets \left[ D_{\theta}
\left( \frac{\widetilde{\mathbf{x}}_s}{\alpha_s} ; \sigma_s \right) -
\frac{\sigma_d^2}{s^2 + \sigma_d^2} \frac{\widetilde{\mathbf{x}}_s}{\alpha_s} \right]
\cdot \frac{\sqrt{s^2 + \sigma_d^2}}{s \sigma_d},$
    \State $\widetilde{\mathbf{x}}_{s_1} \gets \frac{s_1^2 + \sigma_d^2}{s^2 + \sigma_d^2}
\widetilde{\mathbf{x}}_s + 2 \frac{s_1  \sqrt{s_1^2 + \sigma_d^2}}{\sigma_d} (e^{r_1 h} - 1) \epsilon_{\theta} \left( \widetilde{\mathbf{x}}_s, s \right) + \sqrt{\frac{(s_1^2 + \sigma_d^2) (s^2 - s_1^2)}{(s^2 + \sigma_d^2)}} z^1$
    \State $\Av \gets \sqrt{\frac{(s_2^2 + \sigma_d^2) (s_1^2 -  s_2^2)}{(s_1^2 +
\sigma_d^2)}} (e^{r_1 h} z^1 + z^2), \quad P_1 \gets \epsilon_{\theta} \left({\widetilde{\x}_{s_1}}, s_1 \right) -\epsilon_{\theta} \left( \x_s, s \right) $
    \State $\widetilde{\mathbf{x}}_{s_2} \gets \frac{s_2^2 + \sigma_d^2}{s^2 + \sigma_d^2}
\widetilde{\mathbf{x}}_s + 2 \frac{s_2  \sqrt{s_2^2 + \sigma_d^2}}{\sigma_d} (e^{r_2 h} - 1) \epsilon_{\theta} \left( \widetilde{\mathbf{x}}_s, s \right) + \frac{2 r_2}{r_1} \frac{s_2   \sqrt{s_2^2 + \sigma_d^2}}{\sigma_d} \left(
\frac{e^{r_2 h} - 1}{r_2 h} - 1 \right) P_1  + \Av$
    \State $\Bv \gets \sqrt{\frac{(t^2 + \sigma_d^2) (s_2^2 - t^2)}{(s_2^2 +
\sigma_d^2)}} \left( e^{r_2 h} z^1 + e^{r_1 h} z^2 + z^3 \right), \quad P_2 \gets \epsilon_{\theta} \left({\widetilde{\x}_{s_2}}, s_2 \right) -\epsilon_{\theta} \left( \x_s, s \right)$
    \State $\widetilde{\mathbf{x}}_t \gets \frac{t^2 + \sigma_d^2}{s^2 + \sigma_d^2} \widetilde{\mathbf{x}}_s + 2 \frac{t  \sqrt{t^2 + \sigma_d^2}}{\sigma_d} (e^{ h} - 1) \epsilon_{\theta} \left({\widetilde{\mathbf{x}}_{s}}, s \right) + \frac{2}{r_2} \frac{t  \sqrt{t^2 + \sigma_d^2}}{\sigma_d} \left( \frac{e^h -
1}{h} - 1 \right) P_2 + \Bv$
    \EndFor
   \State Return $\tilde\x_{t_M}\gets \text{last-step}(\tilde\x_{t_{M-1}}, t_{M-1}, t_M)$
\end{algorithmic}
\end{algorithm}

%% file: App-C.tex
\section{Convergence Proofs}\label{app:proofs}

In this Section, we give detailed proofs of Theorem \ref{theorem1}, Proposition \ref{noise-dec} and
Corollary \ref{cor:wocv} stated in the main part of this paper. Let us start
recalling its framework. We start by considering the NP NRSDE \eqref{nrsde-vp}
with VP coefficients:
\begin{eqnarray}
  \mathd \mathbf{x}_t & = & \left[ f (t) \mathbf{x}_t + \frac{g^2
  (t)}{\bar{\sigma}_t} \epsilonv_{\theta} (\mathbf{x}_t, t) \right] \mathd t +
  g (t) \mathd \tmmathbf{\bar{\omega}}_t  \label{eq:proof1}\\
  & = & \left[ f (t) \mathbf{x}_t + \frac{2 \alpha^2_t  \dot{\sigma}_t
  \sigma_t}{\bar{\sigma}_t} \epsilonv_{\theta} (\mathbf{x}_t, t) \right]
  \mathd t + g (t) \mathd \tmmathbf{\bar{\omega}}_t \nonumber\\
  & = & \left[ \frac{\mathd \log \alpha_t}{\mathd t} \mathbf{x}_t + 2
  \alpha_t  \dot{\sigma}_t \epsilonv_{\theta} (\mathbf{x}_t, t) \right] \mathd
  t + \alpha_t  \sqrt{\frac{\mathd [\sigma^2_t]}{\mathd t}} \mathd
  \tmmathbf{\bar{\omega}}_t. \nonumber
\end{eqnarray}
Denote $t_{\lambda}$ the inverse of $\lambda_t : = - \log (\sigma_t)$ (which
is a strictly decreasing function of $t$) and denote
$\widehat{\mathbf{x}}_{\lambda} \assign \mathbf{x} (t_{\lambda} (\lambda)),
\hat{\tmmathbf{\epsilon}}_{\theta} (\widehat{\mathbf{x}}_{\lambda}, \lambda)
\assign \tmmathbf{\epsilon}_{\theta} (\mathbf{x} (t_{\lambda} (\lambda)),
t_{\lambda} (\lambda))$. We consider a time discretization $\{t_i \}^{M +
1}_{i = 0}$ going backwards in time starting from $t_0 = T$ to $t_{M + 1} = 0$
and to ease the notation we will always denote $t < s$ for two consecutive
time-steps $t_i < t_{i - 1}$. The analytic solution at time $t$ of the RSDE
\eqref{eqq} with coefficients \eqref{vp-coefs} and initial value
$\mathbf{x}_s$ reads
\begin{eqnarray}
  \mathbf{x}_t & = & \frac{\alpha_t}{\alpha_s} \mathbf{x}_s - 2 \alpha_t 
  \int_{\lambda_s}^{\lambda_t} e^{- \lambda}
  \hat{\tmmathbf{\epsilon}}_{\theta} (\widehat{\mathbf{x}}_{\lambda}, \lambda)
  \mathd \lambda - \sqrt{2} \alpha_t  \int_{\lambda_s}^{\lambda_t} e^{-
  \lambda} \mathd \tmmathbf{\bar{\omega}} (\lambda).  \label{exact-1}
\end{eqnarray}
Given an initial condition $\widetilde{\mathbf{x}}_{t_0} = \mathbf{x}_T$, the
SEEDS-1 iterates read, for $h_i = \lambda_{t_i} - \lambda_{t_{i - 1}} =
\lambda_t - \lambda_s$,
\[ \widetilde{\mathbf{x}}_t = \cfrac{\alpha_t}{\alpha_s}
   \widetilde{\mathbf{x}}_s - 2 \bar{\sigma}_t  (e^{h_i} - 1)
   \tmmathbf{\epsilon}_{\theta} (\widetilde{\mathbf{x}}_s, s) - \bar{\sigma}_t
   \sqrt{e^{2 h_i} - 1} \epsilon \quad \epsilon \sim \mathcal{N}
   (\tmmathbf{0}, \mathbf{I}_d) . \]
{\assumption{\label{assumption-1}
\begin{enumerate}
  \item The function $\epsilonv_{\theta} (\x, t)$ is continuous (and hence
  bounded) on $[0, T]$, Lipschitz with respect to $\x$ and there is a constant
  $C$ such that, for $t, s \in [0, T]$ with $t < s$, we have
  \begin{eqnarray}
    | \epsilonv_{\theta} (\mathbf{x}, t) - \epsilonv_{\theta} (\mathbf{y}, t)
    |^2 & \leqslant & L_1 | \mathbf{x} -\mathbf{y} |^2 \\
    | \epsilonv_{\theta} (\mathbf{x}, t) |^2 \vee | g (t) |^2 & \leqslant &
    L_2 (1 + | \mathbf{x} |^2) \\
    | \epsilonv_{\theta} (\mathbf{x}, t) - \epsilonv_{\theta} (\mathbf{x}, s)
    |^2 & \leqslant & L_3 (1 + | \mathbf{x} |^2)  |t - s|^2 \label{lip:time2} 
  \end{eqnarray}
  \item $h = \max_{1 \le i \le M} |h_i | \sim O (1 / M)$, where $h_i =
  \lambda_{t_i} - \lambda_{t_{i - 1}}$.
\end{enumerate}}}

Let $\mathcal{C}^l_P (\mathbb{R}^d, \mathbb{R})$ denote the family of $l$
times continuously differentiable real-valued functions on $\mathbb{R}^d$
whose partial derivatives of order $\leqslant l$ have polynomial growth and
let $\mathcal{C}^{k, l}_P (I \times \mathbb{R}^d, \mathbb{R})$ be the space of
functions $g (\cdot, \cdot)$ such that, for all $(t, x) \in I \times
\mathbb{R}^d$, $g (\cdot, x) \in \mathcal{C}^k (I, \mathbb{R})$ and $g (t,
\cdot) \in \mathcal{C}^l_P (\mathbb{R}^d, \mathbb{R})$.

{\assumption{\label{assumption-2}In addition to Assumption \ref{assumption-1},
assume that all the components of $\tmmathbf{\epsilon}_{\theta}$ belong to $\mathcal{C}^{4,2}_P (
  \mathbb{R}^d \times [0,1], \mathbb{R})$.}}

Before going into the proofs, we give some context that lead us to necessitate
such assumptions.
\subsection{Preliminaries}

For an interval $I = [t_0, T]$, let $\mathbf{x} = (\mathbf{x} (t))_I$ the
solution of the following SDE
\begin{equation}
    \label{ref:sde}
    \mathd \mathbf{x} (t) = f (\mathbf{x} (t), t) \mathd t + g (t) \mathd
   \tmmathbf{\omega} (t),
\end{equation}
where $g (t) = \hat{g} (t) \cdot \tmop{Id}_d$ is considered here as a diagonal
matrix with identical diagonal entries $\hat{g} (t)$. Suppose that $f, g$ are
continuous, and satisfy a linear growth and Lipschitz condition so that the
conditions of the Existence and Uniqueness Theorem are fulfilled for the SDE
\eqref{ref:sde}.

Let $I_h = \{t_0, \ldots, t_M \}$ be a time discretization of $I$ with step
sizes $h_n = t_{n + 1} - t_n$ for $n = 0, \ldots, M - 1$ and let $h = \max_{0
\leqslant n < M} h_n$. A time discrete approximation scheme
$\widehat{\mathbf{x}} = (\widehat{\mathbf{x}}_n)_{I_h}$ will be defined as a
sequence $\widehat{\mathbf{x}}_0 = \mathbf{x} (t_0)$ and
\[ \widehat{\mathbf{x}}_{n + 1} = \Phi (\widehat{\mathbf{x}}_n, h_n, I_n), 
   \qquad n = 0, \ldots, M - 1, \]
where $I_0$ is independent of $\widehat{\mathbf{x}}_0$, with $I_n
=\tmmathbf{\omega} (t_{n + 1} - t_n)$ Wiener increments drawn from the normal
distributions with zero mean and variance $h_n$ and which are independent of
$\widehat{\mathbf{x}}_0, \ldots, \widehat{\mathbf{x}}_n$ and $I_0, \ldots,
I_{n - 1}$.

A scheme $\widehat{\mathbf{x}}$ converges in the strong (resp. weak) sense,
with global order $p > 0$, to the solution $\mathbf{x}$ of the SDE
\eqref{ref:sde} if there is a constant $C > 0$, independent of $h$ and $\delta
> 0$, such that, for each $h \in] 0, \delta]$, we have
\[ (\mathbb{E}[| \mathbf{x} (T) - \widehat{\mathbf{x}}_M |^2])^{1 / 2} \le
   Ch^p, \quad (\tmop{resp} .|\mathbb{E}[G (\mathbf{x} (T))] -\mathbb{E}[G
   (\widehat{\mathbf{x}}_M)] | \le Ch^p, \forall G \in \mathcal{C}^{2 (p +
   1)}_P (\mathbb{R}^d, \mathbb{R})) . \]
Notice that if $(\mathbb{E}[| \mathbf{x} (T) - \widehat{\mathbf{x}}_M |^2])^{1
/ 2} =\mathcal{O} (h^p)$ then for every function $f$ satisfying a Lipschitz
condition, we have $|\mathbb{E}[f (\mathbf{x} (T))] -\mathbb{E}[f
(\widehat{\mathbf{x}}_M)] | =\mathcal{O} (h^p)$. Nevertheless, this is not
enough to infer the optimal weak order of convergence of such method.

Strong convergence is concerned with the precision of the path, while the weak
convergence is with the precision of the moments. As, for DPMs,
the center of attention is the evolution of the probability densities rather
than that of the noising process of single data samples, weak convergence is
enough to guarantee the well-conditioning of our numerical schemes.
Moreover, when the diffusion coefficient vanishes, then both strong and weak
convergence (with the choice $G = \tmop{id}$) reduce to the usual
deterministic convergence criterion for ODEs.

Let us now state some useful results that will be used later on.
\begin{assumption}
  \label{assum1}All moments of the initial value $\widehat{\mathbf{x}}_0$
  exist, $f$ is continuous, satisfy a linear growth and globally Lipschitz
  condition.
\end{assumption}
In particular, since $I$ is a closed finite interval in $\mathbb{R}$, $f
(\cdot, x)$ and $g$ are bounded by some constant.
\begin{theorem}[{\cite{MT}}]
  \label{th:strong}In addition to \ref{assum1}, suppose that
  \begin{eqnarray*}
    |\mathbb{E}[\mathbf{x} (t_1) - \widehat{\mathbf{x}}_1] | & \leqslant &
    Ch^{p_1}\\
    (\mathbb{E}[| \mathbf{x} (t_1) - \widehat{\mathbf{x}}_1 |^2])^{1 / 2} &
    \leqslant & Ch^{p_2}
  \end{eqnarray*}
  with $p_2 \geqslant 1 / 2$ and $p_1 \geqslant p_2 + 1 / 2$. Then
  $\widehat{\mathbf{x}}$ is of strong global order $p = p_2 - 1 / 2$.
\end{theorem}

\begin{assumption}
  \label{assum2}All moments of the initial value $\widehat{\mathbf{x}}_0$
  exist, $f$ is continuous, satisfy a linear growth and Lipschitz condition
  with all their components belonging to $\mathcal{C}^{p + 1, 2 (p + 1)}_P (I
  \times \mathbb{R}^d, \mathbb{R})$ and $g \in \mathcal{C}^{p + 1} (I,
  \mathbb{R})$.
\end{assumption}

\begin{theorem}[{\cite{MT}}]
  \label{wlocal}In addition to \ref{assum2}, suppose that
  \begin{enumerate}
    \item for large enough $r \in \mathbb{N}$, the moments $\mathbb{E} [|
    \widehat{\mathbf{x}}_n |^{2 r}]$ exist and are uniformingly bounded with
    respect to $M$ and $n = 0, \ldots, M$
    
    \item for all $G \in \mathcal{C}^{2 (p + 1)}_P (\mathbb{R}^d,
    \mathbb{R})$, if $\widehat{\mathbf{x}}_n = \mathbf{x} (t_n)$, then for
    some $K \in \mathcal{C}^0_P (\mathbb{R}^d, \mathbb{R})$, we have
    \[ |\mathbb{E}[G (\mathbf{x} (t_{n + 1}))] -\mathbb{E}[G
       (\widehat{\mathbf{x}}_{n + 1})] | \le K (\widehat{\mathbf{x}}_n) h^{p +
       1} . \]
  \end{enumerate}
  Then $\widehat{\mathbf{x}}$ is of weak global order $p$.
\end{theorem}

\begin{lemma}
  Suppose that $\widehat{\mathbf{x}}_0$ has moments of all orders and that,
  for $h < 1$,
  \begin{eqnarray*}
    |\mathbb{E}[\Phi (\widehat{\mathbf{x}}_n, h_n, I_n) -
    \widehat{\mathbf{x}}_n] | & \leqslant & K (1 + | \widehat{\mathbf{x}}_n |)
    h\\
    | \Phi (\widehat{\mathbf{x}}_n, h_n, I_n) - \widehat{\mathbf{x}}_n | &
    \leqslant & X_n  (1 + | \widehat{\mathbf{x}}_n |) h^{1 / 2}
  \end{eqnarray*}
  where $X_n$ has moments of all orders. Then Condition 1 in Theorem
  \ref{wlocal} is fulfilled.
\end{lemma}

\subsection{Convergence of SEEDS-1}\label{proof-main}

In this section we will prove that that SEEDS-1 as described above is of
global strong order 1.0.
\subsubsection{Strong It{\^o}-Taylor approximation}
Then the truncated It{\^o}-Taylor expansion of the analytic solution $\x_t$ of
the VP NP NRSDE starting from $\x_s$ is given, for $\epsilon \sim \mathcal{N}
(0, 1)$, by
\begin{eqnarray*}
  \x_t & = & \frac{\alpha_t}{\alpha_s} \mathbf{x}_s - 2 \alpha_t 
  \int_{\lambda_s}^{\lambda_t} e^{- \lambda} \widehat{\epsilonv}_{\theta}
  (\widehat{\mathbf{x}}_{\lambda}, \lambda) \mathd \lambda - \sqrt{2} \alpha_t
  \int_{\lambda_s}^{\lambda_t} e^{- \lambda} \mathd
  \tmmathbf{\omega}_{\lambda}\\
  & = & \frac{\alpha_t}{\alpha_s} \mathbf{x}_s - 2 \bar{\sigma}_t  (e^h - 1)
  \epsilonv_{\theta} (\mathbf{x}_s, s) - \bar{\sigma}_t  \sqrt{e^{2 h} - 1}
  \epsilon +\mathcal{R}_1\\
  & = & \frac{\alpha_t}{\alpha_s} \mathbf{x}_s - 2 \bar{\sigma}_t  (e^h - 1)
  \epsilonv_{\theta} (\mathbf{x}_s, s) - \bar{\sigma}_t  \sqrt{e^{2 h} - 1}
  \epsilon +\mathcal{O} (h),
\end{eqnarray*}
where the symbol $\mathcal{O} (h^p)$ represent terms $\mathbf{u}$ such that
$\| \mathbf{u} \| \leqslant |K (\mathbf{x}_s) | h^p$, for $K \in
\mathcal{C}^0_P (\mathbb{R}^d, \mathbb{R})$ and small $h > 0$. The SEEDS-1
scheme corresponds to such truncated It{\^o}-Taylor expansion containing only
the time and Wiener integrals of multiplicity one. As such, since $G_t g (t) =
0$ and assuming Lipschitz and linear growth conditions on $\epsilonv_{\theta}$
as in Assumption \ref{assumption-1}, $\x_t$ can be interpreted as an order 1.0
strong It{\^o}-Taylor approximation {\cite[Theorem 10.6.3]{KP1}} of the
solution to \eqref{eq:proof1}.
\subsubsection{Continuous approximation of SEEDS-1}

Let $\tilde{\alpha}_t \assign \frac{1}{2} \beta_d t^2 + \beta_m t$, where
$\beta_d = \beta_{\max} - \beta_{\min} = 19.9$, $\beta_m = \beta_{\min} =
0.1$. We have
\[ \sigma_t = \sqrt{e^{\tilde{\alpha}_t} - 1}, \qquad \alpha_t = e^{-
   \frac{1}{2}  \tilde{\alpha}_t} = \frac{1}{\sqrt{\sigma^2_t + 1}} \]
so, in particular, as $T = 1$, we have $\tilde{\alpha}_{t_0} =
\tilde{\alpha}_1 = \frac{1}{2}  (\beta_{\max} - \beta_{\min}) + \beta_{\min} =
\frac{1}{2}  (\beta_{\max} + \beta_{\min}) \approx 10.05$ and
$\tilde{\alpha}_{t_{M + 1}} = \tilde{\alpha}_0 = 0$. We deduce $\alpha_{t_0} =
\alpha_1 \approx e^{- \frac{10.05}{2} } < 1$, $\alpha_{t_{M + 1}} = \alpha_0 =
1$. Next, $\sigma_{t_0} = \sigma_1 \approx \sqrt{e^{10.05} - 1} > 1$ and
$\sigma_{t_{M + 1}} = \sigma_0 = 0$. As such, $\lambda_{t_0} = - \log
(\sigma_1) \assign - L_0 < 0$ and $\lambda_t \underset{t \rightarrow t_{M +
1}}{\longrightarrow} + \infty$. As such, we will set $t_M = t_{M + 1} +
\varepsilon$ the end time so that $\lambda_{t_M} = L_0$ is finite. This
implies that, for $\lambda \in [- L_0, + \infty [$, \ $0 < e^{- L_0}
\leqslant e^{\lambda} < 1 < e^{\lambda_{t_{M}}}$ and, for $\hat{T} =
\lambda_{t_M} - \lambda_{t_0}$, $e^h = e^{\lambda_t - \lambda_s} \leqslant
e^{\hat{T}}$. Now set, for $\lambda \in [- L_0, + \infty [$,
\[ \hat{\alpha}_{\lambda} \assign \sqrt{\frac{1}{1 + e^{- 2 \lambda}}},
   \qquad \hat{\sigma}_{\lambda} \assign \sqrt{\frac{1}{1 + e^{+ 2 \lambda}}}.
\]
Then, as $\lambda$ increases, $\hat{\alpha}_{\lambda}$ increases starting from
$0 < \hat{\alpha}_{\lambda_{t_0}} < 1$ and $\hat{\alpha}_{\lambda}
\underset{\lambda \rightarrow + \infty}{\longrightarrow} 1$ while at the same
time $\hat{\sigma}_{\lambda}$ decreases starting from $0 <
\hat{\sigma}_{\lambda_{t_0}} < 1$ and $\hat{\sigma}_{\lambda}
\underset{\lambda \rightarrow + \infty}{\longrightarrow} 0$. As such, we can
rewrite the exact solution \eqref{exact-1} as
\begin{equation}
  \widehat{\x}_{\lambda_t} =
  \frac{\hat{\alpha}_{\lambda_t}}{\hat{\alpha}_{\lambda_s}}
  \widehat{\x}_{\lambda_s} - 2 \hat{\alpha}_{\lambda_t} 
  \int_{\lambda_s}^{\lambda_t} e^{- \lambda}
  \hat{\tmmathbf{\epsilon}}_{\theta} (\widehat{\mathbf{x}}_{\lambda}, \lambda)
  \mathd \lambda - \sqrt{2}  \hat{\alpha}_{\lambda_t}
  \int_{\lambda_s}^{\lambda_t} e^{- \lambda} \mathd \tmmathbf{\bar{\omega}}
  (\lambda) . \label{itera}
\end{equation}
Notice that $\frac{\hat{\alpha}_{\lambda_t}}{\hat{\alpha}_{\lambda_s}}$ is
bounded for all $t, s$ by
\[ \frac{\hat{\alpha}_{\lambda_t}}{\hat{\alpha}_{\lambda_s}} \leqslant
   \frac{1}{\hat{\alpha}_{\lambda_{t_0}}} \]
and $0 < \hat{\alpha}_{\lambda} < 1$ for all $\lambda \in [- L_0, + \infty
[$.

Recall that the SEEDS-1 is defined recursively as
\[ \y_{t_0} \leftarrow \x_T, \y_{t_i} \leftarrow
   \frac{\alpha_{t_i}}{\alpha_{t_{i - 1}}} \y_{t_{i - 1}} - 2
   \bar{\sigma}_{t_i}  (e^{h_i} - 1) \epsilonv_{\theta} (\y_{t_{i - 1}}, t_{i
   - 1}) - \bar{\sigma}_{t_i}  \sqrt{e^{2 h_i} - 1} \epsilon_i \]
and, for simplicity, we will denote $\y_{\lambda_{t_i}}$ for the iterates
\eqref{itera}.

Define a continuous approximation of SEEDS-1 as follows. For $\hat{h} = t -
s$, we write $\hat{s} = [s / \hat{h}]  \hat{h}$ where $[x]$ denotes the
largest integer lesser or equal to $x$ and $I_{[A]}$ is the indicator function
associated to a set $A$. We define the step function:
\[ \widehat{\y} (\lambda) \assign \sum_{k \geqslant 0} I_{[\lambda_{t_k},
   \lambda_{t_{k + 1}} [} \y_{\lambda_{t_k}} \]
and the continuous approximation
\[ \y (t) \assign \frac{\alpha_t}{\alpha_{t_0}} \y (t_0) - 2 \alpha_t 
   \int_{\lambda_{t_0}}^{\lambda_t} e^{- \hat{\lambda}} 
   \hat{\tmmathbf{\epsilon}}_{\theta} (\widehat{\y} (\lambda), \hat{\lambda})
   \mathd \lambda - \sqrt{2} \alpha_t  \int_{\lambda_{t_0}}^{\lambda_t} e^{-
   \lambda} \mathd \tmmathbf{\omega}_{\lambda}. \]
\begin{proposition}
  \label{prepa-th}There are two constants $C_1, C_2$ independent of $h$ such
  that, for all $t \in [0, T]$, we have
  \begin{eqnarray*}
    \mathbb{E} \left[ \sup_{t_0 \leqslant t \leqslant t_M} \left| \y (t)
    \right|^2 \right] & \leqslant & C_1\\
    \mathbb{E} \left[ \left| \y (t) - \widehat{\y} (t) \right|^2 \right] &
    \leqslant & C_2 h^2 .
  \end{eqnarray*}
\end{proposition}

\begin{proof}
  Recall the standard inequality $(a + b + c)^2 \leqslant 3 (a^2 + b^2 + c^2)$ for $a,b,c\in \mathbb{R}$. Then
  \begin{eqnarray*}
    \left| \y (t) \right|^2 & \leqslant & 3 \left[ \left|
    \frac{\alpha_t}{\alpha_{t_0}} \y_{t_0} \right|^2 + 4 \alpha_t^2 \left|
    \int_{\lambda_{t_0}}^{\lambda_t} e^{- \hat{\lambda}} 
    \hat{\epsilon}_{\theta} (\widehat{\y} (\lambda), \hat{\lambda}) \mathd
    \lambda \right|^2 + 2 \alpha_t^2 \left| \int_{\lambda_{t_0}}^{\lambda_t}
    e^{- \lambda} \mathd \tmmathbf{\omega}_{\lambda} \right|^2 \right].
  \end{eqnarray*}
  Using the fact
  that $\alpha_t \leqslant 1$, we have, by writing $\hat{T} = \lambda_{t_M} - \lambda_{t_0}$ and taking the expectation:
  \begin{eqnarray*}
    &  & \mathbb{E} \left[ \sup_{t_0 \leqslant t \leqslant t_M} \left| \y (t)
    \right|^2 \right]\\
    & \leqslant & 3 \left[ \left| \frac{\alpha_t}{\alpha_{t_0}} \right|^2
    \mathbb{E} \left[ \left| \y_{t_0} \right|^2 \right] + 4\mathbb{E} \left[
    \left| \int_{\lambda_{t_0}}^{\lambda_t} e^{- \hat{\lambda}} 
    \hat{\epsilon}_{\theta} (\widehat{\y} (\lambda), \hat{\lambda}) \mathd
    \lambda \right|^2 \right] + 2 \alpha_t^2 \mathbb{E} \left[ \left|
    \int_{\lambda_{t_0}}^{\lambda_t} e^{- \lambda} \mathd
    \tmmathbf{\omega}_{\lambda} \right|^2 \right] \right]\\
    & \leqslant & 3 \left[ \left| \frac{\alpha_t}{\alpha_{t_0}} \right|^2
    \mathbb{E} \left[ \left| \y_{t_0} \right|^2 \right] + 4 \hat{T} \mathbb{E}
    \left[ \int_{\lambda_{t_0}}^{\lambda_t} |e^{- \hat{\lambda}} |^2 \left|
    \hat{\epsilon}_{\theta} (\widehat{\y} (\lambda), \hat{\lambda}) \right|^2
    \mathd \lambda \right] + 2 \alpha_t^2 \mathbb{E} \left[
    \int_{\lambda_{t_0}}^{\lambda_t} |e^{- \lambda} |^2 \mathd \lambda \right]
    \right]\\
    & \leqslant & 3 \left[ \left| \frac{\alpha_t}{\alpha_{t_0}} \right|^2
    \mathbb{E} \left[ \left| \y_{t_0} \right|^2 \right] + 4 \hat{T} 
    \int_{\lambda_{t_0}}^{\lambda_t} |e^{- \hat{\lambda}} |^2 \mathbb{E}
    \left[ \left| \hat{\epsilon}_{\theta} (\widehat{\y} (\lambda),
    \hat{\lambda}) \right|^2 \right] \mathd \lambda + 2 \alpha_t^2 
    \int_{\lambda_{t_0}}^{\lambda_t} e^{- 2 \lambda} \mathd \lambda \right]\\
    & \leqslant & 3 K \left[ \mathbb{E} \left[ \left| \y_{t_0} \right|^2
    \right] + 4 \hat{T}  \int_{\lambda_{t_0}}^{\lambda_t} \mathbb{E} \left[
    \left| \hat{\epsilon}_{\theta} (\widehat{\y} (\lambda), \hat{\lambda})
    \right|^2 \right] \mathd \lambda + \left( e^{2 (\lambda_t -
    \lambda_{t_0})} - 1 \right) \right]\\
    & \leqslant & 3 K \left[ \mathbb{E} \left[ \left| \y_{t_0} \right|^2
    \right] + 4 \hat{T} L_2 \int_{\lambda_{t_0}}^{\lambda_t} \left( 1
    +\mathbb{E} \left[ \left| \widehat{\y} (\lambda) \right|^2 \right] \right)
    \mathd \lambda + (e^{2 \hat{T} } - 1) \right]\\
    & \leqslant & 3 K \left[ \mathbb{E} \left[ \left| \y_{t_0} \right|^2
    \right] + 4 \hat{T} L_2 \hat{T} + e^{2 \hat{T} } - 1 + 4 \hat{T} L_2
    \int_{\lambda_{t_0}}^{\lambda_t} \mathbb{E} \left[ \left| \widehat{\y}
    (\lambda) \right|^2 \right] \mathd \lambda \right]\\
    & \leqslant & 3 K \left( \mathbb{E} \left[ \left| \y_{t_0} \right|^2
    \right] + 4 \hat{T} L_2 \hat{T} + e^{2 \hat{T} } - 1 \right) + 3 K 4
    \hat{T} L_2 \int_{\lambda_{t_0}}^{\lambda_t} \mathbb{E} \left[ \left|
    \widehat{\y} (\lambda) \right|^2 \right] \mathd \lambda\\
    & \leqslant & 3 K \left( \mathbb{E} \left[ \left| \x (t_0) \right|^2
    \right] + 4 \hat{T}^2 L_2 + e^{2 \hat{T} } - 1 \right) + 12 K \hat{T} L_2
    \int_{\lambda_{t_0}}^{\lambda_t} \mathbb{E} \left[ \sup_{\lambda_0
    \leqslant r \leqslant \lambda} \left| \y (r) \right|^2 \right] \mathd
    \lambda,
  \end{eqnarray*}
  where we used the linearity of expectation, H{\"o}lder's inequality, Doob's
  martingale inequality, It{\^o} isometry, the linear growth condition of
  $\hat{\epsilon}_{\theta}$, and we set $K = \max \left( \left|
  \frac{\alpha_t}{\alpha_{t_0}} \right|^2, |e^{L_0} |^2, \bar{\sigma}_t^2, 1
  \right)$. As we know that $\mathbb{E} \left[ \left| \x (t_0) \right|^2
  \right] < \infty$, we apply Gr{\"o}nwall's inequality in the last line to
  obtain
  \[ \mathbb{E} \left[ \sup_{t_0 \leqslant t \leqslant t_M} \left| \y (t)
     \right|^2 \right] \leqslant C_1, \qquad C_1 \assign 3 K \left[ \mathbb{E} \left[
     \left| \x (t_0) \right|^2 \right] + 4 \hat{T}^2 L_2 + e^{2 \hat{T} } - 1
     \right] e^{12 K \hat{T} L_2} . \]
  Second, we have, for $s = t_i$, $u = t_{i + 1}$ and $t \in [t_{\lambda} (u),
  t_{\lambda} (s) [$,
  \[ \y (t) - \widehat{\y} (t) = \left( \frac{\alpha_t}{\alpha_s} - 1 \right)
     \y_s - 2 \alpha_t  \int_{\lambda_s}^{\lambda_t} e^{- \hat{\lambda}} 
     \hat{\epsilon}_{\theta} (\y_{\lambda_s}, \lambda_s) \mathd \lambda, \]
  so that, using H{\"o}lder's inequality, we get
  \begin{eqnarray*}
    \left| \y (t) - \widehat{\y} (t) \right|^2 & \leqslant & 3 \left[ \left|
    \frac{\alpha_t}{\alpha_s} - 1 \right|^2 \left| \y_s \right|^2 + 4 h
    \int_{\lambda_s}^{\lambda_t} |e^{- \hat{\lambda}} |^2 \left|
    \hat{\epsilon}_{\theta} (\y_{\lambda_s}, \lambda_s) \right|^2 \mathd
    \lambda \right] .
  \end{eqnarray*}
  Now using It{\^o} isometry we obtain:
  \begin{eqnarray*}
    \mathbb{E} \left[ \left| \y (t) - \widehat{\y} (t) \right|^2 \right] &
    \leqslant & 3 \left[ \left| \frac{\alpha_t}{\alpha_s} - 1 \right|^2
    \mathbb{E} \left[ \left| \y_s \right|^2 \right] + 4 h\mathbb{E} \left[
    \int_{\lambda_s}^{\lambda_t} |e^{- \hat{\lambda}} |^2 \left|
    \hat{\epsilon}_{\theta} (\y_{\lambda_s}, \lambda_s) \right|^2 \mathd
    \lambda \right] \right].
  \end{eqnarray*}
  Now, using the bound $\mathbb{E} \left[ \max \left| \y_t \right|^2 \right]
  \leqslant C_1$, the fact that $\epsilon_{\theta}$ is bounded and the same
  arguments as above, we get
  \begin{eqnarray*}
    \mathbb{E} \left[ \left| \y (t) - \widehat{\y} (t) \right|^2 \right] &
    \leqslant & 3 \left[ \left| \frac{\alpha_t}{\alpha_s} - 1 \right|^2 C_1 +
    4 Kh\mathbb{E} \left[ \int_{\lambda_s}^{\lambda_t} \left|
    \hat{\epsilon}_{\theta} (\y_{\lambda_s}, \lambda_s) \right|^2 \mathd
    \lambda \right] \right]\\
    & \leqslant & 3 \left[ \left| \frac{\alpha_t}{\alpha_s} - 1 \right|^2 C_1
    + 4 Kh L_2  \int_{\lambda_s}^{\lambda_t} (1 + C_1) \mathd \lambda
    \right]\\
    & \leqslant & 3 \left[ \left| \frac{\alpha_t}{\alpha_s} - 1 \right|^2 C_1
    + 4 Kh^2 L_2  (1 + C_1) \right].
  \end{eqnarray*}
  Finally, as we have
  \[ e^h = \frac{\sigma_s}{\sigma_t} = \frac{\alpha_t}{\alpha_s} \frac{\sqrt{1
     - \alpha_s^2}}{\sqrt{1 - \alpha_t^2}}, \]
  and $1 < \frac{\sqrt{1 - \alpha_s^2}}{\sqrt{1 - \alpha_t^2}} \underset{M
  \rightarrow \infty}{\longrightarrow} 1$, we obtain $\left|
  \frac{\alpha_t}{\alpha_s} - 1 \right|^2 \sim |e^h - 1|^2 \sim \mathcal{O}
  (h^2)$. Now, by denoting $|e^h - 1|^2 \leqslant K_2 h^2$, we conclude that
  \[ \mathbb{E} \left[ \left| \y (t) - \widehat{\y} (t) \right|^2 \right]
     \leqslant C_2 h^2, \]
  with $C_2 : = 3 K_2 C_1 + 4 K L_2 (1 + C_1)$. This finishes the proof.
\end{proof}

\subsubsection{Proof of Theorem \ref{theorem1}}

Let's now take a look at the approximation given $\y_{t_0} = \x_{t_0}$. We
have
\begin{eqnarray*}
  \y_t - \x_t & = & 2 \alpha_t  \int_{\lambda_{t_0}}^{\lambda_t} \left[ e^{-
  \lambda} \widehat{\epsilonv}_{\theta} (\widehat{\mathbf{x}}_{\lambda},
  \lambda) - e^{- \hat{\lambda}} \widehat{\epsilonv}_{\theta}
  (\widehat{\y}_{\lambda}, \hat{\lambda}) \right] \mathd \lambda.
\end{eqnarray*}
Using the inequality $\alpha_t \le 1$, the Lipschitz property of
$\widehat{\epsilonv}_{\theta}$, Assumption \eqref{assumption-1}, and
H{\"o}lder's inequality we deduce the bound:
\begin{eqnarray*}
  \left| \y_t - \x_t \right|^2 & \leqslant & 2 \left[ 4 \left|
  \int_{\lambda_{t_0}}^{\lambda_t} \left[ e^{- \lambda}
  \widehat{\epsilonv}_{\theta} (\widehat{\mathbf{x}}_{\lambda}, \lambda) -
  e^{- \hat{\lambda}} \widehat{\epsilonv}_{\theta} (\widehat{\y}_{\lambda},
  \hat{\lambda}) \right] \mathd \lambda \right|^2 \right]\\
  & \leqslant & 8 \hat{T} \int_{\lambda_{t_0}}^{\lambda_t} \left| e^{-
  \lambda} \widehat{\epsilonv}_{\theta} (\widehat{\mathbf{x}}_{\lambda},
  \lambda) - e^{- \hat{\lambda}} \widehat{\epsilonv}_{\theta}
  (\widehat{\y}_{\lambda}, \hat{\lambda}) \right|^2 \mathd \lambda.
\end{eqnarray*}
Taking the expectation yields
\begin{eqnarray*}
  &  & \mathbb{E} \left[ \sup_{t_0 \leqslant s \leqslant t} | \y_s - \x_s |^2
  \right]\\
  & \leqslant & 8 \hat{T} \mathbb{E} \left[ \int_{\lambda_{t_0}}^{\lambda_t}
  \left| e^{- \lambda} \widehat{\epsilonv}_{\theta}
  (\widehat{\mathbf{x}}_{\lambda}, \lambda) - e^{- \hat{\lambda}}
  \widehat{\epsilonv}_{\theta} (\widehat{\y}_{\lambda}, \hat{\lambda})
  \right|^2 \mathd \lambda \right].
\end{eqnarray*}
Now, for the first integral, by writing:
\begin{eqnarray*}
  &  & e^{- \lambda} \widehat{\epsilonv}_{\theta}
  (\widehat{\mathbf{x}}_{\lambda}, \lambda) - e^{- \hat{\lambda}}
  \widehat{\epsilonv}_{\theta} (\widehat{\y}_{\lambda}, \hat{\lambda})\\
  & = & e^{- \lambda} \widehat{\epsilonv}_{\theta}
  (\widehat{\mathbf{x}}_{\lambda}, \lambda) - e^{- \hat{\lambda}}
  \widehat{\epsilonv}_{\theta} (\widehat{\mathbf{x}}_{\lambda}, \lambda) +
  e^{- \hat{\lambda}} \widehat{\epsilonv}_{\theta}
  (\widehat{\mathbf{x}}_{\lambda}, \lambda) - e^{- \hat{\lambda}}
  \widehat{\epsilonv}_{\theta} (\widehat{\mathbf{x}}_{\lambda},
  \hat{\lambda})\\
  &  & + e^{- \hat{\lambda}} \widehat{\epsilonv}_{\theta}
  (\widehat{\mathbf{x}}_{\lambda}, \hat{\lambda}) - e^{- \hat{\lambda}}
  \widehat{\epsilonv}_{\theta} (\y_{\lambda}, \hat{\lambda}) + e^{-
  \hat{\lambda}} \widehat{\epsilonv}_{\theta} (\y_{\lambda}, \hat{\lambda}) -
  e^{- \hat{\lambda}} \widehat{\epsilonv}_{\theta} (\widehat{\y}_{\lambda},
  \hat{\lambda})\\
  & = & (e^{- \lambda} - e^{- \hat{\lambda}}) \widehat{\epsilonv}_{\theta}
  (\widehat{\mathbf{x}}_{\lambda}, \lambda) + e^{- \hat{\lambda}}  \left(
  \widehat{\epsilonv}_{\theta} (\widehat{\mathbf{x}}_{\lambda}, \lambda) -
  \widehat{\epsilonv}_{\theta} (\widehat{\mathbf{x}}_{\lambda}, \hat{\lambda})
  \right)\\
  &  & + e^{- \hat{\lambda}}  \left( \widehat{\epsilonv}_{\theta}
  (\widehat{\mathbf{x}}_{\lambda}, \hat{\lambda}) -
  \widehat{\epsilonv}_{\theta} (\y_{\lambda}, \hat{\lambda}) +
  \widehat{\epsilonv}_{\theta} (\y_{\lambda}, \hat{\lambda}) -
  \widehat{\epsilonv}_{\theta} (\widehat{\y}_{\lambda}, \hat{\lambda}) \right)
  ,
\end{eqnarray*}
we can state the following inequalities:
\begin{eqnarray*}
  &  & \mathbb{E} \left[ \int_{\lambda_{t_0}}^{\lambda_t} \left| e^{-
  \lambda} \widehat{\epsilonv}_{\theta} (\widehat{\mathbf{x}}_{\lambda},
  \lambda) - e^{- \hat{\lambda}} \widehat{\epsilonv}_{\theta}
  (\widehat{\y}_{\lambda}, \hat{\lambda}) \right|^2 \mathd \lambda \right]\\
  & \leqslant & 3 \int_{\lambda_{t_0}}^{\lambda_t} |e^{- \lambda} - e^{-
  \hat{\lambda}} |^2 \mathbb{E} \left[ \left| \widehat{\epsilonv}_{\theta}
  (\widehat{\mathbf{x}}_{\lambda}, \lambda) \right|^2 \right] \mathd \lambda +
  3\mathbb{E} \left[ \int_{\lambda_{t_0}}^{\lambda_t} |e^{- \hat{\lambda}} |^2
  \left| \widehat{\epsilonv}_{\theta} (\widehat{\mathbf{x}}_{\lambda},
  \lambda) - \widehat{\epsilonv}_{\theta} (\widehat{\mathbf{x}}_{\lambda},
  \hat{\lambda}) \right|^2 \mathd \lambda \right]\\
  &  & + 3\mathbb{E} \left[ \int_{\lambda_{t_0}}^{\lambda_t} |e^{-
  \hat{\lambda}} |^2  \left| \widehat{\epsilonv}_{\theta}
  (\widehat{\mathbf{x}}_{\lambda}, \hat{\lambda}) -
  \widehat{\epsilonv}_{\theta} (\y_{\lambda}, \hat{\lambda}) +
  \widehat{\epsilonv}_{\theta} (\y_{\lambda}, \hat{\lambda}) -
  \widehat{\epsilonv}_{\theta} (\widehat{\y}_{\lambda}, \hat{\lambda})
  \right|^2 \mathd \lambda \right]\\
  & \leqslant & 3 \int_{\lambda_{t_0}}^{\lambda_t} |e^{- \lambda} - e^{-
  \hat{\lambda}} |^2 \mathbb{E} \left[ \left| \widehat{\epsilonv}_{\theta}
  (\widehat{\mathbf{x}}_{\lambda}, \lambda) \right|^2 \right] \mathd \lambda +
  3 K L_3 \mathbb{E} \left[ \int_{\lambda_{t_0}}^{\lambda_t} (1 + |
  \widehat{\mathbf{x}}_{\lambda} |^2) | (\lambda - \hat{\lambda}) | ^2 \mathd
  \lambda \right]\\
  &  & + 3\mathbb{E} \left[ \int_{\lambda_{t_0}}^{\lambda_t} |e^{-
  \hat{\lambda}} |^2  \left| \widehat{\epsilonv}_{\theta}
  (\widehat{\mathbf{x}}_{\lambda}, \hat{\lambda}) -
  \widehat{\epsilonv}_{\theta} (\y_{\lambda}, \hat{\lambda}) +
  \widehat{\epsilonv}_{\theta} (\y_{\lambda}, \hat{\lambda}) -
  \widehat{\epsilonv}_{\theta} (\widehat{\y}_{\lambda}, \hat{\lambda})
  \right|^2 \mathd \lambda \right]\\
  & \leqslant & 3 L_2  \int_{\lambda_{t_0}}^{\lambda_t} |e^{- \hat{\lambda}}
  |^2  |e^{\lambda - \hat{\lambda}} - 1|^2  (1 +\mathbb{E}[|
  \widehat{\mathbf{x}}_{\lambda} |^2]) \mathd \lambda + 3 K L_3 h^2 \mathbb{E}
  \int_{\lambda_{t_0}}^{\lambda_t} (1 + | \widehat{\mathbf{x}}_{\lambda} |^2)
  \mathd \lambda\\
  &  & + 3 K L_1  \int_{\lambda_{t_0}}^{\lambda_t} \mathbb{E} \left[ \left|
  \widehat{\mathbf{x}}_{\lambda} - \y_{\lambda} + \y_{\lambda} -
  \widehat{\y}_{\lambda} \right|^2 \right] \mathd \lambda .
\end{eqnarray*}
Thus, we obtain
\begin{eqnarray*}
  &  & \mathbb{E} \left[ \int_{\lambda_{t_0}}^{\lambda_t} \left| e^{-
  \lambda} \widehat{\epsilonv}_{\theta} (\widehat{\mathbf{x}}_{\lambda},
  \lambda) - e^{- \hat{\lambda}} \widehat{\epsilonv}_{\theta}
  (\widehat{\y}_{\lambda}, \hat{\lambda}) \right|^2 \mathd \lambda \right]\\
  & \leqslant & 3 L_2 K \hat{T}  |e^h - 1|^2  (1 + C_1) + 3 L_3 h^2 K (1 +
  C_1)  \hat{T}\\
  &  & + 3 K L_1  \int_{\lambda_{t_0}}^{\lambda_t} \mathbb{E} \left[ \left|
  \widehat{\mathbf{x}}_{\lambda} - \y_{\lambda} + \y_{\lambda} -
  \widehat{\y}_{\lambda} \right|^2 \right] \mathd \lambda\\
  & \leqslant & 3 L_2 K \hat{T}  |e^h - 1|^2  (1 + C_1) + 3 L_3 h^2 K (1 +
  C_1)  \hat{T}\\
  &  & + 6 K L_1  \int_{\lambda_{t_0}}^{\lambda_t} \mathbb{E} \left[ |
  \widehat{\mathbf{x}}_{\lambda} - \textbf{$\y$}_{\lambda} |^2 + \left|
  \textbf{$\y$}_{\lambda} - \widehat{\y}_{\lambda} \right|^2 \right] \mathd
  \lambda
\end{eqnarray*}
\begin{eqnarray*}
  & \leqslant & 3 L_2 K \hat{T}  |e^h - 1|^2  (1 + C_1) + 3 L_3 h^2 K (1 +
  C_1)  \hat{T}\\
  &  & + 6 K L_1  \hat{T} C_2 h^2 + 6 K L_1  \int_{\lambda_{t_0}}^{\lambda_t}
  \mathbb{E} \left[ | \widehat{\mathbf{x}}_{\lambda} - \textbf{$\y$}_{\lambda}
  |^2 \right] \mathd \lambda\\
  & \leqslant & 3 L_2 K \hat{T}  |e^h - 1|^2  (1 + C_1) + 3 L_3 h^2 K (1 +
  C_1)  \hat{T}\\
  &  & + 6 K L_1  \hat{T} C_2 h^2 + 6 K L_1  \int_{\lambda_{t_0}}^{\lambda_t}
  \mathbb{E} \left[ \sup_{\lambda_0 \leqslant r \leqslant \lambda} |
  \widehat{\mathbf{x}} (r) - \textbf{$\y$} (r) |^2 \right] \mathd \lambda\\
  & \leqslant & 3 K \hat{T} (1 + C_1) [L_2 |e^h - 1|^2 + L_3 h^2] + 6 K L_1 
  \hat{T} C_2 h^2\\
  &  & + 6 K L_1  \int_{\lambda_{t_0}}^{\lambda_t} \mathbb{E} \left[
  \sup_{\lambda_0 \leqslant r \leqslant \lambda} | \widehat{\mathbf{x}} (r) -
  \textbf{$\y$} (r) |^2 \right] \mathd \lambda .
\end{eqnarray*}
Putting everything together yields
\begin{eqnarray*}
  &  & \mathbb{E} \left[ \sup_{t_0 \leqslant s \leqslant t} | \y_s - \x_s |^2
  \right]\\
  & \leqslant & 8 \hat{T} [3 K \hat{T} (1 + C_1) [L_2 |e^h - 1|^2 + L_3 h] +
  6 K L_1  \hat{T} C_2 h^2]\\
  &  & + 8 \hat{T} 6 K L_1  \int_{\lambda_{t_0}}^{\lambda_t} \mathbb{E}
  \left[ \sup_{\lambda_0 \leqslant r \leqslant \lambda} | \widehat{\mathbf{x}}
  (r) - \textbf{$\textbf{$\y$}$} (r) |^2 \right] \mathd \lambda\\
  & \leqslant & 24 K \hat{T}^2 (1 + C_1) [L_2 |e^h - 1|^2 + L_3 h^2] + 48 K
  L_1  \hat{T}^2 C_2 h^2\\
  &  & + 48 \hat{T} K L_1  \int_{\lambda_{t_0}}^{\lambda_t} \mathbb{E} \left[
  \sup_{\lambda_0 \leqslant r \leqslant \lambda} | \widehat{\mathbf{x}} (r) -
  \textbf{$\textbf{$\y$}$} (r) |^2 \right] \mathd \lambda.
\end{eqnarray*}
Now, by denoting $|e^h - 1|^2 \leqslant K_2 h^2$, we apply the continuous
version of the Gr{\"o}nwall Lemma to obtain:
\[ \mathbb{E} \left[ \sup_{t_0 \leqslant s \leqslant t} | \y_t - \x_t |^2
   \right] \leqslant C_0 h^2, \]
where
\[ C_0 = [24 K \hat{T}^2 (1 + C_1) [L_2 K_2 + L_3 ] + 48 K \hat{T}^2 L_1 C_2
   ] e^{48 \hat{T} K L_1 }. \]
Finally, using Lyapunov's inequality we obtain, for $C = \sqrt{C_0}$
\[ \mathbb{E} [| \y_T - \x_T |] \leqslant \left( \mathbb{E}[| \y_T - \x_T |^2]
   \right)^{1 / 2} \leqslant Ch. \]
In other words, for $C$ as stated above, we have the following inequality
\[ \sqrt{\mathbb{E} \left[ \underset{t_0 \leqslant t \leqslant t_M}{\sup} |
   \widetilde{\mathbf{x}}_t -\mathbf{x}_t |^2 \right]} \leqslant C h, \qquad
   \text{as } h \longrightarrow 0. \]
\begin{remark}
  From the above, it is easy to induce that the following order for the one
  step error
  \[ \mathbb{E} [| \mathbf{x} (t_1) - \textbf{$\y$}_1 |^2] =\mathcal{O} (h^3)
     . \]
  Now, since $G_t g (t) = 0$ in this case (additive noise), it is easy to see
  from the truncated It{\^o}-Taylor expansion of $\mathbf{x}$ that
  $|\mathbb{E}[\mathbf{x} (t_1) - \textbf{$\y$}_1] | =\mathcal{O} (h^2)$. As
  such, we apply Theorem \ref{th:strong} to conclude that SEEDS-1 has strong
  convergence of global order 1.0.
\end{remark}
\subsubsection{Discrete-time approximation}\label{w-disc}

By Theorem \ref{theorem1}, we know that SEEDS-1, being of strong order 1.0, it
is immediately also of weak order 1. Nonetheless, let's give a discrete
approach of this statement that we will use for the proofs of convergence for
the remaining solvers as stated in Corollary \ref{cor:wocv}. Define the
discrete time process:
\[ \y_{t_0} \leftarrow \x_T, \y_{t_i} \leftarrow
   \frac{\alpha_{t_i}}{\alpha_{t_{i - 1}}} \y_{t_{i - 1}} - 2 \sigma_{t_i} 
   (e^{h_i} - 1) \epsilonv (\y_{t_{i - 1}}, t_{i - 1}) - \sqrt{2} \alpha_{t_i}
   \int_{\lambda_{t_{i - 1}}}^{\lambda_{t_i}} e^{- s} \mathd
   \tmmathbf{\bar{\omega}} (s) . \]
We will prove that $\mathbb{E} [| \y_{t_M} - \x_{t_M} |]$ is of order $h$ as
$h \longrightarrow 0$. Note that $(\y_{t_i})_i$ has the same distribution as
$(\widetilde{\x}_{t_i})_i$ described in Algorithm \ref{alg:SERK-solver-1}
since the stochastic integrals $\left( \int_{\lambda_{t_{i -
1}}}^{\lambda_{t_i}} e^{- s} \mathd \tmmathbf{\bar{\omega}} (s) \right)_i$ are
independent and each $\int_{\lambda_{t_{i - 1}}}^{\lambda_{t_i}} e^{- s}
\mathd \tmmathbf{\bar{\omega}} (s)$ is distributed as $\frac{1}{\sqrt{2}} 
\frac{\sigma_{t_i}}{\alpha_{t_i}}  \sqrt{e^{2 h_i} - 1} \epsilon$ with $h_i =
\lambda_{t_i} - \lambda_{t_{i - 1}}, \quad \epsilon \sim \mathcal{N}
(\tmmathbf{0}, \mathbf{I}_d)$. We have:
\[ \y_{t_i} - \x_{t_i} = \frac{\alpha_{t_i}}{\alpha_{t_{i - 1}}}  (\y_{t_{i -
   1}} - \x_{t_{i - 1}}) + 2 \alpha_{t_i}  \int_{\lambda_{t_{i -
   1}}}^{\lambda_{t_i}} e^{- \lambda}  (\widehat{\epsilonv}
   (\widehat{\mathbf{x}}_{\lambda}, \lambda) - \epsilonv (\y_{t_{i - 1}}, t_{i
   - 1})) \mathd \lambda . \]
For simplicity, in what follows $C$ will denote a constant not dependent on
the subdivision of $[0, T]$ that may change from one line to the next by
systematically denoting the maximum value of the different constants appearing
in the line before. Using the inequality $\alpha_t \le 1$, the Lipschitz
property of $\epsilon$, we deduce the bound:
\[ | \y_{t_i} - \x_{t_i} | \leq \frac{\alpha_{t_i}}{\alpha_{t_{i - 1}}} |
   \y_{t_{i - 1}} - \x_{t_{i - 1}} | + C \left[ \int_{\lambda_{t_{i -
   1}}}^{\lambda_{t_i}} e^{- \lambda} | \widehat{\mathbf{x}}_{\lambda} -
   \y_{t_{i - 1}} | \mathd \lambda + \int_{\lambda_{t_{i -
   1}}}^{\lambda_{t_i}} e^{- \lambda} |t_{\lambda} (\lambda) - t_{i - 1} |
   \mathd \lambda \right] . \]
Notice that
\begin{eqnarray*}
  \int_{\lambda_{t_{i - 1}}}^{\lambda_{t_i}} e^{- \lambda}  |
  \widehat{\mathbf{x}}_{\lambda} - \y_{t_{i - 1}} | \mathd \lambda & \leq &
  \int_{\lambda_{t_{i - 1}}}^{\lambda_{t_i}} e^{- \lambda}  |
  \widehat{\mathbf{x}}_{\lambda} - \x_{t_{i - 1}} | \mathd \lambda +
  \int_{\lambda_{t_{i - 1}}}^{\lambda_{t_i}} e^{- \lambda}  | \x_{t_{i - 1}} -
  \y_{t_{i - 1}} | \mathd \lambda\\
  & \leq & \int_{\lambda_{t_{i - 1}}}^{\lambda_{t_i}} e^{- \lambda}  |
  \widehat{\mathbf{x}}_{\lambda} - \x_{t_{i - 1}} | \mathd \lambda + Ch |
  \x_{t_{i - 1}} - \y_{t_{i - 1}} |,
\end{eqnarray*}
and recall that $\widehat{\mathbf{x}}_u = \x_{t_{\lambda} (u)}$. Using the
fact that $t_{\lambda}$ is increasing and Lemma \ref{bound}, we have:
\[ \mathbb{E} [\int_{\lambda_{t_{i - 1}}}^{\lambda_{t_i}} e^{- \lambda} |
   \widehat{\mathbf{x}}_{\lambda} - \x_{t_{i - 1}} | \mathd \lambda] \leq C
   \int_{\lambda_{t_{i - 1}}}^{\lambda_{t_i}} e^{- u}  \sqrt{|t_{\lambda} -
   t_{i - 1} |} \mathd \lambda \leq Ch^{3 / 2} . \]
Introduce $U_i =\mathbb{E} [| \y_{t_i} - \x_{t_i} |]$. Since
$\int_{\lambda_{t_{i - 1}}}^{\lambda_{t_i}} e^{- \lambda}  |t_{\lambda}
(\lambda) - t_{i - 1} | \mathd \lambda \leq Ch^2$:
\[ U_i \leq \left( \frac{\alpha_{t_i}}{\alpha_{t_{i - 1}}} + Ch \right) U_{i -
   1} + Ch^2 . \]
Let $a_i = \frac{\alpha_{t_i}}{\alpha_{t_{i - 1}}} + Ch$ and $b_i = Ch^2$. By
applying Lemma \ref{disc}, we have: $U_M \leq A_M U_0 + \sum_{k = 1}^M A_{k,
n} b_k$ with $A_M = \prod_{k = 1}^M a_k$ and $A_{k, M} = A_M / A_k = \prod_{i
= k + 1}^M a_i$. Note that $U_0 = 0$ since $\y_{t_0} = \x_T$ and so:
\[ U_M \leq Ch^2  \sum_{k = 0}^{M - 1} (\sup_{1 \leq i \le M} a_i)^k . \]
Using our hypothesis, we can bound:
\[ \sum_{k = 0}^{M - 1} (\sup_{1 \leq i \le M} a_i)^k \le \sum_{k = 0}^{M - 1}
   (\exp (Ch) + Ch)^k = \frac{(\exp (Ch) + Ch)^{1 / h} - 1}{(\exp (Ch) + Ch) -
   1} . \]
The quantity on the right is of order $C / h$. Indeed, as $h$ goes to $0$,
$\exp (Ch) + Ch - 1$ is equivalent to $2 Ch$ and $(\exp (Ch) + Ch)^{1 / h}$
converges to a constant. This gives the bound $U_M \leq Ch$, when $h \to 0$ and
by using Proposition \ref{prepa-th} and Theorem \ref{wlocal} we conclude that
SEEDS-1 is convergent of weak order 1.

\subsubsection{Useful Lemmas}\label{cv_lemmas}

\begin{lemma}
  \label{continuous}(Continuous Gr{\"o}nwall Lemma) Let $I = [a, b]$ denote a
  compact interval of the real line with $a < b$. Let $\alpha, \beta, u$ be
  continuous real-valued functions defined on $I$. Assume $\beta$ is
  non-negative, $\alpha$ is non-decreasing and $u$ satisfies the integral
  inequality
  \[ u (t) \leq \alpha (t) + \int_a^t \beta (s) u (s) \hspace{0.17em}
     \mathrm{d} s, \qquad \forall t \in I. \]
  Then
  \[ u (t) \leq \alpha (t) \exp \left( \int_a^t \beta (s) \hspace{0.17em}
     \mathrm{d} s \right), \qquad \forall t \in I. \]
\end{lemma}

\begin{lemma}
  \label{bound}Assume the following forward SDE is satisfied $dX_t = F (t,
  X_t) dt + G (t) dW_t, t \in [0, T]$, where $T > 0$, $F (t, x)$ is Lipschitz
  with respect to $(t, x)$, $G$ is continuous and $X_0$ is an integrable
  random variable. Then, there exists $C > 0$ such that for all $s < t \in [0,
  T]$ with $t - s \le 1$,
  \[ \mathbb{E} [|X_t - X_s |] \le C \sqrt{t - s} . \]
\end{lemma}

\begin{proof}
  We take $s = 0$ and apply the triangular inequality:
  \[ |X_t - X_0 | \leq \int_0^t |F (u, X_u) - F (0, X_0)) | \mathd u + |
     \int_0^t G (u) dW_u | + t|F (0, X_0) | . \]
  Setting $u (t) =\mathbb{E} [|X_t - X_0 |]$ and taking the expectation, we
  deduce:
  \[ u (t) \leq K \frac{t^2}{2} + t\mathbb{E} [|F (0, X_0) |] +\mathbb{E}
     [(\int_0^t G (u) dW_u)^2]^{\frac{1}{2}} + K \int_0^t u (s) ds, \]
  where $K$ is a positive constant. Note that $\mathbb{E} [(\int_0^t G (u)
  dW_u)^2]^{\frac{1}{2}} = [\int_0^t G^2 (s) ds]^{\frac{1}{2}}$ by the It{\^o}
  isometry property which is less than $C \sqrt{t}$ since $G$ is continuous.
  Thus, we have proved that:
  \[ u (t) \leq \alpha (t) + K \int_0^t u (s) ds, \]
  where $\alpha (t) = K \frac{t^2}{2} + t\mathbb{E} [|F (0, X_0) |] + C
  \sqrt{t}$ is non-decreasing. By Lemma \ref{continuous}: $u (t) \le \alpha
  (t) \exp (Kt) \leq \alpha (t) \exp (KT)$. Since $\alpha (t) \leq C \sqrt{t}$
  for $t \in [0, 1]$, the lemma holds.
\end{proof}

\begin{lemma}
  \label{disc}(Discrete Gr{\"o}nwall Lemma) Consider a real number sequence
  $(u_n)_n$ such that
  \[ u_{n + 1} \leq a_{n + 1} u_n + b_{n + 1}, n \ge 0 \]
  where $(a_n)$ and $(b_n)_n$ are two given sequences such that $b_n$ is
  positive. Then
  \[ u_n \leq A_n u_0 + \sum_{k = 1}^n A_{k, n} b_k, \]
  where $A_n = \prod_{k = 1}^n a_k$ and $A_{k, n} = A_n / A_k = \prod_{i = k +
  1}^n a_i$.
\end{lemma}

\subsection{Proof of Proposition \ref{noise-dec}}
Let us prove this statement for SEEDS-2. On the one hand, we have
\[ \mathbf{u}_1 = \frac{\alpha_{s_1}}{\alpha_s} \tilde{\mathbf{x}}_s - 2
   \bar\sigma_{s_1} \left( e^{\frac{h}{2}} - 1 \right)
   F_{\theta} (\tilde{\mathbf{x}}_s, s) - B^1, \qquad B^1
   \assign \sqrt{2} \alpha_{s_1} \int_{\lambda_s}^{\lambda_{s_1}} e^{-
   \lambda} \mathd \omega_{\lambda} . \]
Now $B^1$ depends on the Brownian movement $\left( \omega_{\lambda_{s_1}} -
\omega_{\lambda_s} \right)_{\lambda_{s_1} \geqslant \lambda_s}$. By the Markov
property, this is independent of $(\omega_{\lambda_u})_{\lambda_u \leqslant
\lambda_s}$. Since $\tilde{\mathbf{x}}_s$ is a function of
$(\omega_{\lambda_u})_{\lambda_u \leqslant \lambda_s}$, we deduce that $B^1$
has to be independent of all $(\tilde{\mathbf{x}}_u)_{u \geqslant s}$. By
It{\^o} Isometry, we obtain
\[ B^1 = \bar{\sigma}_{s_1} \sqrt{e^h - 1} z^1, \quad z^1 \sim \mathcal{N}
   (\tmmathbf{0}, \mathbf{I}_d) . \]
On the other hand, the update $\tilde{\mathbf{x}}_t$ is
\[ \tilde{\mathbf{x}}_t = \frac{\alpha_t}{\alpha_s} \tilde{\mathbf{x}}_s - 2
   \bar\sigma_t (e^h - 1) F_{\theta} (\mathbf{u}_1,
   s_1) - B^0, \qquad B^0 = \sqrt{2} \alpha_t \int_{\lambda_s}^{\lambda_t}
   e^{- \lambda} \mathd \omega_{\lambda} . \]
Hence, the Wiener process $W = \{ \omega_{\lambda} : \lambda \in [\lambda_s,
\lambda_t] \}$ is predetermined on the interval $[\lambda_s, \lambda_{s_1}]$.
Then, by using independent increments property of Wiener process $W$, we can
deduce, for $0 \leq \lambda_s < \lambda_{s_1} < \lambda_t$, that
\[ \omega_{\lambda_{s_1}} - \omega_{\lambda_s} \quad \text{and} \quad
   \omega_{\lambda_t} - \omega_{\lambda_{s_1}} \text{ are independent} . \]
Then, by the above Brownian independence property, the random variables $B^0$
and $B^1$ are
\begin{itemize}
  \item Independent on non-overlapping time intervals
  \item Dependent on overlapping intervals.
\end{itemize}
By the Chasles rule, we then decompose
\[ B^0 = \sqrt{2} \alpha_t \int_{\lambda_s}^{\lambda_{s_1}} e^{- \lambda}
   \mathd \omega_{\lambda} + \sqrt{2} \alpha_t
   \int_{\lambda_{s_1}}^{\lambda_t} e^{- \lambda} \mathd \omega_{\lambda} \]
and, as
\[ \int_{\lambda_s}^{\lambda_{s_1}} e^{- \lambda} \mathd \omega_{\lambda} =
   \frac{B^1}{\sqrt{2} \alpha_{s_1}}, \]
we obtain
\[ B^0 = \sqrt{2} \alpha_t \left[ \frac{B^1}{\sqrt{2} \alpha_{s_1}} +
   \int_{\lambda_{s_1}}^{\lambda_t} e^{- \lambda} \mathd \omega_{\lambda}
   \right] . \]
Then we have
\begin{eqnarray*}
  B & = & \sqrt{2} \alpha_t \int_{\lambda_s}^{\lambda_{s_1}} e^{- \lambda}
  \mathd \omega_{\lambda} + \sqrt{2} \alpha_t \int_{\lambda_s}^{\lambda_t}
  e^{- \lambda} \mathd \omega_{\lambda}\\
  & = & \sqrt{2} \alpha_t \frac{1}{\sqrt{2}} \sigma_{s_1} \sqrt{e^{2
  (\lambda_{s_1} - \lambda_s)} - 1} z^1 + \sqrt{2} \alpha_t \frac{1}{\sqrt{2}}
  \sigma_t \sqrt{e^{2 (\lambda_t - \lambda_s)} - 1} z^2\\
  & = & \alpha_t \sigma_{s_1} \sqrt{e^h - 1} z^1 + \alpha_t \sigma_t
  \sqrt{e^{2 h} - 1} z^2\\
  & = & \bar{\sigma}_t \left( \frac{\sigma_{s_1}}{\sigma_t} \sqrt{e^h - 1}
  z^1 + \sqrt{e^{2 h} - 1} z^2 \right)\\
  & = & \bar{\sigma}_t \left( e^{\lambda_t - \lambda_{s_1}} \sqrt{e^h - 1}
  z^1 + \sqrt{e^{2 h} - 1} z^2 \right)\\
  & = & \bar{\sigma}_t \left( e^{\frac{h}{2}} \sqrt{e^h - 1} z^1 + \sqrt{e^{2
  h} - 1} z^2 \right)\\
  & = & \bar{\sigma}_t \left( \sqrt{e^{2 h} - e^h} z^1 + \sqrt{e^{2 h} - 1}
  z^2 \right),
\end{eqnarray*}
which completes the proof.

The proof for SEEDS-3 is straightforward from the proof for SEEDS-2.

\subsection{Proof of Corollary \ref{cor:wocv}}

\subsubsection{Convergence of SEEDS-2:}

Let $\{ \y_{t_i} \}_i$ be the discrete stochastic process defined as follows:
\[ \y_{t_0} \leftarrow \x_T, \y_{t_i} \leftarrow
   \frac{\alpha_{t_i}}{\alpha_{t_{i - 1}}} \y_{t_{i - 1}} - 2 \alpha_{t_i} 
   \int_{\lambda_{t_{i - 1}}}^{\lambda_{t_i}} e^{- u} \epsilonv (\uv_i, s_i)
   \mathd u - \sqrt{2} \alpha_{t_i}  \int_{\lambda_{t_{i -
   1}}}^{\lambda_{t_i}} e^{- s} \mathd \tmmathbf{\bar{\omega}} (s), \]
with $s_i \leftarrow t_{\lambda}  \left( \lambda_{t_{i - 1}} + \frac{h_i}{2}
\right)$ and
\[ \uv_i \leftarrow \frac{\alpha_{s_i}}{\alpha_{t_{i - 1}}} \y_{t_{i - 1}} - 2
   \sigma_{s_i}  \left( e^{\frac{h_i}{2}} - 1 \right) \epsilonv (\y_{t_{i -
   1}}, t_{i - 1}) - \sqrt{2} \alpha_{s_i}  \int_{\lambda_{t_{i -
   1}}}^{\lambda_{s_i}} e^{- s} \mathd \tmmathbf{\bar{\omega}} (s), \]
Then $\{ \y_{t_i} \}_i$ has the same distribution as $\{ \widetilde{\x}_{t_i}
\}_i$ in Algorithm \ref{alg:SERK-solver-2}. We can compute the difference:
\[ \y_{t_i} - \x_{t_i} = \frac{\alpha_{t_i}}{\alpha_{t_{i - 1}}}  (\y_{t_{i -
   1}} - \x_{t_{i - 1}}) + \Gamma ; \Gamma = \Gamma_1 + \Gamma_2 \]
and:
\begin{eqnarray*}
  \Gamma_1 & = & - 2 \alpha_{t_i}  \int_{\lambda_{t_{i - 1}}}^{\lambda_{t_i}}
  e^{- u}  \left[ \epsilonv (\uv_i, s_i) - \epsilonv (x_{s_i}, s_i) \right]
  \mathd u\\
  \Gamma_2 & = & - 2 \alpha_{t_i}  \int_{\lambda_{t_{i - 1}}}^{\lambda_{t_i}}
  e^{- u}  \left[ \epsilonv (x_{s_i}, s_i) - \widehat{\epsilonv}
  (\widehat{\mathbf{x}}_u, u) \right] \mathd u
\end{eqnarray*}
Similarly to the case $k = 1$, $\mathbb{E} [| \Gamma_2 |] \le Ch^2$. Note that
$\mathbb{E} [| \Gamma_1 |] \leq Ch \Delta_i$ with $\Delta_i =\mathbb{E} [|
\uv_i - \x_{s_i} |]$. Introduce: $U_i =\mathbb{E} [| \y_{t_i} - \x_{t_i} |]$.
We will bound $\Delta_i$ by a function of $U_{i - 1}$. For this, recall:
\[ \uv_i = \frac{\alpha_{s_i}}{\alpha_{t_{i - 1}}} \y_{t_{i - 1}} - 2
   \alpha_{s_i}  \int_{\lambda_{t_{i - 1}}}^{\lambda_{s_i}} e^{- \lambda}
   \epsilonv (\y_{t_{i - 1}}, t_{i - 1}) \mathd \lambda - \sqrt{2}
   \alpha_{s_i}  \int_{\lambda_{t_{i - 1}}}^{\lambda_{s_i}} e^{- \lambda}
   \mathd \tmmathbf{\bar{\omega}} (\lambda), \]
and write the difference:
\[ \uv_i - \x_{s_i} = \frac{\alpha_{s_i}}{\alpha_{t_{i - 1}}}  (\y_{t_{i - 1}}
   - \x_{t_{i - 1}}) + 2 \alpha_{t_i}  \int_{\lambda_{t_{i -
   1}}}^{\lambda_{t_i}} e^{- u}  (\widehat{\epsilonv} (\widehat{\mathbf{x}}_u,
   u) - \epsilonv_{\theta} (\y_{t_{i - 1}}, \lambda_{t_{i - 1}})) \mathd
   \lambda . \]
By the triangular inequality:
\[ | \widehat{\epsilonv} (\widehat{\mathbf{x}}_u, u) - \epsilonv (\y_{t_{i -
   1}}, \lambda_{t_{i - 1}}) | \leq C (| \mathbf{x}_{t_{\lambda} (u)} -
   \x_{t_{i - 1}} | + | \x_{t_{i - 1}} - \y_{t_{i - 1}} | + |t_{\lambda} (u) -
   t_{i - 1} |), \]
and so: $\mathbb{E} [| \widehat{\epsilonv} (\widehat{\mathbf{x}}_u, u) -
\epsilonv (\y_{t_{i - 1}}, \lambda_{t_{i - 1}}) |] \leq C (\sqrt{h} + U_{i -
1} + h)$. Finally:
\[ \Delta_i \leq \frac{\alpha_{s_i}}{\alpha_{t_{i - 1}}} U_{i - 1} + Ch \left(
   \sqrt{h} + U_{i - 1} + h \right) \leq U_{i - 1} + Ch \left( \sqrt{h} + U_{i
   - 1} \right), \]
and $\mathbb{E} [| \Gamma_1 |] \leq Ch^{5 / 2} + C (h + h^2) U_{i - 1}$. This
gives the bound:
\[ U_i \leq Ch^2 + (\frac{\alpha_{t_i}}{\alpha_{t_{i - 1}}} + Ch + Ch^2) U_{i
   - 1} . \]
Now using Theorem \ref{wlocal}, the proof can now be finished following as in
proof in Section \ref{w-disc}.

\subsubsection{Convergence of SEEDS-3:}

Continuing with the same notations as before, we will prove by analogy that:
\begin{equation}
  \label{eq} U_i \leq Ch^2 + (\frac{\alpha_{t_i}}{\alpha_{t_{i - 1}}} + Ch +
  Ch^2 + Ch^3) U_{i - 1},
\end{equation}
so that, by Theorem \ref{wlocal}, we obtain the desired result.

Let $\y_{t_i} = \widetilde{\x}_{t_i}$. We have:
\[ y_{t_i} = \frac{\alpha_{t_i}}{\alpha_{t_{i - 1}}} y_{t_{i - 1}} - 2
   \alpha_{t_i}  \int_{\lambda_{t_{i - 1}}}^{\lambda_{t_i}} e^{- \lambda}
   \epsilonv (y_{t_{i - 1}}, t_{i - 1}) d \lambda \]
\[ - \frac{2 \alpha_{t_i}}{h_i r_2}  \int_{\lambda_{t_{i -
   1}}}^{\lambda_{t_i}} e^{- \lambda}  (\lambda - \lambda_{t_{i - 1}}) 
   (\epsilonv (u_{2 i}, s_{2 i}) - \epsilonv (y_{t_{i - 1}}, t_{i - 1})) d
   \lambda + \text{Noise}, \]
and $x_{t_i} = \frac{\alpha_{t_i}}{\alpha_{t_{i - 1}}} x_{t_{i - 1}} - 2
\alpha_{t_i}  \int_{\lambda_{t_{i - 1}}}^{\lambda_{t_i}} e^{- \lambda}
\widehat{\epsilonv} (\widehat{\mathbf{x}}_{\lambda}, \lambda) \mathd \lambda
+\text{Noise}$, where Noise is the same in both equations. From the proof for $k =
1$, it suffices to bound:
\[ \Gamma = \frac{1}{h_i}  \int_{\lambda_{t_{i - 1}}}^{\lambda_{t_i}} e^{-
   \lambda}  (\lambda - \lambda_{t_{i - 1}})  (\epsilonv (u_{2 i}, s_{2 i}) -
   \epsilonv (y_{t_{i - 1}}, t_{i - 1})) d \lambda, \]
in the $L^1$ norm. By the Lipschitz property of $\epsilonv$:
\begin{eqnarray}
  \mathbb{E} [| \Gamma |] & \leq & Ch (\mathbb{E}[|u_{2 i} - x_{s_{2 i}} |]
  +\mathbb{E}[|x_{t_{i - 1}} - x_{s_{2 i}} |] +\mathbb{E}[|y_{t_{i - 1}} -
  x_{t_{i - 1}} |]) + Ch^2 \nonumber\\
  & \leq & Ch (\mathbb{E}[|u_{2 i} - x_{s_{2 i}} |] + \sqrt{h} + U_{i - 1}) +
  Ch^2 . \nonumber
\end{eqnarray}
Now, let us bound $\mathbb{E} [|u_{2 i} - x_{s_{2 i}} |]$ and for this, write
\[ u_{2 i} = \frac{\alpha_{s_{2 i}}}{\alpha_{t_{i - 1}}} y_{t_{i - 1}} - 2
   \alpha_{s_{2 i}}  \int_{\lambda_{t_{i - 1}}}^{\lambda_{s_{2 i}}} e^{-
   \lambda} \epsilonv (y_{t_{i - 1}}, t_{i - 1}) \mathd \lambda \]
\[ - 2 \alpha_{s_{2 i}}  \int_{\lambda_{t_{i - 1}}}^{\lambda_{s_{2 i}}} e^{-
   \lambda}  (\lambda - \lambda_{t_{i - 1}})  \frac{\epsilon (u_{2 i - 1},
   s_{2 i - 1}) - \epsilon (y_{t_{i - 1}}, t_{i - 1})}{r_1 h_i} \mathd \lambda
   + \text{Noise}, \]
and
\[ x_{s_{2 i}} = \frac{\alpha_{s_{2 i}}}{\alpha_{t_{i - 1}}} x_{t_{i - 1}} - 2
   \alpha_{s_{2 i}}  \int_{\lambda_{t_{i - 1}}}^{\lambda_{s_{2 i}}} e^{-
   \lambda} \widehat{\epsilonv} (\widehat{\mathbf{x}}_{\lambda}, \lambda)
   \mathd \lambda + \text{Noise}, \]
where Noise is the same in both equations. So
\begin{eqnarray*}
  u_{2 i} - x_{s_{2 i}} & = & \frac{\alpha_{s_{2 i}}}{\alpha_{t_{i - 1}}} 
  (y_{t_{i - 1}} - x_{t_{i - 1}}) - 2 \alpha_{s_{2 i}}  \int_{\lambda_{t_{i -
  1}}}^{\lambda_{s_{2 i}}} e^{- \lambda}  (\epsilonv (\y_{t_{i - 1}}, t_{i -
  1}) - \widehat{\epsilonv} (\widehat{\mathbf{x}}_{\lambda}, \lambda)) \mathd
  \lambda\\
  &  & - 2 \alpha_{s_{2 i}}  \int_{\lambda_{t_{i - 1}}}^{\lambda_{s_{2 i}}}
  e^{- \lambda}  (\lambda - \lambda_{t_{i - 1}}) D \mathd \lambda,
\end{eqnarray*}
where $D = \frac{\epsilon (u_{2 i - 1}, s_{2 i - 1}) - \epsilon (y_{t_{i -
1}}, t_{i - 1})}{r_1 h_i}$. From the convergence proof of SEEDS-1, and the
Lipschitz property of $\epsilon$, we obtain:
\[ \mathbb{E} [|u_{2 i} - x_{s_{2 i}} |] \leq U_{i - 1} + ChU_{i - 1} + Ch^2 +
   Ch\mathbb{E} [|u_{2 i - 1} - y_{t_{i - 1}} |] . \]
Again, by the triangular inequality:
\[ \mathbb{E} [|u_{2 i - 1} - y_{t_{i - 1}} |] \leq \mathbb{E} [|u_{2 i - 1} -
   x_{s_{2 i - 1}} |] + C \sqrt{h} + U_{i - 1}, \]
and since
\[ u_{2 i - 1} - x_{s_{2 i - 1}} = \frac{\alpha_{s_{2 i - 1}}}{\alpha_{t_{i -
   1}}}  (y_{t_{i - 1}} - x_{t_{i - 1}}) - 2 \alpha_{s_{2 i - 1}} 
   \int_{\lambda_{t_{i - 1}}}^{\lambda_{s_{2 i - 1}}} e^{- \lambda}
   (\widehat{\epsilonv} (\widehat{\mathbf{x}}_{\lambda}, \lambda) - \epsilonv
   (y_{t_{i - 1}}, t_{i - 1}) \mathd \lambda, \]
we have, as before,
\[ \mathbb{E} [|u_{2 i - 1} - x_{s_{2 i}} |] \leq U_{i - 1} + ChU_{i - 1} +
   Ch^2 . \]
Combining the previous inequalities leads to (\ref{eq}). This finishes the
proof of Corollary \ref{cor:wocv}.

%% file: App-D.tex
\section{Implementation Details}
\label{app:imple-details}

The SEEDS solvers used in our experiments are exactly the variants we 
described in Algorithms \ref{alg:SERK-solver-1}, \ref{alg:SERK-solver-2}, 
\ref{alg:SERK-solver-3} in the main part of the paper. In particular, SEEDS-2 
and SEEDS-3 solvers are completely determined by one-parameter families of 
deterministic exponential integrators from \cite{hochbruck2010exponential} of 
order 2 and 3 respectively and prescribed by the following Butcher tableaux:
\[ \begin{array}{l|ll}
     0 &  & \\
     c_2 & c_2 \varphi_{1, 2} & \\
     \hline
     & \left( 1 - \frac{1}{2 c_2} \right) \varphi_1 & \frac{1}{2 c_2}
     \varphi_1
   \end{array} \qquad \begin{array}{l|lll}
     0 & 0 & 0 & 0\\
     c_2 & c_2 \varphi_{1, 2} & 0 & 0\\
     \frac{2}{3} & \frac{2}{3} \varphi_{1, 3} - \frac{4}{9 c_2} \varphi_{2, 3}
     & \frac{4}{9 c_2} \varphi_{2, 3} & 0\\
     \hline
     & \varphi_1 - \frac{3}{2} \varphi_2 & 0 & \frac{3}{2} \varphi_2
   \end{array} \]
In all experiments we fix the parameter $c_2=0.5$ for SEEDS-2 and 
$c_2=\frac{1}{3}$ for SEEDS-3. We point out that the Butcher tableau here 
associated to SEEDS-3 is the result of a weakening of the 
\textit{stiff order conditions} so it might suffer from order reduction in 
the deterministic case. We also point out that a full theory of stiff order 
conditions for stochastic exponential Runge-Kutta methods for semi-linear DEs 
with a time-varying linear coefficient have not been yet developed to the 
authors knowledge. As such, this is only an analogy whose purpose is to 
clarify how our solvers relate to well-known solvers from the literature, but 
such Butcher tableaux do not rigorously reflect their convergence order. 

\subsection{Stabilization of the exponential terms}

In the proposed algorithms, one subtle detail is to re-arrange all equations 
in order for them to make use only of the $\texttt{expm1}(h)$ function which 
computes $e^h-1$ with great numerical stability, specially for small values 
of $h$.
We have for values $r_1=1/3$, $r_2=2/3$, the following identity
\begin{eqnarray*}
  \sqrt{e^{2 h_i} - e^{h_i}} z_i + \sqrt{e^{h_i} - 1} v_i & = & \sqrt{e^{h_i}
  - 1} \left( \sqrt{e^{h_i}} z_i + v_i \right)\\
  & = & \sqrt{e^{h_i} - 1} \left( e^{\frac{h_i}{2}} z_i + v_i \right)\\
  & = & \sqrt{\texttt{expm1}(h_i)} \left( \left(\texttt{expm1}\left(
  \frac{h_i}{2} \right) + 1 \right) z_i + v_i \right)\\
  & = & \sqrt{\texttt{expm1}(h_i)} \left(\texttt{expm1} \left( \frac{h_i}{2}
  \right) + 1 \right) z_i + \sqrt{\texttt{expm1}(h_i)} v_i.
\end{eqnarray*}
Now by using
\begin{eqnarray*}
  e^{2 (r_2 - r_1) h_i} & = & e^{2 \left( \frac{2}{3} - \frac{1}{3} \right)
  h_i} = e^{2 \frac{1}{3} h_i} = e^{r_2 h_i},
\end{eqnarray*}
we now compute
\begin{eqnarray*}
  \sqrt{e^{2 r_2 h_i} - e^{r_2 h_i}} z^1_i + \sqrt{e^{r_2 h_i} - 1} z^2_i & =
  & \sqrt{e^{r_2 h_i} - 1} \left( \sqrt{e^{r_2 h_i}} z^1_i + z^2_i \right)\\
  & = & \sqrt{e^{r_2 h_i} - 1} \left( e^{\frac{h_i}{3}} z^1_i + z^2_i
  \right)\\
  & = & \sqrt{\texttt{expm1}(r_2 h_i)} ((\texttt{expm1}(r_1 h_i) + 1) z^1_i +
  z^2_i)\\
  &  & \\
  \sqrt{e^{2 h_i} - e^{2 r_2 h_i}} z^1_i & = & \sqrt{e^{3 r_2 h_i} - e^{2 r_2
  h_i}} z^1_i\\
  & = & e^{r_2 h_i} \sqrt{e^{r_2 h_i} - 1} z^1_i\\
  & = & \sqrt{\texttt{expm1}(r_2 h_i)} (\texttt{expm1} (r_2 h_i) + 1) z^1_i.
\end{eqnarray*}
Finally, we obtain
\begin{eqnarray*}
  &  & \sqrt{e^{2 h_i} - e^{2 r_2 h_i}} z^1_i + \sqrt{e^{2 r_2 h_i} - e^{r_2
  h_i}} z^2_i + \sqrt{e^{r_2 h_i} - 1} z^3_i\\
  & = & \sqrt{\texttt{expm1} (r_2 h_i)} ((\texttt{expm1} (r_2 h_i) + 1) z^1_i +
  (\texttt{expm1} (r_1 h_i) + 1) z^2_i + z^3_i) .
\end{eqnarray*}

\subsection{Noise schedules parameterizations}

The inverses of $\lambda (t)$  are given in the (VP) linear and cosine 
schedules respectively by
  \begin{eqnarray*}
    t_{\lambda} (\lambda) & = & \frac{2 \log (e^{- 2 \lambda} + 1)}{\sqrt{\beta_{\min}^2 + 2
    (\beta_{\max} - \beta_{\min}) \log (e^{- 2 \lambda} + 1)} + \beta_{\min}}, \\
    t_{\lambda} (\lambda) & = & \frac{2 (1 + s)}{\pi} \arccos \left( \text{exp} \Big( -
    \dfrac{1}{2} \log (e^{- 2 \lambda} + 1) + \log \cos \big( \cfrac{\pi s}{2
    (1 + s)} \big)\Big) \right) - s. 
  \end{eqnarray*}
  In EDM Noise Prediction case, the inverses of $\lambda(t)$, as given in 
  \eqref{edm-lambdas}, with respect to the NPFO and NRSDE are respectively 
  given by

  \begin{eqnarray*}
      t_\lambda (\lambda) = \sigma_d \tan (e^{- \lambda})
      \quad\text{and}\quad
      t_\lambda (\lambda) = \dfrac{\sigma_d}{\sqrt{\frac{1}{\sigma^2_d e^{-2\lambda}} - 1}}.
  \end{eqnarray*}

\subsection{EDM discretization}
We follow Karras et al. {\cite{Karras2022edm}} to implement the EDM 
discretization timesteps $\{t_i\}_{i=0}^M$ as $t_i=\sigma^{-1}(\sigma_i)$ 
such that, for $\rho>0$,
\begin{eqnarray*}
\sigma_{i<M}=\left(\sigma_\text{max}^{\frac{1}{\rho}}+\frac{i}{N-1}\left(\sigma_\text{min}^{\frac{1}{\rho}}-\sigma_\text{max}^{\frac{1}{\rho}}\right)
    \right)^{\rho}
    \quad\text{and}\quad 
    \sigma_N=0.
\end{eqnarray*}
From the definition, we note that $\sigma_0=\sigma_\text{max}$ and 
$\sigma_{M-1}=\sigma_\text{min}$, where $\sigma_\text{min}$ and 
$\sigma_\text{max}$ denote the minimum and maximum noise magnitude, 
respectively. We also keep default value $\rho=7$ as in 
{\cite{Karras2022edm}}. However, we figured out that when using EDM 
discretization with linear schedule, the noise schedule improvement in iDDPM 
pre-trained models \cite{nichol2021improved} would result in two consecutive 
time steps of the same value, i.e. $t_j=t_{j+1}$ for some index 
$j=1,\dots,M-1$ and large steps ($M\geq 61$). Thus, for SEEDS-3 and 
DPM-Solver-3 \cite{dpm-solver}, the usage of the function 
$\varphi_2(h)=\dfrac{e^h-h-1}{h^2}$ (see Appendix \ref{app:solvers} for more 
details) will cause zero division error where 
$h=\lambda_{t_{j+1}}-\lambda_{t_j}=0$. Therefore, in our implementation, we 
ignore the noise schedule improvement of iDDPM \cite{nichol2021improved} 
models when using solvers of order three.

\subsection{Final sampling step}
The sampling phase in DPMs using SEEDS follows the RSDE, which requires 
gradual computing through discretization time steps $\{t_i\}_{i=0}^M$ and the 
latter goes from $t_0=T$ to $t_M=0$. In our implementation,  to avoid the 
logarithm of zero error at the last step, i.e. $\log\sigma_{t_M}=\log(0)$, we 
stop the sampling phase at step $(M-1)$. Hence, the NFE used in a run will be 
given as
\begin{eqnarray*}
    \text{NFE}=k\times (M-1),
\end{eqnarray*}
where $k$ represents the order of the solver. We also do not use the 
``denoising'' trick, i.e., ignoring the random noise at the last step and 
leave it to further research. 

%% file: App-E.tex
\section{Reminders on Stochastic Exponential Integrators}
\label{app:solvers}

Let us consider a SDE of the form
\begin{equation}
  \mathd \mathbf{x} (t) = [a (t) \mathbf{x} (t) + c (t) f (\mathbf{x} (t), t)]
  \mathd t + g (t) \mathd \tmmathbf{\omega} (t), \label{GSDE}
\end{equation}
where $a, c : [0, T] \rightarrow \mathbb{R}$ and $g : [0, T] \rightarrow
\mathbb{R}^{d \times d}$. In other words we concentrate to high-dimensional
semi-linear non autonomous SDEs with additive noise. The objective of this
section is to construct explicit stochastic exponential derivative-free
methods for the above equation following the Runge-Kutta (RK) approach. 
These methods ideally should fulfill the
following properties:
\begin{enumerate}
  \item If $f \equiv 0$, then \eqref{GSDE} can be solved \tmtextit{exactly};
  
  \item If $g \equiv 0$, then a SEEDS method for \eqref{GSDE} identifies with
  an exponential RK (ERK) method;
  
  \item If $a \equiv 0$, then a SEEDS method for \eqref{GSDE} identifies with
  a stochastic RK (SRK) method and if moreover $g \equiv 0$ then it identifies with a
  classical RK ODE method.
\end{enumerate}
Before tackling the aimed SEEDS problem let us rapidly recall elementary
constructions of RK, ERK, weak and strong SRK methods. We will not deal with
time-adaptive methods here.

\subsection{Derivative-free exponential ODE schemes}

\subsubsection{Runge-Kutta approach}\label{DFRK}

Derivative-free schemes are obtained by comparing the It{\^o}-Taylor expansion
of the above paragraph with expressions of $\mathbf{x} (t)$ in terms of its
intermediate evaluations between $s$ and $t = s + h$ and the Taylor expansions
of such evaluations. As a simple example, in the ODE regime
\begin{equation}
  \mathd \mathbf{x} (t) = f (\mathbf{x} (t), t) \mathd t. \label{ODE}
\end{equation}
As such, analytic solutions to the above equations are of the form
\begin{eqnarray*}
  \mathbf{x} (t) & = & \mathbf{x} (s) + \int_s^t f (\mathbf{x} (\tau), \tau)
  \mathd \tau .
\end{eqnarray*}
Now, Taylor expansion gives, up to order 2:
\begin{eqnarray*}
  \mathbf{x} (t) & = & \mathbf{x} (s) + h \mathbf{x}' (s) + \frac{h^2}{2}
  \mathbf{x}'' (s) +\mathcal{O} (h^3)\\
  & = & \mathbf{x} (s) + hf (\mathbf{x} (s), s) + \frac{h^2}{2}  (\mathbf{x}'
  (s) \partial_{\mathbf{x}} f (\mathbf{x} (s), s) + \partial_t f (\mathbf{x}
  (s), s)) +\mathcal{O} (h^3)\\
  & = & \mathbf{x} (s) + hf (\mathbf{x} (s), s) + \frac{h^2}{2}  (f
  (\mathbf{x} (s), s) \partial_{\mathbf{x}} f (\mathbf{x} (s), s) + \partial_t
  f (\mathbf{x} (s), s)) +\mathcal{O} (h^3) .
\end{eqnarray*}
A straightforward recursion yields a Taylor expansion
\begin{eqnarray*}
  \mathbf{x} (t) & = & \mathbf{x} (s) + \sum_{k = 1}^n \frac{h^k}{k!} L_t^{k -
  1} f (\mathbf{x} (t), t) + \int_s^t \cdots \int_s^{\tau} L_t^n f (\mathbf{x}
  (\tau), \tau) \mathd \tau^{n + 1},
\end{eqnarray*}
where we denote the generalized infinitesimal operator of the solution
$\mathbf{x}$ of \eqref{ODE} by
\[ L_t (\bullet) = \partial_t (\bullet) + f (\mathbf{x} (t), t) \cdot
   \partial_{\mathbf{x}} (\bullet) . \]
Derivative-free methods seek to get rid of the derivatives in $L_t$ by
computing Taylor expansions of $f$ at well-chosen points. In the explicit
one-step case, this amounts of defining
\[ \widehat{\mathbf{x}} (t) = \widehat{\mathbf{x}} (s) + h \sum_{k = 1}^n \Phi
   (\widehat{\mathbf{x}} (s), s, h), \]
where the $\Phi$ function does not contain derivatives of $f$. The general
high order case is given by well-tuned coefficients in the following scheme
\begin{eqnarray*}
  \widehat{\mathbf{x}} (t) & \assign & \widehat{\mathbf{x}} (s) + h \sum_{i =
  1}^n \alpha_i f (\mathbf{x}_i, s + c_i h)\\
  \mathbf{x}_i & = & \widehat{\mathbf{x}} (s) + h \sum_{j = 1}^n a_{i, j} f
  (\mathbf{x}_j, s + c_j h) .
\end{eqnarray*}
By denoting $\alpha = [\alpha_1 \ldots \alpha_n]$, $c = [c_1 \ldots c_n]^T$
and $A = (a_{i, j})$, these can be represented by a Butcher tableau of the
form
\[ \begin{array}{c|c}
     c & A\\
     \hline
     & \alpha
   \end{array} \]
In the following sections we will present extended version of this tableau for
representing more involved numerical schemes. As such, the Euler, midpoint,
Heun and general second order explicit methods are respectively:
\[ \begin{array}{c|c}
     0 & 0\\
     \hline
     & 1
   \end{array} \qquad \begin{array}{c|cc}
     0 & 0 & 0\\
     1 / 2 & 1 / 2 & 0\\
     \hline
     & 0 & 1
   \end{array} \qquad \begin{array}{c|cc}
     0 & 0 & 0\\
     1 & 1 & 0\\
     \hline
     & 1 / 2 & 1 / 2
   \end{array} \qquad \begin{array}{c|cc}
     0 & 0 & 0\\
     c_2 & c_2 & 0\\
     \hline
     & 1 - \frac{1}{2 c_2} & \frac{1}{2 c_2}
   \end{array} \]

\subsubsection{Exponential Runge-Kutta approach}

We now concentrate on a non-autonomous semi-linear ODE of the form
\begin{equation}
  \mathd \mathbf{x} (t) = [f (t) \mathbf{x} (t) + g (\mathbf{x} (t), t)]
  \mathd t, \label{slode}
\end{equation}
where $f (t) = \bar{f} (t) \cdot \tmop{Id}_d$ and $g (\mathbf{x} (t), t) =
\bar{g} (\mathbf{x} (t), t) \cdot \tmop{Id}_d$ are diagonal time-dependent
matrices with identical diagonal coefficients. In particular, we have
$[f^{(k)} (t), f^{(l)} (s)] = 0$, a property we will need to facilitate
exponential matrix multiplications. As $f (s)$ is a $d$-dimensional diagonal
matrix, the fundamental matrix for \eqref{slode}, usually given by the
Peano-Baker series, simplifies in this case to the form
\[ \Phi (t ; s) \assign e^{\int_s^t f (r) \mathd r} = \sum_{n = 0}^{\infty}
   \frac{1}{n!} \left( \int_s^t f (r) \mathd r \right)^n, \]
which satisfies $\Phi' (t ; s) = f (t) \Phi (t ; s)$ and $\Phi (s ; s) = 1$.
Then the exact solution for \eqref{slode} is given, via the variation of
constants formula, by
\begin{equation*}
      \mathbf{x} (t)  =  \Phi (t ; s)  \left( \mathbf{x} (s) + \int_s^t \Phi^{-
  1} (\tau ; s) g (\mathbf{x} (\tau), \tau) \mathd \tau \right) = \Phi (t ; s)
  \mathbf{x} (s) + \int_s^t \Phi (t ; \tau) g (\mathbf{x} (\tau), \tau) \mathd
  \tau .
\end{equation*}
In light of this integral form, one can formalize a general class of

exponential $n$-stage RK methods
\begin{eqnarray*}
  \mathbf{x}_i & = & \gamma_i (h, f_s) \mathbf{x} (s) + \sum_{j = 1}^n a_{ij} 
  (h, f_s) g (\mathbf{x}_j, s + c_j h)\\
  \mathbf{x} (t) & = & \gamma_0 (h, f_s) \mathbf{x} (s) + \sum_{i = 1}^n b_i
  (h, f_s) g (\mathbf{x}_i, s + c_i h),
\end{eqnarray*}
where $\gamma_0, \gamma_i, a_{ij}, b_j$ are functions of the step-size $h$,
$f$and $f_s (r) \assign \int_0^r f (s + \tau) \mathd \tau$. There are two
possible approaches to create exponential integrators, namely the exponential
time-differencing (ETD) approach that uses the variation of constants formula
and makes use of the $\varphi$ functions, and the Lawson approach, also know
as integrating factor (IF), which makes a change of variables on the above SDE
thus avoiding the use of the $\varphi$ functions but computing exponential
factors step-wise.

\tmtextbf{The Lawson and the ETD approaches for Exponential Euler schemes}

There are two choices one can make when computing first order approximations
of
\begin{eqnarray*}
  \mathbf{x} (t) & = & \Phi (t ; s) \mathbf{x} (s) + \int_s^t \Phi (t ; \tau)
  g (\mathbf{x} (\tau), \tau) \mathd \tau .
\end{eqnarray*}
First, by interpolating $g (\mathbf{x} (\tau), \tau)$ as $g (\mathbf{x} (s),
s)$ we obtain
\begin{eqnarray*}
  \mathbf{x} (t) & = & \Phi (t ; s) \mathbf{x} (s) + g (\mathbf{x} (s), s) 
  \int_s^t \Phi (t ; \tau) \mathd \tau .
\end{eqnarray*}
Now two choices remain. Either $\int_s^t \Phi (t ; \tau) \mathd \tau$ is
computed exactly or we again interpolate $\Phi (t ; \tau)$ as $\Phi (t ; s)$.
Taking for simplicity $f \equiv A$ to be constant and by denoting $h = t - s$,
the first case yields the ETD Euler method
\begin{eqnarray*}
  \widehat{\mathbf{x}} (t) & = & \Phi (t ; s) \widehat{\mathbf{x}} (s) + h
  \gamma_1 (f, t, s) g (\widehat{\mathbf{x}} (s), s)\\
  & = & e^{Ah} \widehat{\mathbf{x}} (s) + h \varphi_1  (Ah) g
  (\widehat{\mathbf{x}} (s), s)\\
  & = & \widehat{\mathbf{x}} (s) + h \varphi_1  (Ah)  [g
  (\widehat{\mathbf{x}} (s), s) - \widehat{\mathbf{x}} (s)],
\end{eqnarray*}
and the second yields the IF Euler (also called Lawson-Euler) method
\begin{eqnarray*}
  \widehat{\mathbf{x}} (t) & = & \Phi (t ; s) \widehat{\mathbf{x}} (s) + h
  \Phi (t ; s) g (\widehat{\mathbf{x}} (s), s)\\
  & = & e^{Ah}  (\widehat{\mathbf{x}} (s) + hg (\widehat{\mathbf{x}} (s), s))
  .
\end{eqnarray*}
Now, consider a solution
\begin{eqnarray*}
  \mathbf{x} (t) & = & \Phi (t ; s) \mathbf{x} (s) + \int_0^h \Phi (h ; \tau)
  g (\mathbf{x} (s + \tau), s + \tau) \mathd \tau .
\end{eqnarray*}
Exponential methods then aim to approximate the term $g (\mathbf{x} (s +
\tau), s + \tau)$ by its interpolation polynomial in certain non-confluent
quadrature nodes $c_1, \ldots, c_n$.

\tmtextbf{The ETD approach}

In this case, the variation of constants formula yields
\begin{eqnarray*}
  \mathbf{x} (s + ch) & = & e^{\int_s^{s + ch} a (\tau) \mathd \tau}
  \mathbf{x} (s) + \int_s^{s + ch} e^{\int_{\tau}^{s + ch} a (r) \mathd r} f
  (\mathbf{x} (\tau), \tau) \mathd \tau .\\
  & = & e^{\int_0^{ch} a (s + \tau) \mathd \tau} \mathbf{x} (t) + \int_0^{ch}
  e^{\int_{\tau - s}^{ch} a (s + r) \mathd r} f (\mathbf{x} (s + \tau), s +
  \tau) \mathd \tau .
\end{eqnarray*}
Now, as before, the Taylor expansion of $f$ yields
\begin{eqnarray*}
  f (\mathbf{x} (s + \tau), s + \tau) & = & \sum_{j = 1}^q \frac{\tau^{j -
  1}}{(j - 1) !} f^{(j - 1)} (\mathbf{x} (s), s)\\
  &  & + \int_0^{\tau} \frac{(\tau - \tau_1)^{q - 1}}{(q - 1) !} f^{(q)}
  (\mathbf{x} (s + \tau_1), s + \tau_1) \mathd \tau_1 .
\end{eqnarray*}
Recall that the $\varphi$ functions are given in an integral form as follows
\[ \varphi_{k + 1} (t) = \int_0^1 e^{(1 - \delta) t}  \frac{\delta^k}{k!}
   \mathd \delta, \]
which satisfy $\varphi_k (0) = \frac{1}{k!}$. Now denote
\begin{eqnarray*}
  \varphi_j (t, a) & \assign & \frac{1}{t^j}  \int_0^t e^{\int_{\tau}^t a (r)
  \mathd r}  \frac{\tau^{j - 1}}{(j - 1) !} \mathd \tau, \qquad j \geqslant
  1.\\
  \varphi_1 (h, a) & = & \frac{1}{h}  \int_0^h e^{\int_{\tau}^h a (r) \mathd
  r} \mathd \tau = \int_0^1 e^{h \int_{\theta}^1 a (r) \mathd r} \mathd \theta
  .
\end{eqnarray*}
The exact solution at $s + ch$ now reads
\begin{eqnarray*}
  \mathbf{x} (s + ch) & = & e^{\int_0^{ch} a (s + \tau) \mathd \tau}
  \mathbf{x} (s) + \sum_{j = 1}^q (ch)^j \varphi_j (ch, a) f^{(j - 1)}
  (\mathbf{x} (s), s)\\
  &  & + \int_0^{ch} e^{\int_{\tau}^{ch} a (s + r) \mathd r} \left(
  \int_0^{\tau} \frac{(\tau - \tau_1)^{q - 1}}{(q - 1) !} f^{(q)} (\mathbf{x}
  (s + \tau_1), s + \tau_1) \mathd s \right) \mathd \tau_1 .
\end{eqnarray*}
Using the left endpoint rule yields
\begin{eqnarray*}
  \mathbf{x} (t) & = & e^{\int_0^h a (s + \tau) \mathd \tau} \mathbf{x} (s) +
  f (\mathbf{x} (s), s)  \int_0^h e^{\int_{\tau}^h a (s + r) \mathd r} \mathd
  \tau +\mathcal{O} (h^2)\\
  & = & e^{\int_0^h a (s + \tau) \mathd \tau} \mathbf{x} (s) + h \varphi_1
  (h, a) f (\mathbf{x} (s), s) +\mathcal{O} (h^2) .
\end{eqnarray*}
\tmtextbf{Second order examples:}

The condition $b_2 (z) c_2 = \varphi_2 (z)$ implies $b_2 = \varphi_2 (z) /
c_2$ and we obtain
\begin{eqnarray}
  \widehat{\mathbf{x}} & = & e^{chA} \mathbf{x} (s) + ch \varphi_1  (chA) f
  (\mathbf{x} (s), s) \nonumber\\
  \mathbf{x} (t) & = & e^{hA} \mathbf{x} (s) + h \left( \varphi_1 (hA) -
  \frac{1}{c} \varphi_2 (hA) \right) f (\mathbf{x} (s), s) + \frac{h}{c}
  \varphi_2  (hA) f (\widehat{\mathbf{x}}, s + ch) . 
\end{eqnarray}
A second method is obtained by weakening the above condition to $b_2 (0) c_2 =
\varphi_2 (0) = 1 / 2$ giving
\begin{eqnarray}
  \widehat{\mathbf{x}} & = & e^{chA} \mathbf{x} (s) + ch \varphi_1  (chA) f
  (\mathbf{x} (s), s) \nonumber\\
  \mathbf{x} (t) & = & e^{hA} \mathbf{x} (s) + h \varphi_1  (hA)  \left( 1 -
  \frac{1}{2 c} \right) f (\mathbf{x} (s), s) + \frac{h}{2 c} \varphi_1  (hA)
  f (\widehat{\mathbf{x}}, s + ch) . 
\end{eqnarray}
In some cases this method can suffer from order reduction and not reach order
2 of convergence. Moreover, setting $c = 1 / 2$ gives the exponential midpoint
method:
\begin{eqnarray}
  \widehat{\mathbf{x}} & = & e^{chA} \mathbf{x} (s) + ch \varphi_1  (chA) f
  (\mathbf{x} (s), s) \nonumber\\
  \mathbf{x} (t) & = & e^{hA} \mathbf{x} (s) + h \varphi_1  (hA) f
  (\widehat{\mathbf{x}}, s + ch) . 
\end{eqnarray}
The above order 1 and 2 exponential methods are represented by the following
\tmtextit{exponential} Butcher tableaux
\[ \begin{array}{c|c}
     0 & 0\\
     \hline
     & \varphi_1
   \end{array} \qquad \begin{array}{l|ll}
     0 &  & \\
     c_2 & c_2 \varphi_{1, 2} & \\
     \hline
     & \varphi_1 - \frac{1}{c_2} \varphi_2 & \frac{1}{c_2} \varphi_2
   \end{array} \qquad \begin{array}{l|ll}
     0 &  & \\
     c_2 & c_2 \varphi_{1, 2} & \\
     \hline
     & \left( 1 - \frac{1}{2 c_2} \right) \varphi_1 & \frac{1}{2 c_2}
     \varphi_1
   \end{array} \qquad \begin{array}{l|ll}
     0 &  & \\
     c_2 & c_2 \varphi_{1, 2} & \\
     \hline
     & 0 & \varphi_1
   \end{array} \]
We check that when formally setting $A = 0$ then the first two methods are
identical and give the generic 2nd order RK method, the choice $c = 1$ gives
the Heun method and the choice $c = 1 / 2$ gives the midpoint method.

\tmtextbf{Fourth order methods}

It can be shown that ERK methods need at least 5 stages to achieve order 4. By
setting formally $A = 0$ these methods do not have a non-exponential
counterpart to our knowledge.

\tmtextbf{5-stage sequential} We have a fourth-order ERK scheme given by
\[ \begin{array}{c|ccccc}
     0 &  &  &  &  & \\
     \frac{1}{2} & \frac{1}{2} \varphi_{1, 2} &  &  &  & \\
     \frac{1}{2} & \frac{1}{2} \varphi_{1, 3} - \varphi_{2, 3} & \varphi_{2,
     3} &  &  & \\
     1 & \varphi_{1, 4} - 2 \varphi_{2, 4} & \varphi_{2, 4} & \varphi_{2, 4} &
     & \\
     \frac{1}{2} & \frac{1}{2} \varphi_{1, 5} - 2 a_{5, 2} - a_{5, 4} & a_{5,
     2} & a_{5, 2} & \frac{1}{4} \varphi_{2, 5} - a_{5, 2} & \\
     \hline
     & \varphi_1 - 3 \varphi_2 + 4 \varphi_3 & 0 & 0 & - \varphi_2 + 4
     \varphi_3 & 4 \varphi_2 - 8 \varphi_3
   \end{array} \]
with
\[ a_{5, 2} = c_5 \varphi_{2, 5} - \varphi_{3, 4} + c_5^2 \varphi_{2, 4} - c_5
   \varphi_{3, 5} . \]
This can also be represented by Rosenbrock-like Butcher tableau
\[ \begin{array}{c|ccccc}
     0 &  &  &  &  & \\
     \frac{1}{2} & \frac{1}{2} \varphi_{1, 2} &  &  &  & \\
     \frac{1}{2} & \frac{1}{2} \varphi_{1, 3} & \varphi_{2, 3} &  &  & \\
     1 & \varphi_1 & \varphi_2 & \varphi_2 &  & \\
     \frac{1}{2} & \frac{1}{2} \varphi_{1, 5} & a_{5, 2} & a_{5, 2} &
     \frac{1}{4} \varphi_{2, 5} - a_{5, 2} & \\
     \hline
     & \varphi_1 & 0 & 0 & 4 \varphi_3 - \varphi_2 & 4 \varphi_2 - 8
     \varphi_3
   \end{array} \]
with
\[ a_{5, 2} = c_5 \varphi_{2, 5} - c_3 \varphi_{3, 3} + c_5^2 \varphi_2 -
   \varphi_3 . \]
This traduces into (recall that $D_i = f (\mathbf{x}_i, s + c_i h) - f
(\mathbf{x} (s), s)$)
\begin{eqnarray*}
  \mathbf{x} (t) & = & \mathbf{x} (s) + h (\varphi_1 (hA)) F (\mathbf{x} (s),
  s) + h (4 \varphi_3 (hA) - \varphi_2 (hA)) D_4\\
  &  & + h (4 \varphi_2 (hA) - 8 \varphi_3 (hA)) D_5\\
  \mathbf{x}_2 & = & \mathbf{x} (s) + c_2 h \varphi_1  (c_2 hA) F (\mathbf{x}
  (s), s)\\
  \mathbf{x}_3 & = & \mathbf{x} (s) + hc_3 \varphi_1  (c_3 hA) F (\mathbf{x}
  (s), s) + h \varphi_2  (c_3 hA) D_2\\
  \mathbf{x}_4 & = & \mathbf{x} (s) + h \varphi_1  (hA) F (\mathbf{x} (s), s)
  + h \varphi_2  (hA)  (D_2 + D_3)\\
  \mathbf{x}_5 & = & \mathbf{x} (s) + h \varphi_1  (hA) F (\mathbf{x} (s), s)
  + ha_{5, 2}  (hA)  (D_2 + D_3) + h (c_5^2 \varphi_2 (c_5 hA) - a_{5, 2}
  (hA)) D_4
\end{eqnarray*}
Inspired by this, we deduce the following fourth order Algorithm
\ref{alg:dpm-solver-4-appendix} specialized for DPMs in the VP noise
prediction regime.

\begin{algorithm}[htbp]
    \centering
    \caption{DPM-Solver-4}\label{alg:dpm-solver-4-appendix}
    \begin{algorithmic}[1]
    \State \textbf{def} DPM-Solver-4$(\epsilonv_\theta,\tilde\x_{t_{i-1}}, t_{i-1}, t_i, r=0.5)$:
        \State\quad $h_i \gets \lambda_{t_{i}} - \lambda_{t_{i-1}}$
    \State
        \quad $s_{2} \gets t_\lambda\left(\lambda_{t_{i-1}} + r h_i\right), \quad s_{3} \gets t_\lambda\left(\lambda_{t_{i-1}} + r h_i\right)$
    \State
        \quad $s_{4} \gets t_\lambda\left(\lambda_{t_{i-1}} +  h_i\right), \quad s_{5} \gets t_\lambda\left(\lambda_{t_{i-1}} + r h_i\right)$
    \State 
        \quad $\kv_{1} \gets \epsilonv_\theta(\tilde \x_{t_{i-1}},t_{i-1})$
    \State 
        \quad $\kv_{2} \gets  \frac{\alpha_{s_2}}{\alpha_s} \tilde \x_{t_{i-1}} -
  \sigma_{s_2} (e^{r h} - 1)  \tmmathbf{k_1}$
    \State 
        \quad $\kv_{3} \gets \frac{\alpha_{s_3}}{\alpha_s} \tilde \x_{t_{i-1}} -
  \sigma_{s_3} (e^{r h} - 1) \tmmathbf{k_1} - \sigma_{s_3} \left( 4 \frac{e^{r
  h} - 1}{h } - 2 \right) [\epsilon_{\theta} (\tmmathbf{k_2}, s_2) -
  \tmmathbf{k_1}]$
    \State 
        \quad $\kv_{4} \gets \frac{\alpha_{s_4}}{\alpha_s} \tilde \x_{t_{i-1}} -
  \sigma_{s_4} (e^h - 1) \tmmathbf{k_1} - \sigma_{s_4} \left( \frac{e^h - 1}{h
  } - 1 \right) [\epsilon_{\theta} (\tmmathbf{k_3}, s_3) + \epsilon_{\theta}
  (\tmmathbf{k_2}, s_2) - 2 \tmmathbf{k_1}]$
    \State 
        \qquad $A= \sigma_{s_5} (e^{rh} - 1) \tmmathbf{k_1} - \frac{1}{4} \sigma_{s_5} \left(
  \frac{e^h - 1}{h } - 1 \right) [\tmmathbf{k_1} + \epsilon_{\theta}
  (\tmmathbf{k_2}, s_2) + \epsilon_{\theta} (\tmmathbf{k_3}, s_3)]$
    \State 
        \qquad $B= \sigma_{s_5} \left( \frac{e^{r h} - 1}{h } - \frac{1}{2} \right)
  [\tmmathbf{k_1} + 4 \epsilon_{\theta} (\tmmathbf{k_2}, s_2) + 4
  \epsilon_{\theta} (\tmmathbf{k_3}, s_3) - \epsilon_{\theta} (\tmmathbf{k_4},
  s_4)]$
    \State 
        \qquad $C= \sigma_{s_5} \left( \frac{e^h - 1 + 4 (e^{r h} - 1) - 3 h}{h^2} - 1
  \right) [- \tmmathbf{k_1} - \epsilon_{\theta} (\tmmathbf{k_2}, s_2) -
  \epsilon_{\theta} (\tmmathbf{k_3}, s_3) + \epsilon_{\theta} (\tmmathbf{k_4},
  s_4)]$
    \State 
        \quad $\kv_{5} \gets \frac{\alpha_{s_5}}{\alpha_{t_{i-1}}} \tilde \x_{t_{i-1}} - A -B - C$
    \State 
        \qquad $D=\sigma_t (e^h -
  1) \tmmathbf{k_1} - \sigma_t \left( \frac{e^h - 1}{h } - 1 \right) [4
  \epsilon_{\theta} (\tmmathbf{k_5}, s_5) - \epsilon_{\theta} (\tmmathbf{k_4},
  s_4) - 3 \tmmathbf{k_1}] $
    \State 
        \qquad $E=\sigma_t \left( 4 \frac{e^h - 1 - h}{h^2} - 2 \right)
  [\tmmathbf{k_1} + \epsilon_{\theta} (\tmmathbf{k_4}, s_4) - 2
  \epsilon_{\theta} (\tmmathbf{k_5}, s_5)]$
    \State
        \quad$\tilde \x_{t_{i}} \gets \frac{\alpha_t}{\alpha_{t_{i-1}}} \tilde \x_{t_{i-1}} - D - E$
    \State\quad Return $\x_{t_{i}}$
    \end{algorithmic}
\end{algorithm}

\subsection{Derivative-free exponential SDE schemes}

Let $\mathbf{x}_t$ be the path at the continuous limit $h \to 0$, and $\{
\widehat{\mathbf{x}}_t \}_{t_0}^{t_M}$ be the discretized numerical path,
computed by a numerical scheme $\mathcal{S}$ with $M = 1 / h$ steps of length
$h > 0$. Then, $\mathcal{S}$ has
\begin{enumerate}
  \item \tmtextit{strong order of convergence} $\gamma$ if there is $K > 0$
  such that
  \begin{equation}
    \mathbb{E} [| \mathbf{x}_{t_M} - \widehat{\mathbf{x}}_{t_M} |] \le
    Kh^{\gamma},
  \end{equation}
  \item \tmtextit{weak order of convergence} $\beta$ if there is $K > 0$ and a
  function class $\mathcal{K}$ such that
  \begin{equation}
    |\mathbb{E}[\phi (\mathbf{x}_{t_M})] -\mathbb{E}[\phi
    (\widehat{\mathbf{x}}_{t_M})] | \le Kh^{\beta}, \qquad \forall \phi
    (\cdot) \in \mathcal{K}.
  \end{equation}
\end{enumerate}
Strong convergence is concerned with the precision of the path, while the weak
convergence is with the precision of the moments. As, for diffusion models,
the center of attention is the evolution of the probability densities rather
than that of the noising process of single data samples, weak convergence is
enough to guarantee the well-conditioning \ of our numerical schemes.
Moreover, when the diffusion coefficient vanishes, then both strong and weak
convergence (with the choice $\phi = \tmop{id}$) reduce to the usual
deterministic convergence criterion for ODEs.

\subsubsection{Strong and Weak Stochastic Runge-Kutta approach}

In all what follows $\omega$ is considered a $d$-dimensional Wiener process
(with identity diffusion matrix).

Consider the following SDE
\[ \mathd \mathbf{x} (t) = f (\mathbf{x} (t), t) \mathd t + g (t) \mathd
   \omega (t), \]
where $g (t) = \hat{g} (t) \cdot \tmop{Id}_d$ is considered here as a diagonal
matrix with identical diagonal entries $\hat{g} (t)$. Given an initial value
independent of $\omega$, the integral form of $\mathbf{x} (t)$ is given by
\begin{eqnarray*}
  \mathbf{x} (t) & = & \mathbf{x} (s) + \int_s^t f (\mathbf{x} (\tau), \tau) d
  \tau + \int_s^t g (\tau) \mathd \omega (\tau) .
\end{eqnarray*}
The idea underlying stochastic numerical schemes is very similar to the one
used in the deterministic approach, that is to take expansions of the terms
inside the integrals based at the integral's initial value and replace the
obtained derivatives that appear by interpolated approximations. A key
difference here is that as we have to consider It{\^o}-Taylor expansions, the
infinitesimal operators are different but most importantly most of the
stochastic iterated integrals will need to be approximated in an appropriate
sense, whenever that is possible. We will develop this expansion up to triple
integrals.

\subsubsection{Truncated It{\^o}-Taylor expansions}\label{app:inf-op}

Applying It{\^o} formula to $h = f$ or $g$ yields
\begin{eqnarray*}
  h (\mathbf{x} (t), t) & = & h (\mathbf{x} (s), s) + \int_s^t g (\tau) \cdot
  \partial_{\mathbf{x}} h (\mathbf{x} (\tau), \tau) \mathd \omega (\tau)\\
  &  & + \int_s^t \left( \partial_t h (\mathbf{x} (\tau), \tau) + f
  (\mathbf{x} (\tau), \tau) \cdot \partial_{\mathbf{x}} h (\mathbf{x} (\tau),
  \tau) + \frac{g^2 (\tau)}{2} \partial_{\mathbf{x}^2}^2 h (\mathbf{x} (\tau),
  \tau) \right) \mathd \tau .
\end{eqnarray*}
This allows us to define two differential operators $L, G$ as
\begin{eqnarray}
  L_t & = & \partial_t + f (\mathbf{x} (t), t) \cdot \partial_{\mathbf{x}} +
  \frac{g^2 (t)}{2} \cdot \partial^2_{\mathbf{x}}  \label{diff-op1}\\
  G_t & = & g (t) \cdot \partial_{\mathbf{x}} . 
\end{eqnarray}
In particular $L_t g (t) = \partial_t g (t)$ and $G_t g (t) = 0$. With this
notation, we have
\begin{eqnarray*}
  h (\mathbf{x} (t), t) & = & h (\mathbf{x} (s), s) + \int_s^t L_t h
  (\mathbf{x} (\tau), \tau) \mathd \tau + \int_s^t G_t h (\mathbf{x} (\tau),
  \tau) \mathd \omega (\tau),
\end{eqnarray*}
so our solution reads
\begin{eqnarray*}
  \mathbf{x} (t) & = & \mathbf{x} (s) + \int_s^t f (\mathbf{x} (\tau_1),
  \tau_1) \mathd \tau_1 + \int_s^t g (\tau_1) \mathd \omega (\tau_1)\\
  & = & \mathbf{x} (s) + \int_s^t \left( f (\mathbf{x} (t), t) +
  \int_s^{\tau_1} L_t f (\mathbf{x} (\tau_2), \tau_2) \mathd \tau_2 +
  \int_s^{\tau_1} G_t f (\mathbf{x} (\tau_2), \tau_2) d \omega (\tau_2)
  \right) \mathd \tau_1\\
  &  & + \int_s^t \left( g (t) + \int_s^{\tau_1} L_t g (\tau_2) \mathd \tau_2
  + \int_s^{\tau} G_t g (\tau_2) \mathd \omega (\tau_2) \right) \mathd \omega
  (\tau_1)\\
  & = & \mathbf{x} (s) + f (\mathbf{x} (s), s) h + g (t)  (\omega (t) -
  \omega (s)) + R_1 .
\end{eqnarray*}
Now $G_t g (\tau_2) = 0$ and
\[ g (s) + \int_s^{\tau_1} L_t g (\tau_2) \mathd \tau_2 = g (s) +
   \int_s^{\tau_1} \partial_t g (\tau_2) \mathd \tau_2 = g (s) + g (\tau_1) -
   g (s) = g (\tau_1) . \]
So we have
\[ \int_s^t \left( g (s) + \int_s^{\tau_1} L_t g (\tau_2) \mathd \tau_2 +
   \int_s^{\tau} G_t g (\tau_2) \mathd \omega (\tau_2) \right) \mathd \omega
   (\tau_1) = g (s)  (\omega (t) - \omega (s)), \]
and now our solution reads
\[ \mathbf{x} (t) = \mathbf{x} (s) + f (\mathbf{x} (s), s) h + g (s)  (\omega
   (t) - \omega (s)) + R_1, \]
where
\begin{eqnarray*}
  R_1 & = & \int_s^t \int_s^{\tau_1} L_t f (\mathbf{x} (\tau_2), \tau_2)
  \mathd \tau_2 \mathd \tau_1 + \int_s^t \int_s^{\tau_1} G_t f (\mathbf{x}
  (\tau_2), \tau_2) \mathd \omega (\tau_2) \mathd \tau_1 .
\end{eqnarray*}
Now we have
\begin{eqnarray*}
  &  & L_t f (\mathbf{x} (t), t)\\
  & = & L_t f (\mathbf{x} (s), s) + \int_s^t g (\mathbf{x} (\tau), \tau)
  \cdot \partial_{\mathbf{x}} L_t f (\mathbf{x} (\tau), \tau) \mathd \omega
  (\tau)\\
  &  & + \int_s^t \left( \partial_t L_t f (\mathbf{x} (\tau), \tau) + f
  (\mathbf{x} (\tau), \tau) \cdot \partial_{\mathbf{x}} L_t f (\mathbf{x}
  (\tau), \tau) + \frac{g^2 (\mathbf{x} (\tau), \tau)}{2}
  \partial_{\mathbf{x}^2}^2 L_t f (\mathbf{x} (\tau), \tau) \right) \mathd
  \tau\\
  & = & L_t f (\mathbf{x} (s), s) + \int_s^t L^2_t f (\mathbf{x} (\tau),
  \tau) \mathd \tau + \int_s^t G_t L_t f (\mathbf{x} (\tau), \tau) \mathd
  \omega (\tau),
\end{eqnarray*}
and
\begin{eqnarray*}
  &  & G_t f (\mathbf{x} (t), t)\\
  & = & G_t f (\mathbf{x} (s), s) + \int_s^t g (\mathbf{x} (\tau), \tau)
  \cdot \partial_{\mathbf{x}} G_t f (\mathbf{x} (\tau), \tau) \mathd \omega
  (\tau) + \\
  &  & \int_s^t \left( \partial_t G_t f (\mathbf{x} (\tau), \tau) + f
  (\mathbf{x} (\tau), \tau) \cdot \partial_{\mathbf{x}} G_t f (\mathbf{x}
  (\tau), \tau) + \frac{g^2 (\mathbf{x} (\tau), \tau)}{2}
  \partial_{\mathbf{x}^2}^2 G_t f (\mathbf{x} (\tau), \tau) \right) \mathd
  \tau\\
  & = & G_t f (\mathbf{x} (s), s) + \int_s^t L_t G_t f (\mathbf{x} (\tau),
  \tau) \mathd \tau + \int_s^t G^2_t f (\mathbf{x} (\tau), \tau) \mathd \omega
  (\tau) .
\end{eqnarray*}
Now, if we denote $\mathd \omega^0 (\tau) = \mathd \tau$, $\mathd \omega^1
(\tau) = \mathd \omega (\tau)$ and
\[ I_{(i)} = \int_s^t \mathd \omega^i (\tau_1)  \qquad I_{(i, j)} = \int_s^t
   \int_s^{\tau_1} \mathd \omega^j (\tau_2) \mathd \omega^i (\tau_1)  \qquad
   i, j \in \{0, 1\}, \]
applying the same procedure to $R_1$ leads to
\begin{eqnarray*}
  &  & R_1\\
  & = & \int_s^t \int_s^{\tau_1} L_t f (\mathbf{x} (\tau_2), \tau_2) \mathd
  \tau_2 \mathd \tau_1 + \int_s^t \int_s^{\tau_1} G_t f (\mathbf{x} (\tau_2),
  \tau_2) \mathd \omega (\tau_2) \mathd \tau_1\\
  & = & \int_s^t \int_s^{\tau_1} \left[ L_t f (\mathbf{x} (t), t) +
  \int_s^{\tau_2} L^2_t f (\mathbf{x} (\tau_3), \tau_3) \mathd \tau_3 +
  \int_s^{\tau_2} G_t L_t f (\mathbf{x} (\tau_3), \tau_3) \mathd \omega
  (\tau_3) \right] \mathd \tau_2 \tau_1\\
  & + &  \int_s^t \int_s^{\tau_1} \left[ G_t f (\mathbf{x} (t), t) +
  \int_s^{\tau_2} L_t G_t f (\mathbf{x} (\tau_3), \tau_3) \mathd \tau +
  \int_s^{\tau_2} G^2_t f (\mathbf{x} (\tau_3), \tau_3) \mathd \omega (\tau_3)
  \right] \mathd \omega (\tau_2) \mathd \tau_1\\
  & = & L_t f (\mathbf{x} (t), t)  \int_s^t \int_s^{\tau_1} \mathd \tau_2
  \mathd \tau_1 + G_t f (\mathbf{x} (t), t)  \int_s^t \int_s^{\tau_1} \mathd
  \omega (\tau_2) \mathd \tau_1\\
  & + &  \int_s^t \int_s^{\tau_1} \int_s^{\tau_2} L^2_t f (\mathbf{x}
  (\tau_3), \tau_3) \mathd \tau_3 \mathd \tau_2 \mathd \tau_1 + \int_s^t
  \int_s^{\tau_1} \int_s^{\tau_2} G_t L_t f (\mathbf{x} (\tau_3), \tau_3)
  \mathd \omega (\tau_3) \mathd \tau_2 \mathd \tau_1\\
  & + &  \int_s^t \int_s^{\tau_1} \int_s^{\tau_2} L_t G_t f (\mathbf{x}
  (\tau_3), \tau_3) \mathd \tau_3 \mathd \omega (\tau_2) \mathd \tau_1\\
  & + &  \int_s^t \int_s^{\tau_1} \int_s^{\tau_2} G^2_t f (\mathbf{x}
  (\tau_3), \tau_3) \mathd \omega (\tau_3) \mathd \omega (\tau_2) \mathd
  \tau_1\\
  & = & L_t f (\mathbf{x} (s), s) I_{(0, 0)} + G_t f (\mathbf{x} (s), s)
  I_{(0, 1)} + L^2_t f (\mathbf{x} (s), s) I_{(0, 0, 0)} + G_t L_t f
  (\mathbf{x} (s), s) I_{(0, 0, 1)}\\
  &  & + L_t G_t f (\mathbf{x} (s), s) I_{(0, 1, 0)} + G^2_t f (\mathbf{x}
  (s), s) I_{(1, 1, 0)} + R_2,
\end{eqnarray*}
and with $R_2$ consisting on quadruple integrals. As such, our solution now
reads
\begin{eqnarray*}
  \mathbf{x} (t) & = & \mathbf{x} (s) + f (\mathbf{x} (s), s) h + g (s) 
  (\omega (t) - \omega (s)) + L_t f (\mathbf{x} (s), s) I_{(0, 0)} + G_t f
  (\mathbf{x} (s), s) I_{(0, 1)}\\
  &  & + L^2_t f (\mathbf{x} (s), s) I_{(0, 0, 0)} + G_t L_t f (\mathbf{x}
  (s), s) I_{(0, 0, 1)}\\
  &  & + L_t G_t f (\mathbf{x} (s), s) I_{(0, 1, 0)} + G^2_t f (\mathbf{x}
  (s), s) I_{(1, 1, 0)} + R_2 .
\end{eqnarray*}
Now, in the SDE regime, one cannot get rid of the iterated It{\^o} integral
and so stochastic RK methods cannot be derived as simple extensions of their
deterministic counterparts. In order to continue we now take into account the
fact that the diffusion SDE has additive and diagonal noise. In this case,
both the It{\^o} and the Stratonovich SDE coincide.

\tmtextbf{Iterated integrals}

Now, for simplicity, set $t = 0$. We then have $I_{(0)} = h$, $I_{(0, 0)} =
\int_0^h \int_0^{\tau_1} \mathd \tau_2 \mathd \tau_1 = \frac{h^2}{2}$, $I_{(0,
0, 0)} = \frac{h^3}{6}$. Now notice that
\begin{eqnarray*}
  I_{(1)} & \assign & \hat{w}_h \sim \mathcal{N} (0, h)\\
  I_{(0, 1)} & \assign & \hat{z}_h \assign \int_0^h \int_0^{\tau_1} \mathd
  \omega (\tau_2) \mathd \tau_1 = \underset{n \rightarrow \infty}{\lim}
  \frac{h}{n}  \sum^{n - 1}_{i = 0} \sum^i_{j = 1} \epsilon_j, \qquad
  \epsilon_j \sim \mathcal{N} \left( 0, \frac{h}{n} \right) .
\end{eqnarray*}
Additionally, $\hat{w}_h$ and $\hat{z}_h$ satisfy $\mathbb{E} [\hat{w}_h^2] =
h$,
\begin{eqnarray*}
  \mathbb{E} [(\hat{w}_h h - \hat{z}_h)^2] & = & \mathbb{E} \left[ \left(
  \int_0^h \tau \mathd \omega (\tau) \right)^2 \right] = \frac{1}{3} h^3\\
  \mathbb{E} [\hat{w}_h  \hat{z}_h] & = & \mathbb{E} \left[ \hat{w}_h 
  \int_0^h \tau \mathd \omega (\tau) \right] =\mathbb{E} \left[ \int_0^h \tau
  \mathd \tau \right] = \frac{1}{2} h^2\\
  \mathbb{E} [\hat{z}_h^2] & = & \mathbb{E} [(\hat{w}_h h - \hat{z}_h)^2 - h^2
  \hat{w}_h^2 + 2 h \hat{w}_h  \hat{z}_h] = \frac{1}{3} h^3 .
\end{eqnarray*}

\subsubsection{Integral approximations}

\tmtextbf{Weak Approximations}

When crafting weak stochastic approximations to SDEs one may replace multiple
It{\^o} integrals by other random variables satisfying the corresponding
moment conditions. We will denote $\hat{I}_{\alpha}$ the approximation of
$I_{\alpha}$ for $\alpha$ a multi-index following \cite[Corollary
5.12.1]{KP1}. Of course, the deterministic integrals $I_{(0, \ldots, 0)}$ need not
to be approximated.

\tmtextbf{First order approximations}

The random variable $\hat{I}_{(1)}$ must satisfy \ for some constant $K$:
\[ |\mathbb{E}[\hat{I}_{(1)}] | + |\mathbb{E}[(\hat{I}_{(1)})^3] | +
   |\mathbb{E}[(\hat{I}_{(1)})^2 - h] | \leqslant Kh^2 \]
Two possible choices for $\hat{I}_{(1)}$ are either $\hat{I}_{(1)} \sim
\mathcal{N} (0, h)$ or $\hat{I}_{(1)}$ is a two-pointed distributed discrete
random variable with
\[ \mathbb{P} \left[ \hat{I}_{(1)} = \pm \sqrt{h} \right] = \frac{1}{2} . \]
\tmtextbf{Second order approximations}

The random variable $\hat{I}_{(1)}$ must satisfy for some constant $K$:
\[ |\mathbb{E}[\hat{I}_{(1)}] | + |\mathbb{E}[(\hat{I}_{(1)})^3] | +
   |\mathbb{E}[(\hat{I}_{(1)})^5] | + |\mathbb{E}[(\hat{I}_{(1)})^2 - h] | +
   |\mathbb{E}[(\hat{I}_{(1)})^4 - 3 h^2] | \leqslant Kh^2 . \]
Two possible choices for $\hat{I}_{(1)}$ are either $\hat{I}_{(1)} \sim
\mathcal{N} (0, h)$ or $\hat{I}_{(1)}$ is a three-pointed distributed random
variable with
\[ \mathbb{P} \left[ \hat{I}_{(1)} = \pm \sqrt{3 h} \right] = \frac{1}{6},
   \qquad \mathbb{P} [\hat{I}_{(1)} = 0] = \frac{2}{3} . \]
The rest follows from the above calculations:
\[ \hat{I}_{(0, 1)} = \frac{1}{2} h \hat{I}_{(1)}, \qquad i = 0, 1. \]
\tmtextbf{Third order approximations}

One can choose $\hat{I}_{(1)} \sim \mathcal{N} (0, h)$, $\hat{I}_{(0, 1)} \sim
\mathcal{N} \left( 0, \frac{1}{3} h^3 \right)$ satisfying $\mathbb{E}
[\hat{I}_{(1)}  \hat{I}_{(0, 1)}] = \frac{1}{2} h^2$.

Then, one can deduce the following:
\begin{eqnarray*}
  \hat{I}_{(1, 0, 0)} = \hat{I}_{(0, 1, 0)} = \hat{I}_{(0, 0, 1)} & = &
  \frac{1}{6} h^2  \hat{I}_{(1)}\\
  \hat{I}_{(0, 1, 1)} = \hat{I}_{(1, 0, 1)} = \hat{I}_{(1, 1, 0)} & = &
  \frac{1}{6} h (\hat{I}_{(1)}^2 - h) .
\end{eqnarray*}
Thus, we can write the solution weak approximation as
\begin{eqnarray}
  \mathbf{x} (t) & = & \mathbf{x} (s) + f (\mathbf{x} (s), s) h + g (s) 
  \hat{I}_{(1)} + L_t f (\mathbf{x} (s), s) \frac{h^2}{2} + R_2 \nonumber\\
  &  & + G_t f (\mathbf{x} (s), s)  \hat{I}_{(0, 1)} + L^2_t f (\mathbf{x}
  (s), s) \frac{h^3}{6} + G_t L_t f (\mathbf{x} (s), s) \frac{1}{6} h^2  \hat{I}_{(1)}  \nonumber\\
  &  & + L_t
  G_t f (\mathbf{x} (s), s) \frac{1}{6} h^2  \hat{I}_{(1)} + G^2_t f
  (\mathbf{x} (s), s) \frac{1}{6} h (\hat{I}_{(1)}^2 - h). 
\end{eqnarray}
An example of such a pair $(\hat{I}_{(1)}, \hat{I}_{(0, 1)}) = (\hat{w}_h,
\hat{z}_h)$ can be easily obtained as follows
\begin{eqnarray*}
  \left[ \begin{array}{c}
    \hat{w}_h\\
    \hat{z}_h
  \end{array} \right] & = & \left[ \begin{array}{cc}
    \sqrt{h} & 0\\
    \frac{h \sqrt{h}}{2} & \frac{h \sqrt{h}}{2 \sqrt{3}}
  \end{array} \right] \left[ \begin{array}{c}
    u_1\\
    u_2
  \end{array} \right], \qquad u_1, u_2 \overset{\text{i.i.d.}}{\sim}
  \mathcal{N} (0, 1) .
\end{eqnarray*}
Indeed, for such a pair we have
\begin{eqnarray*}
  \mathbb{E} \left[ \left[ \begin{array}{c}
    \hat{w}_h\\
    \hat{z}_h
  \end{array} \right] \left[ \begin{array}{cc}
    \hat{w}_h & \hat{z}_h
  \end{array} \right] \right] & = & \left[ \begin{array}{cc}
    h & \frac{h^2}{2}\\
    \frac{h^2}{2} & \frac{h^3}{3}
  \end{array} \right] .
\end{eqnarray*}
In light of the above expression, the truncated Taylor expansions we refer in
the main part of the paper consists on the consideration of only the
coefficients in $L_t^k$. The only noise noise contribution we will consider
corresponds to $g (s)  \hat{I}_{(1)}$.

%% file: App-F.tex
\section{Experiment Details}
\label{app:exp-details}

We evaluate the Fréchet inception distance (FID) after generating 50K samples with each solver, and compare with the statistics of real-data. 
In our experiments we make use of the code from \cite{Karras2022edm} for continuously trained models as well as their reference FID stats\footnote{\url{https://nvlabs-fi-cdn.nvidia.com/edm/fid-refs/}} and that of \cite{dpm-solver} for discretely trained models. 

All the experiments of SEEDS for continuous-time models are parameterized within the EDM framework with the discretization of type 
EDM, linear schedule, and scaling none as described on \cite{Karras2022edm} in noise prediction mode unless explicitly stated. We use the SEEDS-3 method that has 3 NFEs per step and fixed step-size and report FID scores at NFEs divisible by 3. 

We leverage the explicit Langevin-like ``churn'' trick using in \cite{Karras2022edm} to add or remove noise in the sampling phase. Specifically, \cite{Karras2022edm} uses 4 hyper-parameters
$S_\text{churn}, S_\text{min}, S_\text{max}$ and $S_\text{noise}$ in which $S_\text{churn}$ controls the overall amount of stochasticity added before giving the input to the SEEDS-3 method when the noise level (or time step in EDM configuration) $t_i$ is contained in the noise interval $[S_\text{min}, S_\text{max}]$. It means that the EDM proposed sampler is stochastic under some conditions of those hyper-parameters and deterministic otherwise, while our method is completely stochastic. In our experiments, we set $S_\text{churn}=0$ except for ImageNet-64 EDM optimized model. We noticed that using the additional stochasticity indeed helps to improve the image quality as in Fig. \ref{fig:fid-cont} (c). Moreover, setting $S_\text{noise}$ slightly above 1 might correct the errors in earlier steps more effectively as indicated in \cite{Karras2022edm}.

\subsection{Pre-trained model specifications}
For producing the CIFAR-10 time-continuous results in Table \ref{table-fid-all}, we use the VP DDPM++ continuous architecture. These models are publicly available in conditional\footnote{\url{https://nvlabs-fi-cdn.nvidia.com/edm/pretrained/baseline/baseline-cifar10-32x32-cond-vp.pkl}} and unconditional\footnote{\url{https://nvlabs-fi-cdn.nvidia.com/edm/pretrained/baseline/baseline-cifar10-32x32-uncond-vp.pkl}} versions and were directly derived from \cite{song2020score} under the Apache 2.0 license. On the unconditional mode (Figure \ref{fig:fid-cont} (a-b)), we first generate the FID curves of 3 types of DPM-Solver (with orders 1, 2 and 3) using the updates from their official implementation\footnote{\url{https://github.com/LuChengTHU/dpm-solver}} in noise prediction mode. Taking profit of the tuning advancements proposed by \cite{Karras2022edm}, we used a linear noise schedule with $\beta_{\text{d}}=19.1$ and $\beta_{\min}=0.1$ that slightly differs from the original parameters from \cite{song2020score} but were proven beneficial. We set the end time of sampling $\varepsilon$ to $1$e-$4$ as recommended by \cite[Appendix D.2]{dpm-solver}. The values of all benchmark models for Figure \ref{fig:fid-cont} (b) were taken directly from tables provided by \cite{dpm-solver}.

In our FFHQ-64 experiments, we employ the unconditional VP pretrained\footnote{\url{https://nvlabs-fi-cdn.nvidia.com/edm/pretrained/baseline/baseline-ffhq-64x64-uncond-vp.pkl}} model provided by \cite{Karras2022edm}. 

For the CelebA-64 experiments, we use the pre-trained VP unconditional model whose checkpoint\footnote{\url{https://drive.google.com/file/d/1R_H-fJYXSH79wfSKs9D-fuKQVan5L-GR/view?usp=sharing}} is provided by \cite{song2020denoising}. We use the Type-1 discretization proposed in \cite{dpm-solver} to ensure compatibility of our method with the prescribed trained steps of such model. 

For ImageNet-64, we both use the baseline and the optimized pre-trained models given in \cite{Karras2022edm}. We note that the baseline is trained on the iDDPM class of model \cite{nichol2021improved}, which actually uses different preconditioning and thus the change of variables compared to the optimized model. The Figure \ref{fig:fid-cont} (c) was obtained using the EDM-preconditioned checkpoint\footnote{\url{https://nvlabs-fi-cdn.nvidia.com/edm/pretrained/edm-imagenet-64x64-cond-adm.pkl}}. The added noise settings of SEEDS-3 solver  were not subject to a grid-search optimization procedure. The chosen hyper-parameters were $S_{\tmop{churn}}=11$, $S_{\tmop{noise}}=1.003$, $S_{\tmop{min}}=0.05$, and $S_{\tmop{max}}=15$ but we are confident that this configuration can be optimized to further improve our results.

\subsection{Noise vs. Data Prediction approaches}
In Appendix B of \cite{lu2022dpm}, the authors compare DPM-Solver2 and DPM-Solver++(2S), which amounts on comparing in our framework the difference between the obtained exponential integrators for the PFO in the noise and data prediction regimes to detect exactly a coefficient on the non-linear term that is absent in the noise prediction regime. The term they find corresponds exactly to the difference between applying the variation of constants formula before (instead of after) replacing the score function with the desired neural network. In Tab. \ref{tab:grid-cifar-cont-dp-vs-np} we report both data and noise prediction SEEDS-3. At low NFEs the DP approach gives better results but stabilizes in high NFEs at a FID score that is worse than the one the NP approach reaches. 
 \begin{center}
\begin{table*}[ht]
\caption{Comparison between noise prediction $F_{\theta,t}$ and data prediction $D_{\theta,t}$ modes of SEEDS-3 on CIFAR-10 (VP uncond. cont.).}
\label{tab:grid-cifar-cont-dp-vs-np}
\vskip 0.15in
\begin{center}
\begin{small}
\begin{sc}
\begin{tabular}{lrrrrrrr}
    \toprule
  Method $\backslash$ NFE & 9 & 30 & 60 & 90 & 150 & 165 & 180\\
  \midrule
  SEEDS-3 data prediction & 60.75 & 22.42 & 12.47 & 2.95 & 2.51 & 2.54 & 2.55 \\
  SEEDS-3 noise prediction & 471.29 & 288.20 & 33.92 & 3.76 & 2.40 & \textbf{2.39} & 2.47 \\
  \bottomrule
\end{tabular}
\end{sc}
\end{small}
\end{center}
\end{table*}
\end{center}
\subsection{Low vs. High stage Solvers}
Similar to DPM-Solver \cite{dpm-solver}, the FID scores in Tab. \ref{table:SEEDS-vs-DPMSolver-low-NFEs} and \ref{table:SEEDS-vs-DPMSolver-high-NFEs} show that at low NFEs, higher stage methods performs more poorly while at higher NFEs, DPM-Solver-3 and SEEDS-3 are better than their counterparts with 1 and 2 stages.
\begin{center}
\begin{table*}[ht]
 \centering
\caption{FID comparison between SEEDS (Ours) and DPM-Solver for low NFEs on CIFAR-10 VP uncond. discrete. We recomputed the DPM-Solver score using the "non-deep" model while \cite{dpm-solver} reports results for the "deep" architecture. The symbol $^{\star}$ is used when using 1 NFE more and $^{\dagger}$ when using 1 NFE less because the given NFE cannot be divided by 2 or 3. This corresponds to Figure \ref{fig:fid-cont} (a).}
\label{table:SEEDS-vs-DPMSolver-low-NFEs}
\begin{center}
\begin{small}
\begin{sc}
\begin{tabular}{lrrrrrrrr}
    \toprule
  Method  $\backslash$ NFE &  10 & 12 & 15 & 20 & 30 & 40 & 50 & 100\\
    \midrule
    DPM-Solver-1 & 22.90 & 17.73 & 13.36 & 9.78 & 6.87 & 5.77 & 5.17 & 4.22 \\
    DPM-Solver-2 & \textbf{12.22} & \textbf{6.52} & $^{\star}$\textbf{4.55} & -- & 3.75 & 3.68 & 3.64 & 3.60 \\
    DPM-Solver-3 & $^{\dagger}$66.92 & 9.72 & 5.32 & $^{\star}$3.83 & \textbf{3.66} & -- & $^{\star}$\textbf{3.61} & $^{\dagger}$3.58 \\
    SEEDS-1 & 303.48 & 239.79 & 279.84 & 192.68 & 84.78 & 45.26 & 28.18 & 8.24 \\
    SEEDS-2 & 481.09 & 473.48 & $^{\star}$430.98 & 305.88 & 223.01 & 51.41 & 11.10 & \textbf{3.19} \\
    SEEDS-3 & $^{\dagger}$483.04 & 482.19 & 479.63 & $^{\star}$462.61 & 280.48 & $^{\dagger}$247.44 & $^{\star}$62.62 & $^{\dagger}$3.53 \\
  \bottomrule
\end{tabular}
\end{sc}
\end{small}
\end{center}
\end{table*}
\end{center}
 \begin{center}
\begin{table*}[ht]
 \centering
\caption{FID comparison between SEEDS (Ours) and DPM-Solver for high NFEs on CIFAR-10 VP uncond. discrete. We recomputed the DPM-Solver score using the "non-deep" model while \cite{dpm-solver} reports results for the "deep" architecture. The symbol $^{\star}$ is used when using 1 NFE more because the given NFE cannot be divided by 2 or 3. This corresponds to Figure \ref{fig:fid-cont} (a).}
\label{table:SEEDS-vs-DPMSolver-high-NFEs}
 \vskip 0.15in
\begin{center}
\begin{small}
\begin{sc}
\begin{tabular}{lrrrr}
    \toprule
  Method  $\backslash$ NFE & 150 & 200 & 300 & 510 \\
  \midrule
  DPM-Solver-3 & 3.59 & $^{\star}$3.58 & -- & 3.58\\
  SEEDS-1 & -- &  4.07 & 3.40 & -- \\
  SEEDS-3 & 3.12 & $^{\star}$3.08 & 3.14 & 3.24 \\
  \bottomrule
\end{tabular}
\end{sc}
\end{small}
\end{center}
\end{table*}
\end{center}
\subsection{Deterministic vs. Stochastic Solvers}
\label{app:tables}

Deterministic solvers as DPM-Solver \cite{dpm-solver} are fast and well-adapted to applications in which speed is the most concern. As shown in \cite{dpm-solver}, Table \ref{table:SEEDS-vs-DPMSolver-low-NFEs} and \ref{table:SEEDS-vs-DPMSolver-high-NFEs}, DPM-Solver converges to a local minimum at early steps and cannot be improved in large NFEs. Moreover, preconditioned deterministic solver in EDM gives optimal quality on unconditional CIFAR-10 \cite[Fig. 5 (b)]{Karras2022edm}. However, for more complicated datasets as ImageNet64, the stochasticity indeed helps improve the samples quality \cite[Fig. 5 (c)]{Karras2022edm}. We can consider the random noise as a corrector that approaches a better local or even global minimum. At $S_{\text{churn}}=0$, our SEEDS-3 with stochasticity gives lower FID score than EDM deterministic Heun
(see Fig. \ref{fig:fid-cont} (c)). SEEDS-3 also reaches the best quality prior to the number steps needed in Euler Maruyama and other solvers as in Table \ref{table:SEEDS-vs-other-solvers}. Table \ref{table:SEEDS-vs-0} completes the analysis in Table \ref{table-fid-all} for CIFAR-10 showing FID scores for varying NFEs and multiple versions of the EDM solver, depending on the chosen optimization hyper-parameters. Tables \ref{table:SEEDS-vs-1} to \ref{table:SEEDS-vs-3} complete the findings in Table \ref{table-fid-all} for the remaining used pretrained models. In particular, we report the exact FID values of SEEDS in the low NFE regime.

\begin{center}
\begin{table*}[!t]
 \centering
\caption{FID comparison between SEEDS-3 (Ours) and other solvers on CIFAR-10 VP unconditional discrete. The symbol $^{\star}$ is used when using 1 NFE more and $^{\dagger}$ when using 1 NFE less because the given NFE cannot be divided by 2 or 3. This corresponds to  Figure \ref{fig:fid-cont} (b).}
\label{table:SEEDS-vs-other-solvers}
\vskip 0.15in
\begin{center}
\begin{small}
\begin{sc}
\begin{tabular}{lrrrrrrr}
    \toprule
  Method  $\backslash$ NFE & 10 & 12 & 15 & 20 & 50 & 200 & 1000 \\
  \midrule
  Euler-Maruyama & 278.67 & 246.29 & 197.63 & 137.34 & 32.63 & 4.03 & 3.16 \\
  Analytic DDPM & 35.03 & 27.69 & 20.82 & 15.35 & 7.34 & 4.11 & 3.84 \\
  Analytic DDIM & 14.74 & 11.68 & 9.16 & 7.20 & 4.28 & 3.60 & 3.86 \\
  DDIM & 13.58 & 11.02 & 8.92 & 6.94 & 4.73 & 4.07 & 3.95 \\
  DPM-Solver-3 & $^\dagger$\textbf{6.92} & \textbf{9.72} & \textbf{5.32} & $^{\star}$\textbf{3.83} & $^\star$\textbf{3.61} & $^\star$3.58 & --\\
  SEEDS-3 & $^\dagger$483.04 & 482.19 & 479.63 & $^\star$462.61 & $^\star$62.62 & $^\star$\textbf{3.08} & -- \\
  \bottomrule
\end{tabular}
\end{sc}
\end{small}
\end{center}
\end{table*}
\end{center}

\begin{table}[ht]
  \begin{center}
   \caption{FID comparison of different solvers on CIFAR-10-uncond-vp-continuous. The symbol $^{\star}$ is used when using 1 NFE more and $^{\dagger}$ when using 1 NFE less because the given NFE cannot be divided by 2 or 3. This corresponds to  Figure \ref{fig:fid-cont} (b). All values for EDM are retrieved from the Latex source file SdePlotNfe.tex in the arXiv version of \cite{Karras2022edm}.}
    \label{table:SEEDS-vs-0}
    \begin{center}
      {\small  \textsc{ \begin{tabular}{lrrrrrcc}
        \toprule
        Method $\backslash$ NFE & 30 & 48 & 63 & 126 & 165 & 180 & 511\\
        \midrule
        GottaGoFast & - & 82.42 & - & - &  2.75 & 2.44 & -\\
        EDM ($S_{\text{churn}} = 0$) & $^\star$3.10 &
        $^\dagger$2.99 & 2.94 & $^\star$2.98 & - & - & 2.93\\
        EDM ($S_{\text{tmin,tmax}} + S_{\text{noise}} = 1$) &
        $^\star$3.44 & $^\dagger$2.89 & 2.77 & 2.72 & - & - &
        2.69\\
        EDM ($S_{\text{noise}} = 1$) & $^\star$3.99 &
        $^\dagger$3.13 & 2.90 & 2.60 & - & - & 2.55\\
        EDM ($S_{\text{tmin,tmax}} = [0, \infty]$) & $^\star$3.43 &
        $^\dagger$2.87 & 2.71 & 2.52 & - & - & 2.54\\
        EDM (optimized) & $^\star$3.77 & $^\dagger$3.08 & 2.84
        & 2.47 & - & - & \textbf{2.27}\\
        DPM-Solver (non-deep) & 2.95 & 2.90 &  2.89 & - & - & - & -\\
        SEEDS ($S_{\text{churn}} = 0$) & $^\star$288.20 &
        $^\dagger$90.25 & 33.91 & 2.45 & \textbf{2.39} & 2.47 & -\\
        \bottomrule
      \end{tabular}}}
    \end{center}
  \end{center}
 
\end{table}

\begin{table}[ht]
  \begin{center}
    \caption{FID comparison of different solvers on CelebA 64x64 discrete. The symbol $^{\star}$ is used when using 1 NFE more and $^{\dagger}$ when using 1 NFE less because the given NFE cannot be divided by 2 or 3.}
    \label{table:SEEDS-vs-1}
    
    \begin{center}
      {\small  \textsc{ \begin{tabular}{lrrrrrrrrc}
        \toprule
        Method $\backslash$ NFE & 9 & 12 & 15 & 21 & 51 & 60 & 90 & 102 & 200\\
        \midrule
        E-M & $^\star$310.22 & 227.16 & 207.97 &
        $^\dagger$120.44 & $^\dagger$29.25 & - & - & - &
        3.90\\
        An.-DDPM & $^\star$28.99 & 25.27 & 21.80 &
        $^\dagger$18.14 & $^\dagger$11.23 & - & - & - &
        6.51\\
        An.-DDIM & $^\star$15.62 & 13.90 & 12.29 &
        $^\dagger$10.45 & $^\dagger$6.13 & - & - & - &
        3.46\\
        DDIM & $^\star$10.85 & 9.99 & 7.78 & $^\dagger$6.64 &
        $^\dagger$5.23 & - & - & - & 4.78\\
        DPM-Solver & 6.92 & 4.20 & 3.05 & 2.82 & 2.72 & -
        & - & - &- \\
        SEEDS & 460.87 & 374.48 & 301.66 & 261.87 & 3.84 & 6.58 &
        \textbf{1.88} & 1.97 & -\\
        \bottomrule
      \end{tabular}}}
    \end{center}
  \end{center}

\end{table}

\begin{table}[ht]
  \begin{center}
    \caption{SEEDS-3 on CIFAR-10-cond-vp-continuous (using the baseline model in \cite{Karras2022edm}).}
    \label{table:SEEDS-vs-2}
    
    \begin{center}
      {\small  \textsc{ \begin{tabular}{lrrrrrrrrrr}
        \toprule
        Method $\backslash$ NFE & 15 & 21 & 30 & 60 & 90 & 120 & 129 & 150 & 165 & 180\\
        \midrule
        SEEDS & 239.2 & 167.5 & 131.1 & 25.06 & 3.19 & 2.17
        & \textbf{2.08} & 2.15 & 2.16 & 2.19 \\
        \bottomrule
      \end{tabular}}}
    \end{center}
  \end{center}

\end{table}

\begin{table}[ht]
  \begin{center}
    \caption{Tailored SEEDS-3 for the EDM-preconditioned pretrained model \cite{Karras2022edm} on ImageNet 64x64.}
    \label{table:SEEDS-vs-3}
    
    \begin{center}
      {\small  \textsc{ \begin{tabular}{lrrrrrrr}
        \toprule
        Method $\backslash$ NFE & 12 & 15 & 21 & 51 & 102 & 201 & 270\\
        \midrule
        SEEDS & 209.12 & 197.79 & 153.72 & 63.75 & 16.35 & 1.56
        & \textbf{1.38} \\
        \bottomrule
      \end{tabular}}}
    \end{center}
  \end{center}
\end{table}

\subsection{Run-time Comparison}

As a sanity check, we provide a run-time comparison table for our experiments on different datasets (see Tab. \ref{table:runtime}). One can see that the run-time is linear with respect to the NFE also for SEEDS, as the main advantage of the SETD method is to analytically compute the stochastic components in our solver, making their computation cost negligible. Some mild improvements about repetitive terms allowed our implementation of SEEDS to be slightly more effective than EDM and even DPM-Solver at same NFE.

\begin{table}[ht]
  \caption{\label{table:runtime} Run-time comparison (second/batch + std) on a single NVIDIA V100 of EDM, DPM-Solver and SEEDS. Discretely-trained models are run with the implementation based on the code from \cite{dpm-solver}. Continuously-trained models are run using the code from \cite{Karras2022edm}.}
  \begin{center}
    \begin{center}
      {\small \textsc{ \begin{tabular}{lrrrr}
        \hline
        Method $\backslash$ NFE & 9 & 21 & 51 & 90 \\
        \midrule
        \multicolumn{4}{l}{\textbf{CIFAR-10 32$\times$32 continuous (batch size = 128)}} \\
        EDM & 2.096({$\pm$}0.003) & 4.891({$\pm$}0.004) & 11.592({$\pm$}0.019) &
        21.213({$\pm$}0.009) \\
        DPM-Solver & 2.099({$\pm$}0.002) & 4.871({$\pm$}0.004) &
        11.841({$\pm$}0.014) & 20.888({$\pm$}0.020) \\
        SEEDS & 2.086({$\pm$}0.002) & 4.867({$\pm$}0.003) & 11.817({$\pm$}0.006) &
        20.896({$\pm$}0.028) \\
        \midrule
        \multicolumn{4}{l}{\textbf{FFHQ 64$\times$64 continuous (batch size = 128)}}\\
        EDM & 4.361({$\pm$}0.005) & 10.179({$\pm$}0.005) & 24.738({$\pm$}0.018) &
        44.145({$\pm$}0.016) \\
        DPM-Solver & 4.344({$\pm$}0.005) & 10.148({$\pm$}0.009) &
        24.637({$\pm$}0.007) & 43.526({$\pm$}0.029) \\
        SEEDS & 4.353({$\pm$}0.003) & 10.157({$\pm$}0.004) & 24.661({$\pm$}0.011) &
        43.537({$\pm$}0.013) \\
        \midrule
        \multicolumn{4}{l}{\textbf{ImageNet 64$\times$64 continuous (batch size = 128)}}\\
        EDM & 7.525({$\pm$}0.007) & 17.579({$\pm$}0.004) & 42.696({$\pm$}0.010) &
        76.175({$\pm$}0.024) \\
        DPM-Solver & 7.535({$\pm$}0.048) & 17.556({$\pm$}0.009) &
        42.629({$\pm$}0.018) & 75.222({$\pm$}0.026) \\
        SEEDS & 7.429({$\pm$}0.006) & 17.572({$\pm$}0.013) & 42.654({$\pm$}0.014) &
        75.245({$\pm$}0.034) \\
        \midrule
        \multicolumn{4}{l}{\textbf{CIFAR-10 32$\times$32 discrete (batch size = 128)}}\\
        DPM-Solver & 0.272({$\pm$}0.004) & 0.529({$\pm$}0.007) & 1.324({$\pm$}0.007)
        & 2.632({$\pm$}0.004) \\
        SEEDS & 0.261({$\pm$}0.002) & 0.523({$\pm$}0.002) & 1.299({$\pm$}0.003) &
        2.598({$\pm$}0.003) \\
        \midrule
        \multicolumn{4}{l}{\textbf{CelebA 64$\times$64 discrete (batch size = 128)}}\\
        DPM-Solver & 0.936({$\pm$}0.004) & 1.812({$\pm$}0.003) & 4.558({$\pm$}0.008)
        & 9.108({$\pm$}0.008) \\
        SEEDS & 0.912({$\pm$}0.002) & 1.808({$\pm$}0.004) & 4.526({$\pm$}0.002) &
        9.033({$\pm$}0.004) \\
        \midrule
        \multicolumn{4}{l}{\textbf{LSUN-Bedroom 256$\times$256 discrete (batch size = 32)}}\\
        DPM-Solver & 5.566({$\pm$}0.019) & 11.124({$\pm$}0.018) &
        27.815({$\pm$}0.031) & 55.648({$\pm$}0.021) \\
        SEEDS & 5.498({$\pm$}0.008) & 11.001({$\pm$}0.022) & 27.503({$\pm$}0.029) &
        54.842({$\pm$}0.020) \\
        \midrule
        \multicolumn{4}{l}{\textbf{LDM-CelebAHQ 256$\times$256 (batch size = 64)}}\\
        DPM-Solver & 8.648({$\pm$}0.013) & 17.492({$\pm$}0.013) &
        39.569({$\pm$}0.015) & 68.205({$\pm$}0.017) \\
        SEEDS & 8.652({$\pm$}0.005) & 17.469({$\pm$}0.010) & 39.524({$\pm$}0.010) &
        68.240({$\pm$}0.026) \\
        \midrule
        \multicolumn{4}{l}{\textbf{Stable Diffusion 512$\times$512 (batch size = 16)}}\\
        DPM-Solver & 19.598({$\pm$}0.070) & 40.679({$\pm$}0.107) &
        92.914({$\pm$}0.079) & 163.902({$\pm$}0.116) \\
        SEEDS & 19.656({$\pm$}0.106) & 40.781({$\pm$}0.136) & 93.409({$\pm$}0.123) &
        161.451({$\pm$}0.518) \\
        \bottomrule
      \end{tabular}}}
    \end{center}
  \end{center}
\end{table}

\subsection{Hardware configuration} 
During the experiments, we used three Linux-based servers with 60GB memory each and 4 GPUs NVIDIA V100 32GB, 4 GPUs NVIDIA V100 16GB, and 2 GPUs NVIDIA V100 32GB, respectively.  Table \ref{table-hardware} shows the detail of the configuration utilized for each experiment. We noted that when using the 4 GPUs configuration, the FID results were slightly lower (around 2\%), even after using a stacked fixed random seed. We run the experiments multiple times and reported the minimum FID value each time.

\begin{table}[!t]
\caption{Details of GPUs utilized during the experiments.}
\label{table-hardware}
\begin{center}
\begin{sc}
\begin{tabular}{llcc}
    \toprule
  Experiment & Model & Number & GPU Size\\
  \midrule
  CIFAR-10 continuous & Nvidia V100 & 4 & 16 GB\\
  CIFAR-10 discrete & Nvidia V100  &  2 & 32 GB\\
  FFHQ64  & Nvidia V100 & 4  & 16 GB \\
  CelebA64  & Nvidia A100 & 2 & 16 GB \\
  ImageNet64 & Nvidia V100 & 4 & 32 GB \\
  \bottomrule
\end{tabular}
\end{sc}
\end{center}
\end{table}

\subsection{Licences}
Pre-trained models:  
\begin{itemize}
  \item CIFAR-10 models by \cite{song2020score}: Apache V2.0 license
   \item FFHQ-64 model by \cite{Karras2022edm}: Creative Commons Attribution-NonCommercial-ShareAlike 4.0 International License.
    \item CelebA-64 model by \cite{song2020denoising}: Apache V2.0 license
    \item ImageNet-64 model by \cite{Karras2022edm}: Apache V2.0 license
    \item Inception-v3 model by \cite{szegedy2016rethinking}: Apache V2.0 license
\end{itemize}

\subsection{Supporting samples}\label{app:figs}
In this subsection, we report the image grids supporting our claims in section \ref{sec:exps}.

\begin{figure}[!ht]
\begin{minipage}{.05\linewidth}
    NP\vspace{3cm}\\
    DP
\end{minipage}
\begin{minipage}{.95\linewidth}
    \centering

        \makebox[0.16\linewidth]{\footnotesize NFE=9 }\hfill%
\makebox[0.16\linewidth]{\hspace*{-.30em}\footnotesize NFE=21}\hfill%
\makebox[0.16\linewidth]{\hspace*{-.15em}\footnotesize NFE=30}\hfill
\makebox[0.16\linewidth]{\footnotesize NFE=90}\hfill%
\makebox[0.16\linewidth]{\hspace*{-.30em}\footnotesize NFE=129}\hfill%
\makebox[0.16\linewidth]{\hspace*{-.15em}\footnotesize NFE=165}\hfill
    \vspace{.3cm}
    \noindent\resizebox{\textwidth}{!}{
        \includegraphics{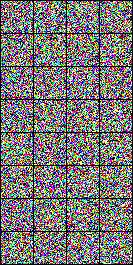}\hspace{.1cm}
        \includegraphics{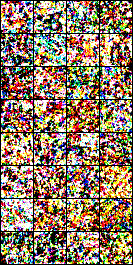}\hspace{.1cm}
        \includegraphics{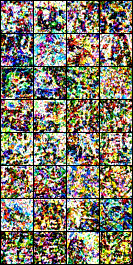}\hspace{.1cm}
        \includegraphics{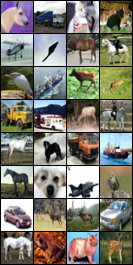}\hspace{.1cm}
        \includegraphics{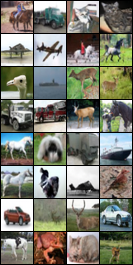}\hspace{.1cm}
        \includegraphics{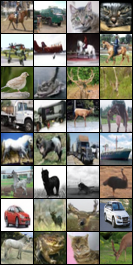}
    }

    \vspace{.1cm}
    \noindent\resizebox{\textwidth}{!}{
        \includegraphics{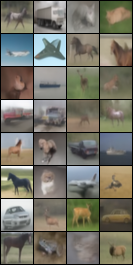}\hspace{.1cm}
        \includegraphics{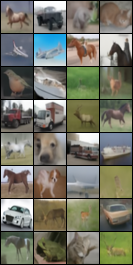}\hspace{.1cm}
        \includegraphics{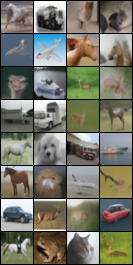}\hspace{.1cm}
        \includegraphics{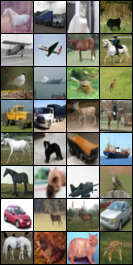}\hspace{.1cm}
        \includegraphics{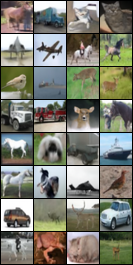}\hspace{.1cm}
        \includegraphics{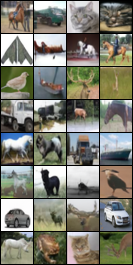}
    }

    \caption{\label{fig:grid-cifar10-cond}Samples on CIFAR-10 from low to high NFEs by SEEDS-3 in Noise Prediction (NP) and Data Prediction (DP) modes, using conditional VP continuous baseline model \cite{Karras2022edm}.}  
\end{minipage}
\end{figure}

\begin{figure}[!ht]
\begin{minipage}{.05\linewidth}
    NP\vspace{3cm}\\
    DP
\end{minipage}
\begin{minipage}{.95\linewidth}
    \centering

        \makebox[0.16\linewidth]{\footnotesize NFE=9 }\hfill%
\makebox[0.16\linewidth]{\hspace*{-.30em}\footnotesize NFE=21}\hfill%
\makebox[0.16\linewidth]{\hspace*{-.15em}\footnotesize NFE=30}\hfill
\makebox[0.16\linewidth]{\footnotesize NFE=90}\hfill%
\makebox[0.16\linewidth]{\hspace*{-.30em}\footnotesize NFE=129}\hfill%
\makebox[0.16\linewidth]{\hspace*{-.15em}\footnotesize NFE=165}\hfill
    \vspace{.3cm}
    \noindent\resizebox{\textwidth}{!}{
        \includegraphics{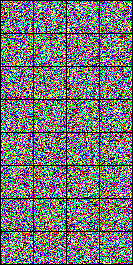}\hspace{.1cm}
        \includegraphics{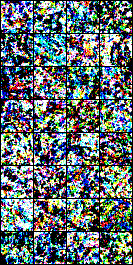}\hspace{.1cm}
        \includegraphics{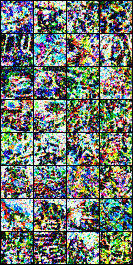}\hspace{.1cm}
        \includegraphics{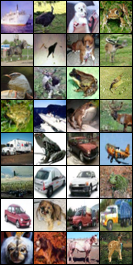}\hspace{.1cm}
        \includegraphics{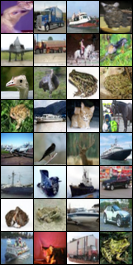}\hspace{.1cm}
        \includegraphics{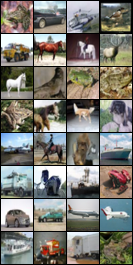}
    }

    \vspace{.1cm}
    \noindent\resizebox{\textwidth}{!}{
        \includegraphics{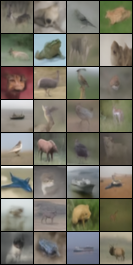}\hspace{.1cm}
        \includegraphics{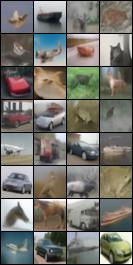}\hspace{.1cm}
        \includegraphics{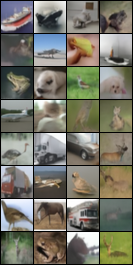}\hspace{.1cm}
        \includegraphics{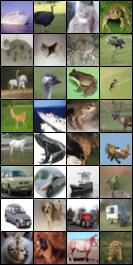}\hspace{.1cm}
        \includegraphics{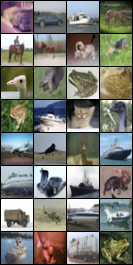}\hspace{.1cm}
        \includegraphics{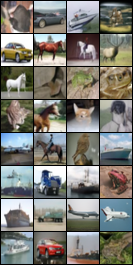}
    }

    \caption{\label{fig:grid-cifar10-uncond}Samples on CIFAR-10 from low to high NFEs by SEEDS-3 in Noise Prediction (NP) and Data Prediction (DP) modes, using unconditional VP continuous baseline model \cite{Karras2022edm}. }
\end{minipage}
\end{figure}

\begin{figure}[!ht]
    \begin{minipage}{.05\linewidth}
    \rotatebox{90}{SEEDS-3}\vspace{2.5cm}\\
    \rotatebox{90}{DPM-Solver-3}
\end{minipage}
\begin{minipage}{.95\linewidth}
    \centering

        \makebox[0.16\linewidth]{\footnotesize NFE=9 }\hfill%
\makebox[0.16\linewidth]{\hspace*{-.30em}\footnotesize NFE=21}\hfill%
\makebox[0.16\linewidth]{\hspace*{-.15em}\footnotesize NFE=30}\hfill
\makebox[0.16\linewidth]{\footnotesize NFE=60 }\hfill%
\makebox[0.16\linewidth]{\hspace*{-.30em}\footnotesize NFE=90}\hfill%
\makebox[0.16\linewidth]{\hspace*{-.15em}\footnotesize NFE=150}\hfill
    \vspace{.3cm}
    \noindent\resizebox{\textwidth}{!}{
        \includegraphics{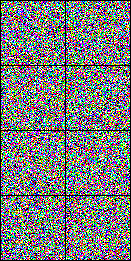}\hspace{.1cm}
        \includegraphics{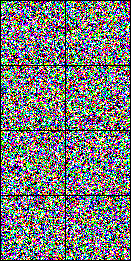}\hspace{.1cm}
        \includegraphics{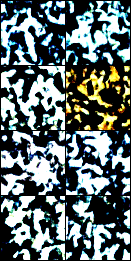}\hspace{.1cm}
        \includegraphics{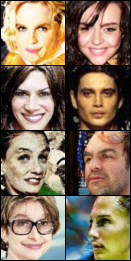}\hspace{.1cm}
        \includegraphics{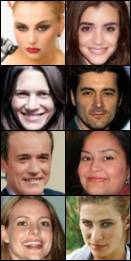}\hspace{.1cm}
        \includegraphics{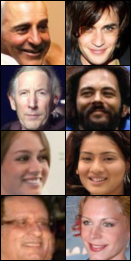}
    }

    \vspace{.1cm}
    \noindent\resizebox{\textwidth}{!}{
        \includegraphics{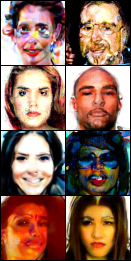}\hspace{.1cm}
        \includegraphics{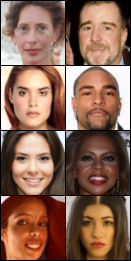}\hspace{.1cm}
        \includegraphics{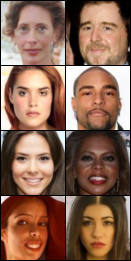}\hspace{.1cm}
        \includegraphics{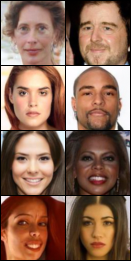}\hspace{.1cm}
        \includegraphics{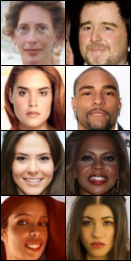}\hspace{.1cm}
        \includegraphics{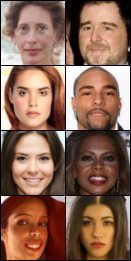}
    }
    \caption{\label{fig:grid-celeba64}Samples on CelebA-64 from low to high NFEs by SEEDS-3 and DPM-Solver-3, using pre-trained model from \cite{song2020denoising}. }
\end{minipage}
\end{figure}

\begin{figure}[!ht]
\begin{minipage}{.05\linewidth}
    NP\vspace{3cm}\\
    DP
\end{minipage}
\begin{minipage}{.95\linewidth}
    \centering

        \makebox[0.16\linewidth]{\footnotesize NFE=9 }\hfill%
\makebox[0.16\linewidth]{\hspace*{-.30em}\footnotesize NFE=21}\hfill%
\makebox[0.16\linewidth]{\hspace*{-.15em}\footnotesize NFE=30}\hfill
\makebox[0.16\linewidth]{\footnotesize NFE=60 }\hfill%
\makebox[0.16\linewidth]{\hspace*{-.30em}\footnotesize NFE=90}\hfill%
\makebox[0.16\linewidth]{\hspace*{-.15em}\footnotesize NFE=150}\hfill
    \vspace{.3cm}
    \noindent\resizebox{\textwidth}{!}{
        \includegraphics{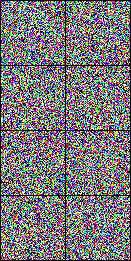}\hspace{.1cm}
        \includegraphics{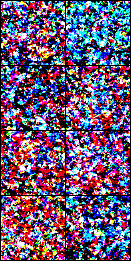}\hspace{.1cm}
        \includegraphics{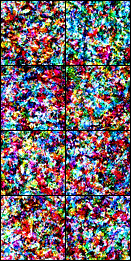}\hspace{.1cm}
        \includegraphics{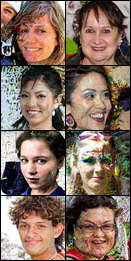}\hspace{.1cm}
        \includegraphics{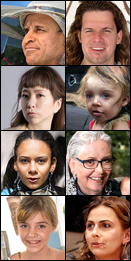}\hspace{.1cm}
        \includegraphics{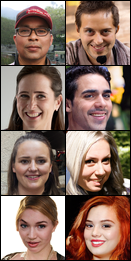}
    }

    \vspace{.1cm}
    \noindent\resizebox{\textwidth}{!}{
        \includegraphics{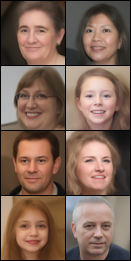}\hspace{.1cm}
        \includegraphics{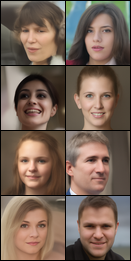}\hspace{.1cm}
        \includegraphics{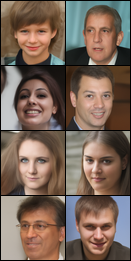}\hspace{.1cm}
        \includegraphics{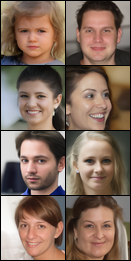}\hspace{.1cm}
        \includegraphics{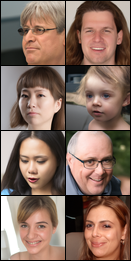}\hspace{.1cm}
        \includegraphics{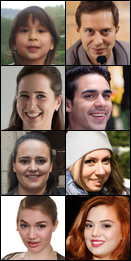}
    }
    \caption{\label{fig:grid-ffhq64}Samples on FFHQ-64 from low to high NFEs by SEEDS-3 in Noise Prediction (NP) and Data Prediction (DP) modes, using unconditional VP continuous baseline model \cite{Karras2022edm}. }
\end{minipage}
\end{figure}

\begin{figure}[!ht]
\begin{minipage}{.05\linewidth}
    NP\vspace{3cm}\\
    DP
\end{minipage}
\begin{minipage}{.95\linewidth}
    \centering

        \makebox[0.16\linewidth]{\footnotesize NFE=30 }\hfill%
\makebox[0.16\linewidth]{\hspace*{-.30em}\footnotesize NFE=60}\hfill%
\makebox[0.16\linewidth]{\hspace*{-.15em}\footnotesize NFE=90}\hfill
\makebox[0.16\linewidth]{\footnotesize NFE=150 }\hfill%
\makebox[0.16\linewidth]{\hspace*{-.30em}\footnotesize NFE=210}\hfill%
\makebox[0.16\linewidth]{\hspace*{-.15em}\footnotesize NFE=270}\hfill
    \vspace{.3cm}
    \noindent\resizebox{\textwidth}{!}{
        \includegraphics{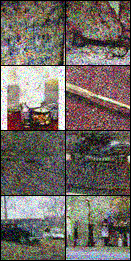}\hspace{.1cm}
        \includegraphics{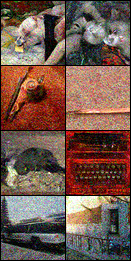}\hspace{.1cm}
        \includegraphics{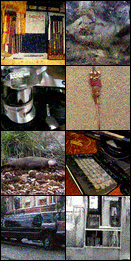}\hspace{.1cm}
        \includegraphics{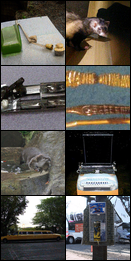}\hspace{.1cm}
        \includegraphics{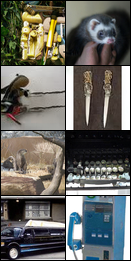}\hspace{.1cm}
        \includegraphics{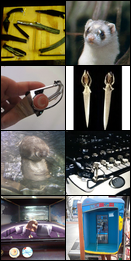}\hspace{.1cm}
    }

    \vspace{.1cm}
    \noindent\resizebox{\textwidth}{!}{
        \includegraphics{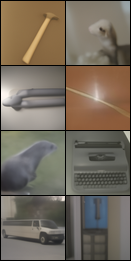}\hspace{.1cm}
        \includegraphics{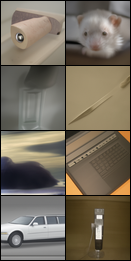}\hspace{.1cm}
        \includegraphics{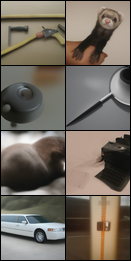}\hspace{.1cm}
        \includegraphics{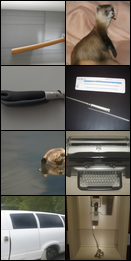}\hspace{.1cm}
        \includegraphics{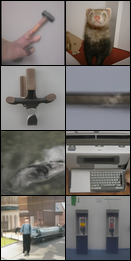}\hspace{.1cm}
        \includegraphics{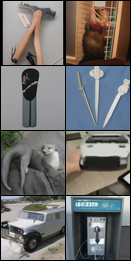}
    }

    \caption{\label{fig:grid-imagenet64-optimized}Samples on ImageNet-64 from low to high NFEs in Noise Prediction (NP) and Data Prediction (DP) modes, using conditional EDM optimized model \cite{Karras2022edm}. }
\end{minipage}
\end{figure}

\begin{figure}[!ht]
\begin{minipage}{.05\linewidth}
    NP\vspace{3cm}\\
    DP
\end{minipage}
\begin{minipage}{.95\linewidth}
    \centering

        \makebox[0.16\linewidth]{\footnotesize NFE=30 }\hfill%
\makebox[0.16\linewidth]{\hspace*{-.30em}\footnotesize NFE=60}\hfill%
\makebox[0.16\linewidth]{\hspace*{-.15em}\footnotesize NFE=90}\hfill
\makebox[0.16\linewidth]{\footnotesize NFE=150 }\hfill%
\makebox[0.16\linewidth]{\hspace*{-.30em}\footnotesize NFE=210}\hfill%
\makebox[0.16\linewidth]{\hspace*{-.15em}\footnotesize NFE=270}\hfill
    \vspace{.3cm}
    \noindent\resizebox{\textwidth}{!}{
        \includegraphics{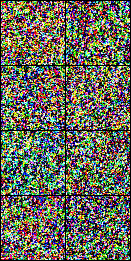}\hspace{.1cm}
        \includegraphics{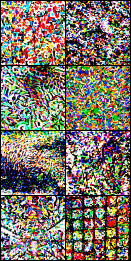}\hspace{.1cm}
        \includegraphics{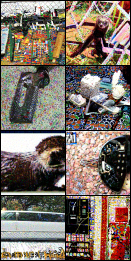}\hspace{.1cm}
        \includegraphics{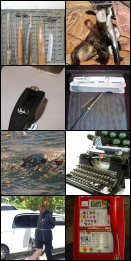}\hspace{.1cm}
        \includegraphics{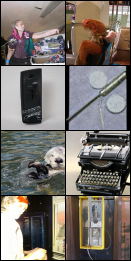}\hspace{.1cm}
        \includegraphics{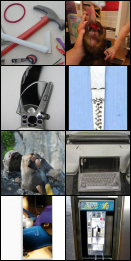}
    }

    \vspace{.1cm}
    \noindent\resizebox{\textwidth}{!}{
        \includegraphics{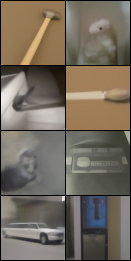}\hspace{.1cm}
        \includegraphics{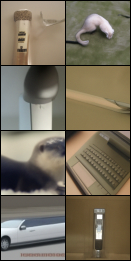}\hspace{.1cm}
        \includegraphics{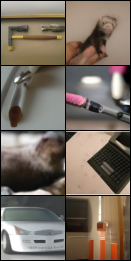}\hspace{.1cm}
        \includegraphics{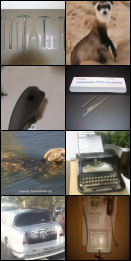}\hspace{.1cm}
        \includegraphics{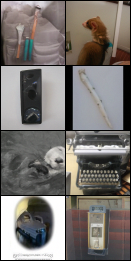}\hspace{.1cm}
        \includegraphics{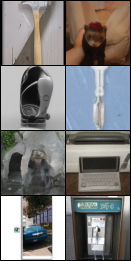}
    }

    \caption{\label{fig:grid-imagenet64-baseline}Samples on ImageNet-64 from low to high NFEs in Noise Prediction (NP) and Data Prediction (DP) modes, using conditional EDM baseline model \cite{Karras2022edm}. }
\end{minipage}
\end{figure}

\begin{figure}[!ht]
 \centering
\includegraphics[width=0.8\linewidth, height=0.8\linewidth]{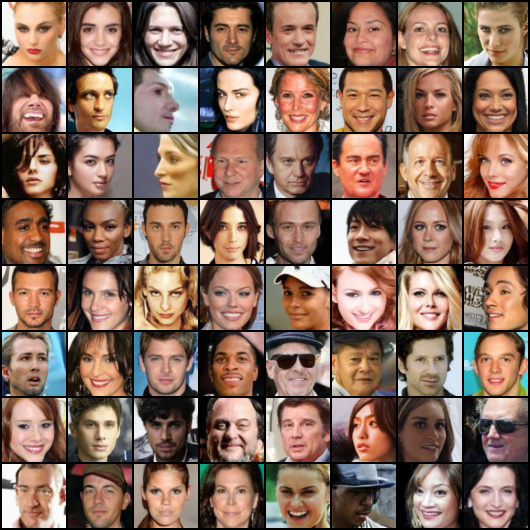}
\caption{\label{fig:grid-celeba-uncon-dis}Example of samples on CelebA-64 generated by SEEDS-3 in 90 NFEs, using pre-trained model from \cite{song2020denoising}. }
\vspace{-.2in}
\end{figure}

 \begin{figure*}[htb]
 \centering
   \subfigure[DPM-Solver-3 at 44 NFE.]{\includegraphics[width=0.49\linewidth,]{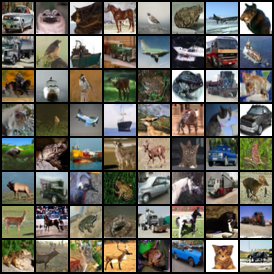}} 
   \subfigure[SEEDS-3 at 201 NFE.]{\includegraphics[width=0.49\linewidth]{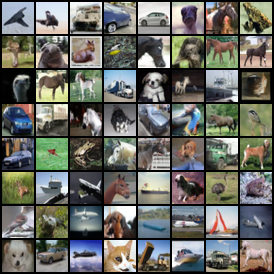}}
\caption{\label{fig:grid-cifar-uncon-dis}Visual sample quality comparison between DPM-Solver-3 and SEEDS-3 using their optimal settings and unconditional CIFAR-10 discrete model \cite{song2020score}. }
\vspace{-.2in}
 \end{figure*}

 \begin{figure*}[htb]
 \centering
   \subfigure[DPM-Solver-3 at 80 NFE.]{\includegraphics[width=0.32\linewidth]{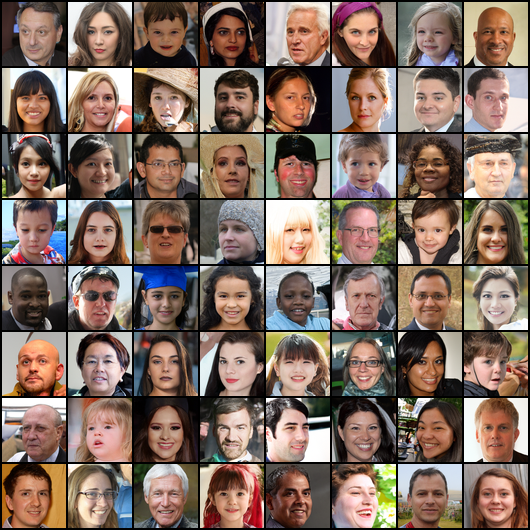}} 
   \subfigure[EDM at 80 NFE.]
   {\includegraphics[width=0.32\linewidth]{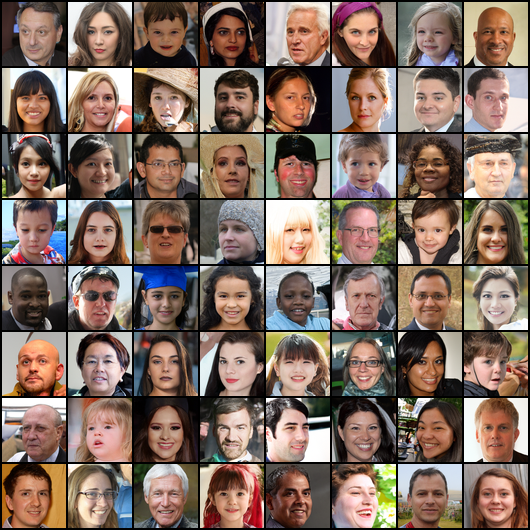}}
    \subfigure[SEEDS-3 at 150 NFE.]
   {\includegraphics[width=0.32\linewidth]{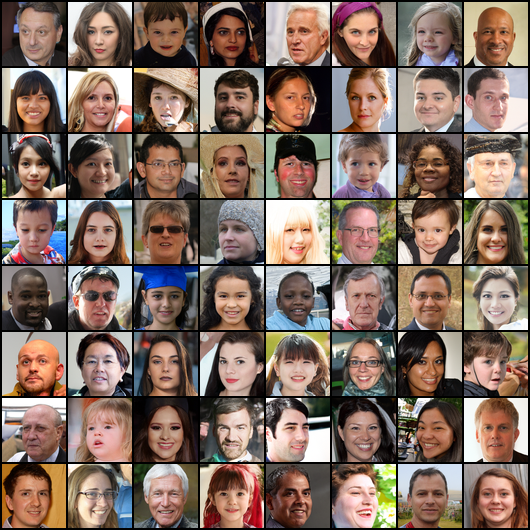}}
\caption{\label{fig:grid-ffqh-cont-dp-vs-np}Visual sample quality comparison between DPM-Solver-3, EDM and SEEDS-3 using their optimal settings and unconditional VP FFHQ-64 continuous model \cite{Karras2022edm}. }
\vspace{-.2in}
 \end{figure*}

\begin{figure*}[!b]
 \centering
   \subfigure[DPM-Solver-3 at 30 NFE.]{\includegraphics[width=0.32\linewidth]{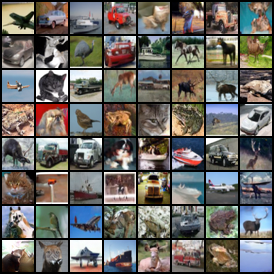}} 
   \subfigure[EDM at 90 NFE.]
   {\includegraphics[width=0.32\linewidth]{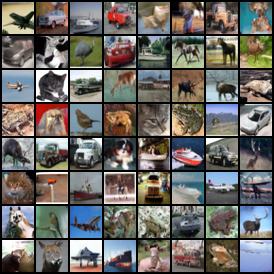}}
    \subfigure[SEEDS-3 at 150 NFE.]
   {\includegraphics[width=0.32\linewidth]{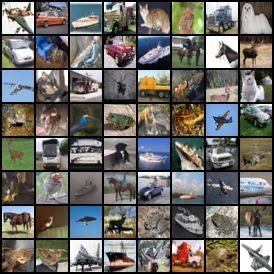}}
\caption{\label{fig:grid-cifar-cont-dp-vs-np}Visual sample quality comparison between DPM-Solver-3, EDM and SEEDS-3 using their optimal settings and conditional VP CIFAR-10 baseline continuous model \cite{Karras2022edm}. }
 \end{figure*}

\begin{figure}[!ht]
    \centering
    \includegraphics[width=0.85\linewidth]{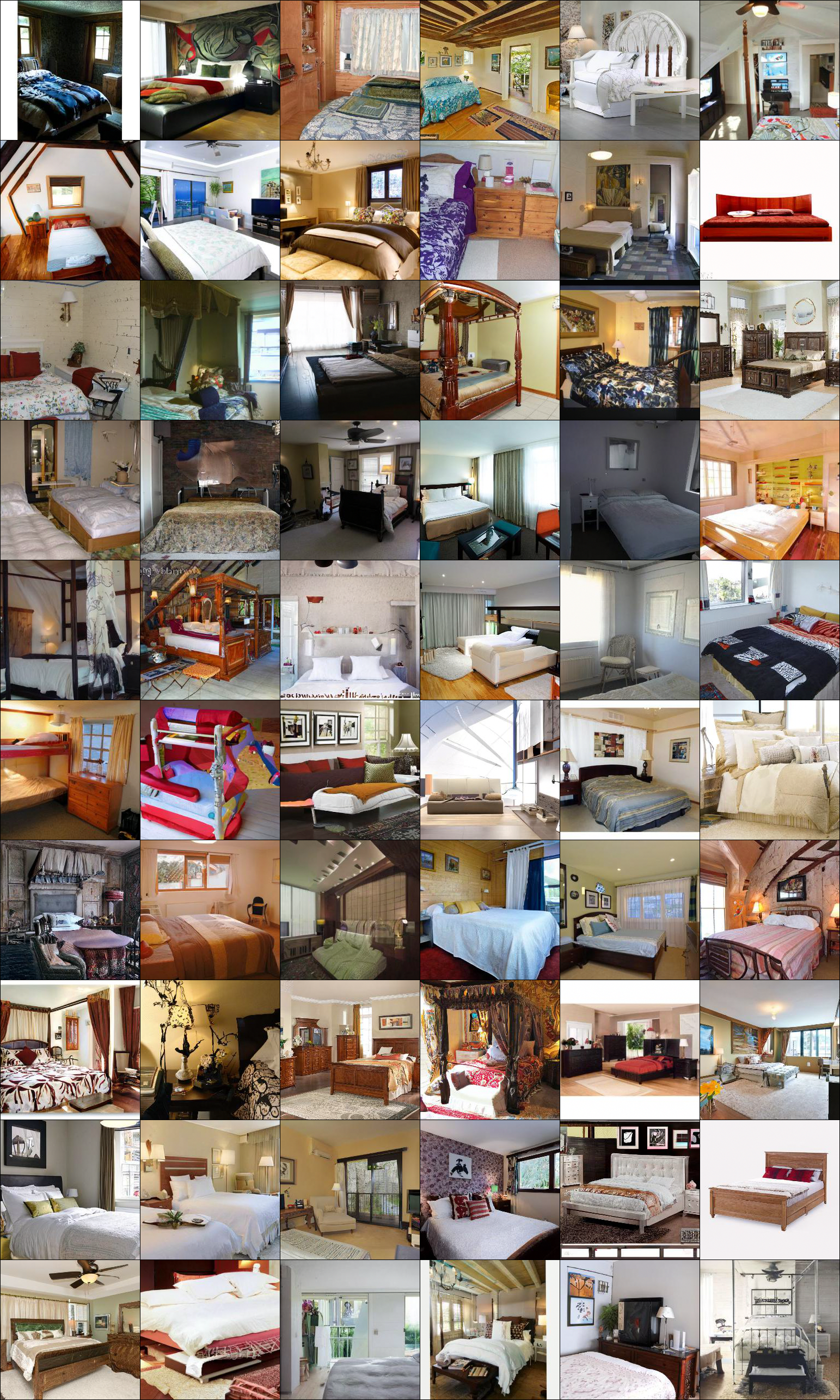}
    \caption{Example of samples on LSUN-Bedroom-256 by SEEDS-3 in 201 NFEs, using pre-trained model from \cite{dhariwal2021diffusion}.}
    \label{fig:grid-lsun-bedroom}
\end{figure}

 \begin{figure*}[t]
    \centering
     \includegraphics[width=.7\textwidth]{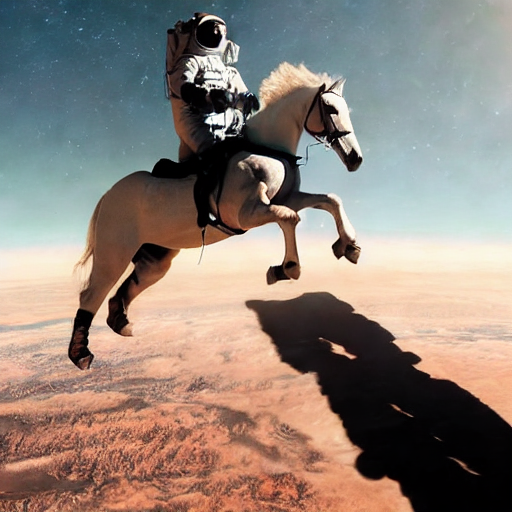}
     \caption{\label{fig:SD-512x512}StableDiffusion-$512\times 512$ by SEEDS-$1$ with prompt ``\textit{High quality photo of an astronaut riding a horse in space}''  and NFE $=90$.}
 \end{figure*}